\RequirePackage{fix-cm}
\documentclass[twocolumn]{svjour3}           % twocolumn
\smartqed  % flush right qed marks, e.g. at end of proof
\journalname{arXiv}
\usepackage[utf8]{inputenc} % to insert the character directly by copy paste or as ^+i typed on your keyboard
\usepackage[T1]{fontenc} % to lower the accent
\usepackage{natbib}
\usepackage{caption} % This seems to crap out my latex compiler offline
\usepackage{amsmath}
\usepackage{amssymb}
\usepackage{amsfonts}
\usepackage{xcolor}
\usepackage{xspace}
\usepackage{url}
\usepackage{enumerate}
\usepackage{comment}
\usepackage{todonotes}
\usepackage{stackengine}
\usepackage{subcaption}
\usepackage{amsfonts}
\usepackage{booktabs}
\usepackage{multirow}
\usepackage{MnSymbol}
\usepackage{graphicx}
\usepackage{pdflscape}
\usepackage{afterpage}

\usepackage[inline]{enumitem}
\newlist{romanenumerate*}{enumerate*}{1}
\setlist[romanenumerate*]{label=(\textit{\roman*})}
\newlist{romanenumerate}{enumerate}{1}
\setlist[romanenumerate]{label=(\textit{\roman*})}

\newtheorem{examplebold}{\textbf{Example}}
\newcommand{\astfootnote}[1]{%
\let\oldthefootnote=\thefootnote%
\setcounter{footnote}{0}%
\renewcommand{\thefootnote}{\fnsymbol{footnote}}%
\footnote{#1}%
\let\thefootnote=\oldthefootnote%
}

\usepackage{tikz}
\usetikzlibrary{automata, positioning, arrows, fit}
\tikzset{->,
    >=stealth,
    node distance=3cm,
    every state/.style={thick, fill=gray!10},
    align=center
}

\definecolor{mygreen}{rgb}{0,0.6,0}
\definecolor{myorange}{rgb}{1.0,0.5,0.3}
\definecolor{mymauve}{rgb}{0.58,0,0.82}
\definecolor{myblue}{rgb}{0.05,0.19,0.57}
\definecolor{mygrey}{rgb}{0.4,0.4,0.4}
\definecolor{myred}{rgb}{0.9,0.2,0.15}
\usepackage{listings}
\lstdefinelanguage{pddl}
{
  basicstyle=\footnotesize,
  sensitive=false,    % not case-sensitive
  morecomment=[l]{;}, % line comment
  alsoletter={:,-},   % consider extra characters
  morekeywords={
    define,domain,problem,exists,either,
    :domain,:extends,:requirements,:types,:objects,:constants,
    :predicates,:functions,
    :fluents,:primary-effect,:side-effect,assign
    :strips,:adl,:equality,:typing,:conditional-effects, :metric, minimize,
    :negative-preconditions,:disjunctive-preconditions,
    :existential-preconditions,:universal-preconditions
  },
  keywords=[2]{not,and,or,when,forall}, % Objects
  keywords=[3]{:parameters,:precondition,:effect,}, % Predicates
  keywords=[4]{assignMenus, assignMC, assignSC, :action, :init, :goal}, % Actions
  keywords=[5]{} % Functions
}
\usepackage{todonotes}

\newcommand{\policy}{\pi}

\newcommand{\pred}[1]{{\fontsize{8.5}{9.5}\selectfont\texttt{#1}\xspace}}

\newcommand{\Pre}{{\mathit{Pre}}\xspace}
\newcommand{\Eff}{{\mathit{Eff}}\xspace}

\DeclareMathOperator*{\argmax}{arg\,max}

\newcommand{\A}{\mathcal{A}} 
 \newcommand{\D}{\mathcal{D}}
\newcommand{\E}{\mathcal{E}} \newcommand{\F}{\mathcal{F}}
\newcommand{\G}{\mathcal{G}}

\renewcommand{\O}{\mathcal{O}} \renewcommand{\P}{\mathcal{P}}
 \newcommand{\R}{\mathcal{R}}
\renewcommand{\S}{\mathcal{S}} \newcommand{\T}{\mathcal{T}}

\newcommand{\V}{\mathcal{V}}
\newcommand{\U}{\mathcal{U}}

\newcommand{\bfmath}[1]{\mathbf{#1}}

\newcommand{\Yesterday}{\circleddash}
\newcommand{\gpast}{{\boxminus}}
\newcommand{\past}{{\diamondminus}}
\newcommand{\Since}{\mathop{\S}}
\newcommand{\Wnext}{\raisebox{-0.27ex}{\LARGE$\bullet$}}
\newcommand{\Next}{\raisebox{-0.27ex}{\LARGE$\circ$}}
\newcommand{\lUntil}{\mathop{\U}}

\newcommand{\true}{\mathit{true}}

\newcommand{\tm}[1]{\ \text{#1}\ }
\newcommand{\trace}{\tau}

\newcommand{\LTL}{{\sc ltl}\xspace}
\newcommand{\LTLf}{{\sc ltl}$_f$\xspace}

\newcommand{\LTLp}{{\sc ppltl}\xspace}
\newcommand{\PLTLf}{{\sc ppltl}\xspace}

\newcommand{\FOL}{{\sc fol}\xspace}

\newcommand{\NFA}{{\textsf{NFA}}\xspace}

\newcommand{\DFA}{{\textsf{DFA}}\xspace}
\newcommand{\DFAs}{{\textsf{DFA}}s\xspace}
\newcommand{\PDFA}{{\textsf{PDFA}}\xspace}

\newcommand{\PDDL}{{\sc PDDL}\xspace}
\newcommand{\FOND}{{\sc fond}\xspace}
\newcommand{\MONA}{\textsf{MONA}\xspace}
\newcommand{\FONDforLTLPLTL}{{\sc fond4\LTLf}\xspace}

\newcommand{\LTLftitle}{\textbf{\textsc{LTL}$_f$}\xspace}
\newcommand{\LTLptitle}{\textbf{\textsc{PPLTL}}\xspace}

\newcommand{\FONDtitle}{\textbf{\textsc{FOND}}\xspace}

\newcommand{\Nat}{{\rm I\kern-.23em N}}

\newcommand{\tup}[1]{\langle #1 \rangle}

\newcommand{\tiff}{\; \text{ iff }\;}

\begin{document}

\title{Temporally Extended Goal Recognition in Fully Observable Non-Deterministic Domain Models}
% \subtitle{Temporally Extended Goal Recognition in FOND Planning}


%----------------------------------------------------------------------------------------
\author{Ramon Fraga Pereira$^{1}$ \and
        Francesco Fuggitti$^{2,3,4}$ 
		\\ Felipe Meneguzzi$^{5}$ \and 
		Giuseppe De Giacomo$^{3,6}$ \\ \\ $^{1}$University of Manchester, UK \\
        $^{2}$IBM Research AI, USA \\
        $^{3}$Sapienza University, Italy \\
        $^{4}$York University, Canada \\
        $^{5}$University of Aberdeen, UK \\
        $^{6}$University of Oxford, UK
}


\institute{Corresponding Author: Ramon Fraga Pereira \at
\email{ramon.fragapereira@manchester.ac.uk}
}

% \institute{}

%----------------------------------------------------------------------------------------
\date{Submitted: 06/2023}

%----------------------------------------------------------------------------------------
\maketitle
%----------------------------------------------------------------------------------------

%----------------------------------------------------------------------------------------

%--------------------------------------------------------------------
\begin{abstract}

\textit{Goal Recognition} is the task of discerning the correct intended goal that an agent aims to achieve, given a set of goal hypotheses, a domain model, and a sequence of observations (i.e., a sample of the plan executed in the environment).  %% FRM - I removed the "being" from "being executed", it was unnecessary
Existing approaches assume that goal hypotheses comprise a single conjunctive formula over a single final state and that the environment dynamics are deterministic, preventing the recognition of temporally extended goals in more complex settings. 
In this paper, we expand goal recognition to \textit{temporally extended goals} in \textit{Fully Observable Non-Deterministic} (\FOND) planning domain models, focusing on goals on finite traces expressed in \textit{Linear Temporal Logic} (\LTLf) and \textit{Pure} \textit{Past Linear Temporal Logic} (\PLTLf).
We develop the first approach capable of recognizing goals in such settings and evaluate it using different \LTLf and \PLTLf goals over six \FOND planning domain models. 
Empirical results show that our approach is accurate in recognizing temporally extended goals in different recognition settings.

\end{abstract}

%--------------------------------------------------------------------
%--------------------------------------------------------------------
\section{Introduction}

\textit{Goal Recognition} is the task of recognizing the intentions of autonomous agents or humans by observing their interactions in an environment. 
Existing work on goal and plan recognition addresses this task over several different types of domain settings, such as plan-libraries \citep{AvrahamiZilberbrand2005}, plan tree grammars~\citep{Geib_PPR_AIJ2009}, classical planning domain models \citep{RamirezG_IJCAI2009,RamirezG_AAAI2010,Sohrabi_IJCAI2016,PereiraOM_AIJ_2020}, stochastic environments~\citep{RamirezG_IJCAI2011}, continuous domain models~\citep{Kaminka_18_AAAI}, incomplete discrete domain models~\citep{PereiraPM_ICAPS_19}, and approximate control models~\citep{PereiraVMR_IJCAI19}. 
Despite the ample literature and recent advances, most existing approaches to \textit{Goal Recognition as Planning} cannot recognize \textit{temporally extended goals}, i.e., goals formalized in terms of time, e.g., the exact order that a set of facts of a goal must be achieved in a plan. 
Recently, \cite{KR2021Eva} propose a general formulation of a temporal inference problem in deterministic planning settings. However, most of these approaches also assume that the observed actions' outcomes are deterministic and do not deal with unpredictable, possibly adversarial, environmental conditions.

Research on planning for \textit{temporally extended goals} in \textit{deterministic} and \textit{non-deterministic} domain settings has increased over the years, starting with the pioneering work on planning for temporally extended goals \citep{bacchus1998planning} and on planning via model checking \citep{cimatti1997planning}. 
This continued with the work on integrating \LTL goals into planning tools \citep{PatriziLGG_IJCAI11,PatriziLG_IJCAI13}, and, most recently, the work of \cite{ICAPS_2023_PLTL_Planning}, introducing a novel Pure-Past Linear Temporal Logic encoding for planning in the \textit{Classical Planning} setting. % FRM - You mention the work of Bonassi again in the related work. This is a bit redundant, especially since in the related work you don't say much more than this. Perhaps remove one of the sentences.
Other existing work relate \LTL goals with \textit{synthesis} for planning in non-deterministic domain models, often focused on the \emph{finite trace} variants of \LTL \citep{DegVa13,DegVa15,CTMBM17,CamachoBMM18,DeGiacomoS18,aminof2020stochastic}.

In this paper, we introduce the task of goal recognition in \textit{discrete domains} that are \textit{fully observable}, and the outcomes of actions and observations are \textit{non-deterministic}, possibly adversarial, i.e., \textit{Fully Observable Non-Deterministic} (\FOND), allowing the formalization of \textit{temporally extended goals} using two types of temporal logic on finite traces: \textit{Linear-time Temporal Logic} (\LTLf) and \textit{Pure-Past Linear-time Temporal Logic} (\PLTLf)~\citep{ijcai2020surveyddfr}. 

The main contribution of this paper is three-fold. First, based on the definition of \textit{Plan Recognition as Planning} introduced in \citep{RamirezG_IJCAI2009}, we formalize \textit{the problem of recognizing temporally extended goals} (expressed in \LTLf or \PLTLf) in \FOND planning domains, handling both stochastic (i.e., strong-cyclic plans) and adversarial (i.e., strong plans) environments~\citep{aminof2020stochastic}. 
Second, we extend the probabilistic framework for goal recognition proposed in \citep{RamirezG_AAAI2010}, and develop a novel \textit{probabilistic approach} that reasons over executions of policies and returns a posterior probability distribution for the goal hypotheses.
Third, we develop a \textit{compilation approach} that generates an augmented \FOND planning problem by compiling temporally extended goals together with the original planning problem.
This compilation allows us to use any off-the-shelf \FOND planner
to perform the recognition task in \FOND planning models with temporally extended goals.

We focus on \FOND domain models with stochastic non-determinism, and conduct an extensive set of experiments with different complex planning problems. 
We empirically evaluate our approach using different \LTLf and \PLTLf goals over six \FOND planning domain models, including a real-world non-deterministic domain model~\citep{nebel13_tidyup_aaaiirs}, and our experiments show that our approach is accurate to recognize temporally extended goals in different two recognition settings: \textit{offline recognition}, in which the recognition task is performed in ``one-shot'', and the observations are given at once and may contain missing information; and \textit{online recognition}, in which the observations are received incrementally, and the recognition task is performed gradually. %, step-by-step, by taking into account the observations in an incremental way. % FRM - This second part of the sentence is repetitive, you already said things were incremental, so no need to say again that they are done in an incremental way

% The remainder of this paper is organized as follows. In Section~\ref{sec:preliminaries}, we present the preliminaries, describing essential background on Linear-time Temporal Logics (\LTL) on finite traces and \FOND planning. Section~\ref{sec:fond_planning_LTLPLTL} introduces our compilation that generates augmented \FOND planning tasks by compiling temporally extended goals. In Sections~\ref{sec:goal_recognition_FOND_LTL}~and~\ref{sec:gr_solution_approach}, we formally define the task of goal recognition in \FOND planning domains with \LTLf and \PLTLf goals and develop a novel probabilistic approach that is capable of recognizing temporal extended goals in this setting. We empirically evaluate our recognition approach in Section~\ref{sec:experiments_evaluation}, which shows how accurate and effective our approach is for recognizing temporally extended goals in \FOND planning settings.
% In Section \ref{sec:related_work}, we survey related work on \FOND Planning for temporally extended goals and \textit{Goal Recognition as Planning}. Finally, in Section~\ref{sec:conclusions}, we conclude this paper by discussing the limitations and future directions of this work.

%--------------------------------------------------------------------
%--------------------------------------------------------------------
\section{Preliminaries}\label{sec:preliminaries}

In this section, we briefly recall the syntax and semantics of \textit{Linear-time Temporal Logics} on finite traces (\LTLf/\LTLp) and revise the concept and terminology of \FOND planning.

%--------------------------------------------------------------------
\subsection{\LTLftitle and \LTLptitle}

\textit{Linear Temporal Logic on finite traces} (\LTLf) is a variant of \LTL introduced in \citep{Pnueli77} interpreted over \textit{finite traces}. Given a set $AP$
of atomic propositions, the syntax of \LTLf formulas $\varphi$ is defined as follows: 
$$
\varphi ::= a \mid \lnot \varphi \mid \varphi\land \varphi \mid \Next\varphi \mid \varphi\lUntil\varphi
$$
where $a$ denotes an atomic proposition in $AP$, $\Next$ is the
\textit{next} operator, and $\lUntil$ is the \textit{until} operator.
Apart from the Boolean connectives, we use the following abbreviations:
\textit{eventually} as $\Diamond\varphi \doteq \true\lUntil\varphi$;
\textit{always} as $\Box\varphi \doteq\lnot\Diamond\lnot\varphi$; \textit{weak next}
$\Wnext\varphi \doteq \lnot\Next\lnot\varphi$.
A trace $\trace = \trace_0 \trace_1 \cdots$ is a sequence of propositional interpretations, where $\trace_m \in 2^{AP} (m \geq 0)$ is the $m$-th interpretation of $\trace$, and $|\trace|$ is the length of $\trace$. 
We denote a finite trace formally as $\trace \in (2^{AP})^*$.
Given a finite trace $\trace$ and an \LTLf formula $\varphi$, we inductively define when $\varphi$ \textit{holds} in $\trace$ at position $i$ $(0 \leq i < |\trace|)$, written $\trace, i \models \varphi$ as follows:

\begin{itemize}\itemsep=0pt
	\item $\trace, i \models a \tiff a \in \trace_i\nonumber$;
	\item $\trace, i \models \lnot \varphi \tiff \trace, i \not\models \varphi\nonumber$;
	\item $\trace, i \models \varphi_1 \land \varphi_2 \tiff \trace, i \models \varphi_1 \tm{and} \trace, i \models \varphi_2\nonumber$;
	\item $\trace, i \models \Next\varphi \tiff i+1 < |\trace| \tm{and} \trace,i+1 \models \varphi$;
	\item $\trace, i \models \varphi_1 \lUntil \varphi_2$ iff there exists $j$ such that $i\le j < |\trace|$ and $\trace,j \models\varphi_2$, and for all $k, ~i\le k < j$, we have  $\trace, k \models \varphi_1$.
\end{itemize}

An \LTLf formula $\varphi$ is \textit{true} in $\trace$, denoted by $\trace \models \varphi$, iff $\trace,0 \models \varphi$. As advocated in \citep{ijcai2020surveyddfr}, we also use the \textit{pure-past} version of \LTLf, here denoted as \LTLp, due to its compelling computational advantage compared to \LTLf when goal specifications are \textit{naturally} expressed in a past fashion. \LTLp refers \textit{only} to the past and has a natural interpretation on finite traces: formulas are satisfied if they hold in the current (i.e., last) position of the trace.

% We start with
% \LTLp has been previously studied only for technical purposes in the context of \LTL and \LTLf
% \citep{LichtensteinPZ85,MP90}. Here, instead, we focus our investigation on  \LTLp, and its extension \LDLp, themselves.

%In this paper, we consider the version of \LTLf and \LDLf
%that refer to the past instead of the future, namely Past \LTLf (\LTLp) and Past \LDLf (\PLDL).

%The syntax of \LTLp is exactly the same of the one seen in Section \ref{ltl-syntax} for \LTL and in Section \ref{ltlf-syntax} for \LTLf except for past temporal operators that are the inverse of the future ones. As stated before, 

%\LTLp formulas are built on top from a set $\P$ of propositional symbols and are closed under the boolean connectives, the unary temporal operator \Yesterday (\textit{previous-time}) and the binary operator $\Since$ (\textit{since}). Formulas can be defined as follows:

Given a set $AP$ of propositional symbols, \LTLp formulas are defined by:
$$\varphi ::= a \mid \lnot \varphi \mid \varphi \land \varphi  \mid \Yesterday\varphi \mid \varphi \Since \varphi$$
where $a\in AP$, $\Yesterday$ is the \textit{before} operator, and $\Since$ is the \textit{since} operator. 
Similarly to \LTLf, common abbreviations are the \textit{once} operator $\past\varphi \doteq \true \Since\varphi$ 
and the \textit{historically} operator $\gpast\varphi \doteq\lnot\past\lnot\varphi$. 
Given a finite trace $\trace$ and a \LTLp formula $\varphi$, we inductively define when $\varphi$ \textit{holds} in $\trace$ at position $i$ $(0 \leq i < |\trace|)$, written $\trace, i \models \varphi$ as follows. For atomic propositions and Boolean operators it is as for \LTLf. For past operators:
\begin{itemize}
    \item $\trace,i \models  \Yesterday \varphi$ \tiff $i-1 \ge 0$ and $\trace,i-1 \models \varphi$;
    \item   $\trace,i \models  \varphi_1 \Since \varphi_2$ \tiff there exists $k$ such that $0 \leq k \leq i$ and $\trace,k \models \varphi_2$, and for all $j$, $ k<j\leq i$, we have $\trace,j \models \varphi_1$.
\end{itemize}
  
A \LTLp formula $\varphi$ is \textit{true} in $\trace$, denoted by $\trace \models \varphi$, if and only if $\trace,  |\trace|-1 \models \varphi$. 
A key property of temporal logics that we exploit in this work is that, for every \LTLf/\LTLp formula $\varphi$, there exists a \textit{Deterministic Finite-state Automaton} (\DFA) $\A_{\varphi}$ accepting the traces $\trace$ satisfying $ \varphi$~ \citep{DegVa13,ijcai2020surveyddfr}.

%--------------------------------------------------------------------
\subsection{\FONDtitle Planning}

\begin{figure*}[ht!]
	\centering
	\begin{subfigure}[b]{0.49\linewidth}
		\centering
 	    \includegraphics[width=0.6\linewidth]{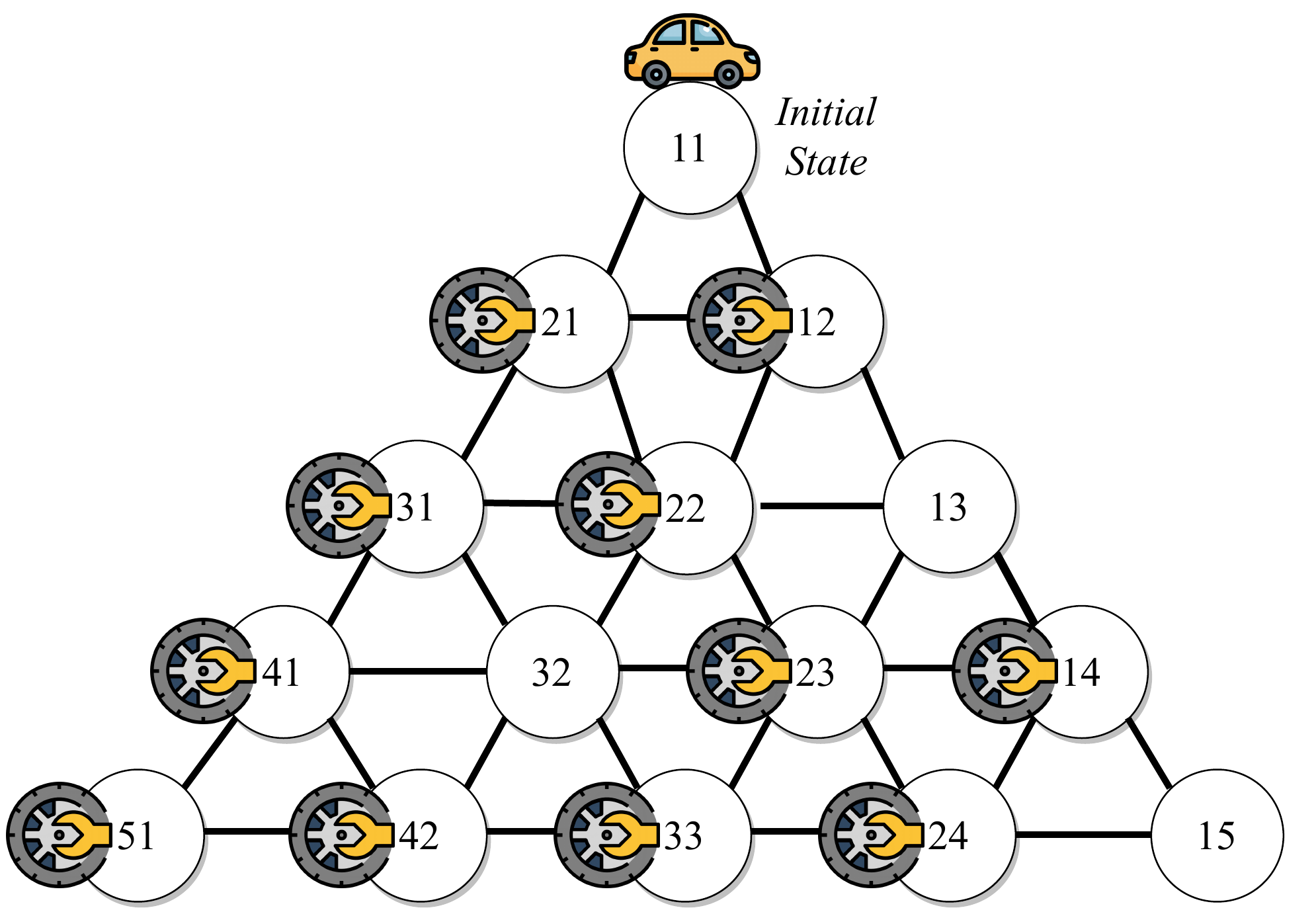}
		\caption{FOND problem example.}
		\label{fig:Triangle_Tireworld_Example}
	\end{subfigure}
	\begin{subfigure}[b]{0.49\linewidth}
		\centering
		\includegraphics[width=0.6\linewidth]{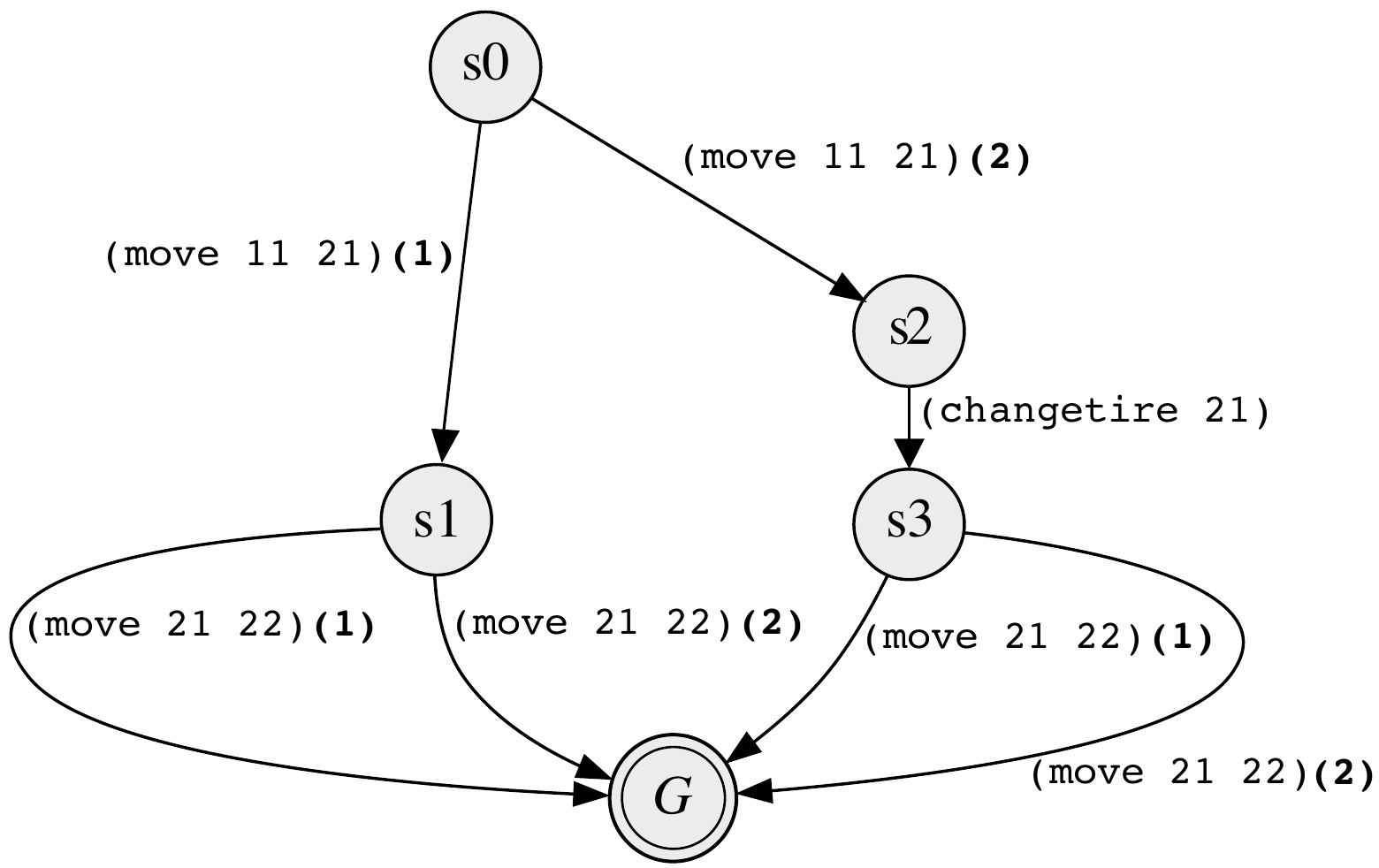}
		\caption{$\pi$ to achieve 22 from 11.}
		\label{fig:Triangle_Tireworld_Policy}
	\end{subfigure}
	\caption{\textsc{Triangle-Tireworld} domain and policy.}
\end{figure*}
% Given that the environment is non-deterministic, a strong-cyclic policy $\pi$ may induce to a set of plans to achieve a goal state $G$ from an initial state $s_{0}$. 

%Following~\citep{GeffnerBonet13_PlanningBook}, a \textit{Fully Observable Non-deterministic Domain} planning model (\FOND) is a tuple $\D = \tup{2^{\F}, \A, s_0, \alpha, tr}$, where $\F$ is a set of \textit{fluents} (i.e., atomic propositions), and 2$^{F}$ is the set of possible states; $A$ is a set of \textit{actions}; $s_0$ is the initial state (initial assignment to fluents); $\alpha(s) \subseteq \A$ represents \textit{action preconditions}; and $(s, a, s') \in tr$ with $a \in \alpha(s)$ represents \textit{action effects} (including frame assumption). We assume the domain $\D$ to be compactly represented (e.g., in \PDDL), assuming the size of the domain as the cardinality of $\F$. Intuitively, a \FOND domain model evolves as follows: from a given state $s$, the agent chooses an action $a \in \alpha(s)$, after which the environment chooses the successor states $s'$ with $(s, a, s') \in tr$.

%We formally define \textit{fully-observable non-deterministic domain} (\FOND) models following the formalism in~\citep{GeffnerBonet13_PlanningBook}, as follows in Definition~\ref{def:fond_domain}.

A \textit{Fully Observable Non-deterministic Domain} planning model (\FOND) is a tuple $\D = \tup{2^{\F}, A, \alpha, tr}$~\citep{GeffnerBonet13_PlanningBook}, where $2^{\F}$ is the set of possible states and $\F$ is a set of fluents (atomic propositions); $A$ is the set of actions; $\alpha(s) \subseteq A$
is the set of applicable actions in a state $s$; and $tr(s, a)$ is the non-empty set of successor states that follow action $a$ in state $s$. 
A domain $\D$ is assumed to be compactly represented (e.g., in \PDDL~\citep{PDDLMcdermott1998}), hence its size is $|\F|$.
Given the set of \textit{literals} of $\F$ as $\mathit{Literals}(\F) = \F \cup \{ \lnot f \mid f \in \F \}$, every action $a \in A$ is usually characterized by $\tup{\Pre_a, \Eff_a}$, where $\Pre_a \subseteq \mathit{Literals}(\F)$ is the action preconditions, and $\Eff_a$ is the action effects. An action $a$ can be applied in a state $s$ if the set of fluents in $\Pre_a$ holds true in $s$. The result of applying $a$ in $s$ is a successor state $s'$ non-deterministically drawn from one of the $\Eff^{i}_{a}$ in $\Eff_a = \{ \Eff^{1}_{a}, ..., \Eff^{n}_{a} \}$.
In \FOND planning, some actions have \textit{uncertain outcomes}, such that they have
\textit{non-deterministic} effects (i.e., $|tr(s, a)| \geq 1$ in all states $s$ in which $a$ is applicable), and effects cannot be predicted in advance. 
\PDDL expresses uncertain outcomes using the \pred{oneof}~\citep{IPC6} keyword, as widely used by several \FOND planners~\citep{MyND_MattmullerOHB10,Muise12ICAPSFond}. 
We define \FOND planning problems as follows.

\begin{definition}\label{def:fond_problem}
\it
A \FOND planning problem is a tuple $\P = \tup{\D, s_{0}, G} $, where $\D$ is a \FOND domain model, $s_{0}$ is an initial assignment to fluents in $\F$ (i.e., initial state), and $G \subseteq \F$ is the goal state. 
\end{definition}

Solutions to a \FOND planning problem $\P$ are \textit{policies}. 
A policy is usually denoted as $\pi$,
and formally defined as a partial function $\pi: 2^{\F} \to A$ mapping \textit{non-goal} states into applicable actions that eventually reach the goal state $G$ from the initial state $s_{0}$.
A \textit{policy} $\pi$ for $\P$ induces a set of possible \textit{executions} $\vec{E} = \{ \vec{e}_1, \vec{e}_2, \dots \}$, that are state trajectories, possibly finite (i.e., histories) $(s_{0},\dots, s_{n})$, where $s_{i+1} \in tr(s_i, a_i)$ and $a_i \in \alpha(s_i)$ for $i = 0,\dots, n-1$, or possibly infinite $s_{0},s_{1},\dots$, obtained by choosing some possible outcome of actions instructed by the policy. A policy $\pi$ is a solution to $\P$ if every generated execution is such that it is finite and satisfies the goal $G$ in its last state, i.e., $s_{n} \models G$. 
In this case, we say that $\pi$ is \textit{winning}.
\cite{CimattiPRT03} define three solutions to \FOND planning problems: \textit{weak, strong} and \textit{strong-cyclic} solutions.
We formally define such solutions in Definitions~\ref{def:weak_solution}, \ref{def:strong_solution}, and~\ref{def:strong_cyclic_solution}.

\begin{definition}\label{def:weak_solution}
\it
A \textbf{weak solution} is a policy that achieves the goal state $G$ from the initial state $s_{0}$ under at least one selection of action outcomes; namely, such solution will have some chance of achieving the goal state $G$.
\end{definition}

\begin{definition}\label{def:strong_cyclic_solution}
\it
A \textbf{strong-cyclic solution} is a policy that guarantees to achieve the goal state $G$ from the initial state $s_{0}$ only under the assumption of fairness\footnote{The fairness assumption defines that all action outcomes in a given state have a non-zero probability.}. However, this type of solution may revisit states, so the solution cannot guarantee to achieve the goal state $G$ in a fixed number of steps.
\end{definition}

\begin{definition}\label{def:strong_solution}
\it
A \textbf{strong solution} is a policy that is guaranteed to achieve the goal state $G$ from the initial state $s_{0}$ regardless of the environment's non-determinism. This type of solution guarantees to achieve the goal state $G$ in a finite number of steps while never visiting the same state twice.
\end{definition}

In this work, we focus on \textit{strong-cyclic solutions}, where the environment acts in an unknown but stochastic way. 
Nevertheless, our recognition approach applies to strong solutions as well, where the environment is purely adversarial (i.e., the environment may always choose effects against the agent). 

As a running example, we use the well-known \FOND domain model called \textsc{Triangle-Tireworld}, where locations are connected by roads, and the agent can drive through them. 
The objective is to drive from one location to another. 
However, while driving between locations, a tire may go flat, and if there is a spare tire in the car's location, then the car can use it to fix the flat tire.
Figure~\ref{fig:Triangle_Tireworld_Example} illustrates a \FOND planning problem for the \textsc{Triangle-Tireworld} domain, where circles are locations, arrows represent roads, spare tires are depicted as tires, and the agent is depicted as a car.
Figure~\ref{fig:Triangle_Tireworld_Policy} shows a policy $\pi$ to achieve location 22. Note that, to move from location 11 to location 21, there are two arrows labeled with the action \pred{(move 11 21)}: (1) when moving does not cause the tire to go flat; (2) when moving causes the tire to go flat. The policy depicted in Figure~\ref{fig:Triangle_Tireworld_Policy} guarantees the success of achieving location 22 despite the environment's non-determinism. 

In this work, we assume from \textit{Classical Planning} that the cost is 1 for all \textit{non-deterministic} instantiated actions $a \in A$.
In this example, policy $\pi$, depicted in Figure~\ref{fig:Triangle_Tireworld_Policy}, has two possible finite executions in the set of executions $\vec{E}$, namely $\vec{E} = \{ \vec{e}_{0}, \vec{e}_{1} \}$, such as: 
\begin{itemize}
	\item $\vec{e}_{0}$: $[$\pred{(move 11 21)}, \pred{(move 21 22)}$]$; and 
	\item $\vec{e}_{1}$: $[$\pred{(move 11 21)}, \pred{(changetire 21)}, \pred{(move 21 22)}$]$.
\end{itemize}

%--------------------------------------------------------------------

%--------------------------------------------------------------------
\section{FOND Planning for \LTLftitle and \LTLptitle Goals}\label{sec:fond_planning_LTLPLTL}

We base our approach to goal recognition in \FOND domains for \textit{temporally extended goals} on \FOND planning for \LTLf and \LTLp goals \citep{CTMBM17,CamachoBMM18,DeGiacomoS18}. 
We formally define a \FOND planning problem for \LTLf/\PLTLf goals in Definition~\ref{def:FOND_planning}, as follows.

\begin{definition}\label{def:FOND_planning}
\it
A \FOND planning problem for \LTLf/\PLTLf goals is a tuple $\Gamma = \tup{\D, s_0, \varphi}$, where $\D$ is a standard \FOND domain model, $s_0$ is the initial state, and $\varphi$ is a goal formula, formally represented either as an \LTLf or a \PLTLf formula.
\end{definition}

In \FOND planning for temporally extended goals, a policy $\pi$ is a partial function $\pi : (2^\F)^{+} \to A$ mapping \textit{histories}, i.e., \textit{states} into applicable actions. A policy $\pi$ for $\Gamma$ achieves a temporal formula $\varphi$ if and only if the sequence of states generated by $\pi$, despite the non-determinism of the environment, is accepted by $\A_\varphi$.

Key to our recognition approach is using off-the-shelf \FOND planners for standard reachability goals to handle also temporally extended goals through an encoding of the automaton for the goal into an extended planning domain expressed in \PDDL.
Compiling temporally extended goals into planning domain models has a long history in the \textit{Planning} literature.
In particular, \cite{BaierM06} develops \emph{deterministic} planning with a special first-order quantified \LTL goals on finite-state sequences. 

Their technique encodes a \textit{Non-Deterministic Finite-state Automaton} (\NFA), resulting from \LTL formulas, into deterministic planning domains for which \textit{Classical Planning} technology can be leveraged. Our parameterization of objects of interest is somehow similar to their approach. % FRM - The somehow part is too fuzzy, can you be more specific?

Starting from \cite{BaierM06}, always in the context of deterministic planning, \cite{TorresBaier15}~proposed a polynomial-time compilation of \LTL goals on finite-state sequences into alternating automata, leaving non-deterministic choices to be decided at planning time.
Finally, \cite{CTMBM17,CamachoBMM18} built upon \cite{BaierM06} and \cite{TorresBaier15}, proposing a compilation in the context of \FOND domain models that simultaneously determinizes on-the-fly the \NFA for \LTLf and encodes it into \PDDL. However, this encoding introduces a lot of bookkeeping machinery due to the removal of any form of angelic non-determinism mismatching with the devilish non-determinism of \PDDL for \FOND.

Although inspired by these works, our approach differs in several technical details.
% Our technique is inspired by these works, though it differs in several technical details.
We encode the \DFA directly into a non-deterministic \PDDL planning domain by taking advantage of the \emph{parametric} nature of \PDDL domains that are then instantiated into propositional problems when solving a specific task.
%
%Our technique to solve \FOND planning for \LTLf and \PLTLf stems from researches in \citep{BaierM06,CTMBM17}. 
%Unlike previous approaches, we employ the encoding of DFAs instead of NFAs, and, instead of checking for nonemptiness of the product automaton between the domain automaton and the formula automaton, 
%Instead of directly checking for nonemptiness of the product automaton between the domain automaton and the formula automaton, we take advantage of the factored representation of a planning domain afforded by standard planning languages, like \PDDL.
%
Given a \FOND planning problem $\Gamma$ represented in \PDDL, we transform $\Gamma$ as follows. 
First, we transform the temporally extended goal formula $\varphi$ (formalized either in \LTLf or \PLTLf) into its corresponding \DFA $\A_\varphi$ through the highly-optimized \MONA tool \citep{Mona95}. 
Second, from $\A_\varphi$, we build a \emph{parametric} \DFA (\PDFA), representing the lifted version of the \DFA. 
Finally, the encoding of such a \PDFA into \PDDL yields an augmented \FOND domain model $\Gamma'$. 
Thus, we reduce \FOND planning for \LTLf/\PLTLf to a standard \FOND planning problem solvable by any off-the-shelf \FOND planner.

\subsection{Translation to Parametric DFA}

The use of \textit{parametric} \DFAs is based on the following observations. In temporal logic formulas and, hence, in the corresponding \DFAs, propositions are represented by domain fluents grounded on specific objects of interest.
We can replace these propositions with predicates using object variables and then have a mapping function $m^{obj}$ that maps such variables into the problem instance objects.
In this way, we get a lifted and \textit{parametric} representation of the \DFA, i.e., \PDFA, which is merged with the domain. 
Here, the objective is to capture the entire dynamics of the \DFA within the planning domain model itself. To do so, starting from the \DFA we build a \PDFA whose states and symbols are the lifted versions of the ones in the \DFA.
Formally, to construct a \PDFA we use a mapping function $m^{obj}$, which maps the set of objects of interest present in the \DFA to a set of \emph{free} variables. Given the mapping function $m^{obj}$, we can define a \PDFA as follows.

\begin{definition}
\it
Given a set of object symbols $\O$, and a set of free variables $\V$, we define a mapping function $m$ that maps each object in $\O$ with a free variable in $\V$.
\end{definition}

Given a \DFA and the objects of interest for $\Gamma$, we can construct a \PDFA as follows:
% Given the mapping function $m^{obj}$, we can formally define a \PDFA as a tuple $\A^{p}_\varphi = \tup{\Sigma^{p}, Q^{p}, q^{p}_0, \delta^{p}, F^{p}}$, where: $\Sigma^{p} = \{ \sigma^p_0, \dots, \sigma^p_n \} = 2^{\F}$ is the alphabet of fluents; $Q^{p}$ is a nonempty set of parametric states; $q^{p}_0$ is the parametric initial state; $\delta^{p}: Q^{p} \times \Sigma^{p} \rightarrow Q^{p}$ is the parametric transition function; $F^{p} \subseteq Q^{p}$ is the set of parametric final states. $\Sigma^{p}, Q^{p}, q^{p}_0, \delta^{p}$ and $F^{p}$ can be obtained by applying $m^{obj}$ to all the components of the corresponding \DFA.
\begin{definition}
\it
A \PDFA is a tuple $\A^{p}_\varphi = \tup{\Sigma^{p}, Q^{p}, q^{p}_0, \delta^{p}, F^{p}}$, where: $\Sigma^{p} = \{ \sigma^p_0, ..., \sigma^p_n \} = 2^{\F}$ is the alphabet of fluents; $Q^{p}$ is a nonempty set of parametric states; $q^{p}_0$ is the parametric initial state; $\delta^{p}: Q^{p} \times \Sigma^{p} \rightarrow Q^{p}$ is the parametric transition function; $F^{p} \subseteq Q^{p}$ is the set of parametric final states. $\Sigma^{p}, Q^{p}, q^{p}_0, \delta^{p}$ and $F^{p}$ can be obtained by applying $m^{obj}$ to all the components of the corresponding \DFA.
\end{definition}

\begin{examplebold}
\it
Given the \LTLf formula ``$\Diamond(vAt ~51)$'', the object of interest ``51'' is replaced by the object variable $x$ (i.e., $m^{obj}(51) = x$), and the corresponding \DFA and \PDFA for this \LTLf formula are depicted in Figures~\ref{fig:ex-DFA}~and~\ref{fig:ex-PDFA}.
\end{examplebold}

\vspace{-5mm}

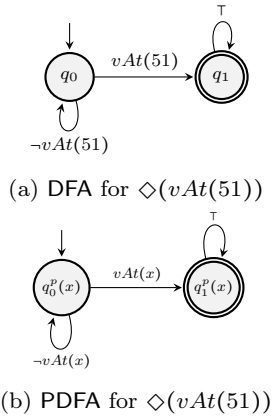
\begin{figure}[ht]
\centering
	\begin{subfigure}[b]{\linewidth}
		\centering
		\begin{tikzpicture}[scale=1, every node/.style={scale=0.8},shorten >=1pt,node distance=2cm,on grid,auto] 
  			\node[state,initial above,initial text=] (q_0) {$q_0$}; 
  			\node[state, accepting] (q_1) [right=of q_0] {$q_1$}; 
   			\path[->] 
   			(q_0) 
     			edge [loop below] node {$\lnot vAt(51)$} (q_0)
     			edge node {$vAt(51)$} (q_1)
   			(q_1) 
    			edge [loop above] node {$\top$} (q_1);
		\end{tikzpicture}
		\caption{\DFA for $\Diamond(vAt(51))$} 
		\label{fig:ex-DFA} 
	\end{subfigure}
	\begin{subfigure}[b]{\linewidth}
		\centering
		\begin{tikzpicture}[scale=1, every node/.style={scale=0.65},shorten >=1pt,node distance=2cm,on grid,auto] 
  			\node[state,initial above,initial text=] (q_0) {$q^p_0(x)$}; 
  			\node[state, accepting] (q_1) [right=of q_0] {$q^p_1(x)$}; 
   			\path[->] 
   			(q_0) 
     			edge [loop below] node {$\lnot vAt(x)$} (q_0)
     			edge node {$vAt(x)$} (q_1)
   			(q_1) 
    			edge [loop above] node {$\top$} (q_1);
		\end{tikzpicture}
		\caption{\PDFA for $\Diamond(vAt(51))$} 
		\label{fig:ex-PDFA} 
	\end{subfigure}
	\caption{\DFA and \PDFA for $\Diamond(vAt(51))$.}
\end{figure}

When the resulting new domain is instantiated, we implicitly get back the original \DFA in the Cartesian product with the original instantiated domain. 
Note that this way of proceeding is similar to what is done in \citep{BaierM06}, where they handle \LTLf goals expressed in a special \FOL syntax, with the resulting automata (non-deterministic B\"uchi automata) 
parameterized by the variables in the \LTLf formulas.

\subsection{PDFA Encoding in PDDL}

Once the \PDFA has been computed, we encode its components within the planning problem $\Gamma$, specified in \PDDL, thus, producing an augmented \FOND planning problem $\Gamma' = \tup{\D', s'_0, G'}$, where $\D' = \tup{2^{\F'}, A', \alpha', tr'}$ and $G'$ is a propositional goal as in \textit{Classical Planning}. 
Intuitively, additional parts of $\Gamma'$ are used to synchronize the dynamics between the domain and the automaton sequentially. Specifically, $\Gamma'$ is composed of the following components.

\subsubsection*{Fluents}
$\F'$ has the same fluents in $\F$ plus fluents representing each state of the \PDFA, and a fluent called \pred{turnDomain}, which controls the alternation between domain's actions and the \PDFA's synchronization action. Formally,
$\F' = \F \cup \{ q \mid q \in Q^p \} \cup \{\pred{turnDomain} \}$.

\subsubsection*{Domain Actions}
Actions in $A$ are modified by adding \pred{turnDomain} in preconditions and the negated \pred{turnDomain} in effects: $\Pre^\prime_a = \Pre_a \cup \{ \pred{turnDomain} \}$ and $\Eff^\prime_a = \Eff_a \cup$ $\{ \lnot \pred{turnDomain} \}$ for all $a \in A$.

\subsubsection*{Transition Operator}
The \emph{transition} function $\delta^{p}$ of a \PDFA is encoded as a new domain operator with conditional effects, called \pred{trans}. 
Namely, we have $\Pre_{\pred{trans}} = \{\lnot \pred{turnDomain} \}$ and $\Eff_{\pred{trans}} = \{ \pred{turnDomain} \} \cup \{ \pred{\textbf{when}}~ (q^p, \sigma^p), \pred{\textbf{then}}~ \delta^p(q^p, \sigma^p) \cup \{ \lnot q \mid q \neq q^p, q \in Q^p \} \}$, for all $(q^p, \sigma^p) \in \delta^p$.
To exemplify how the transition \PDDL operator is obtained, Listing~\ref{lst:trans-op} reports the transition operator for the \PDFA in Figure~\ref{fig:ex-PDFA}.

\begin{lstlisting}[language=pddl, escapechar=£, caption={Transition \PDDL operator for $\Diamond(vAt(x))$}, captionpos=b, label={lst:trans-op}, basicstyle=\scriptsize]
(:action trans
  :parameters (?x - location)
  :precondition (not (turnDomain))
  :effect (and 
  	(when (and (q0 ?x) (not (vAt ?x)))
          (and (q0 ?x) (not (q1 ?x)) (turnDomain))
    (when (or (and (q0 ?x) (vAt ?x)) (q1 ?x))£\label{line:cond-eff-ex}£
              (and (q1 ?x) (not (q0 ?x)) 
			  (turnDomain))))
\end{lstlisting}

\subsubsection*{Initial and Goal States}

The new initial condition is specified as $s'_0 = s_0  \cup \{ q^{p}_0 \}  \cup \{\pred{turnDomain} \}$. 
This comprises the initial condition of the previous domain $D$ ($s_0$) plus the initial state of the \PDFA and the predicate \pred{turnDomain}. 
Considering the example in Figure~\ref{fig:Triangle_Tireworld_Example} and the \PDFA in Figure~\ref{fig:ex-PDFA}, the new initial condition is as follows in \PDDL:

\begin{lstlisting}[language=pddl, escapechar=£, caption={\PDDL initial condition for $\varphi = \Diamond(vAt(51))$}, captionpos=b, label={lst:init}, basicstyle=\scriptsize]
(:init (and (road 11 21) (road 11 21) ...
            (spare-in 21) (spare-in 12) ...
            (q0 51) (turnDomain)))
\end{lstlisting}

The new goal condition is specified as $G' = \{ \bigvee q_i \mid q_i \in F^{p} \}  \cup \{\pred{turnDomain} \}$, i.e., we want the \PDFA to be in one of its accepting states and \pred{turnDomain}, as follows:

\begin{lstlisting}[language=pddl, escapechar=£, caption={\PDDL goal condition for $\varphi = \Diamond(vAt(51))$}, captionpos=b, label={lst:goal}, basicstyle=\scriptsize]
(:goal (and (q1 51) (turnDomain)))
\end{lstlisting}

We note that, both in the initial and goal conditions of the new planning problem, \PDFA states are grounded back on the objects of interest thanks to the inverse of the mapping $m^{obj}$.

Executions of a policy for our new \FOND planning problem $\Gamma'$ are $\vec{e}': [a'_1, t_1, a'_2, t_2, \dots, a'_n,t_n]$, where $a'_i \in A'$ are the real domain actions, and $t_1, \dots, t_n$ are sequences of synchronization \pred{trans} actions, which, at the end, can be easily removed to extract the desired execution $\vec{e} :[a'_1, a'_2, \dots, a'_n]$. 
In the remainder of the paper, we refer to the compilation just exposed as \FONDforLTLPLTL.

\subsubsection*{Theoretical Property of the PDDL Encoding}
We now study the theoretical properties of the encoding presented in this section. 
Theorem~\ref{thm:fond-planning} states that solving \FOND planning for \LTLf/\LTLp goals amounts to solving standard \FOND planning problems for reachability goals. A policy for the former can be easily derived from a policy of the latter.

\begin{theorem}\label{thm:fond-planning}
Let $\Gamma$ be a \FOND planning problem with an \LTLf/\PLTLf goal $\varphi$, and $\Gamma^\prime$ be the compiled \FOND planning problem with a reachability goal state. Then, $\Gamma$ has a policy $\policy: (2^\F)^+ \to A$ iff $\Gamma^\prime$ has a policy $\policy^\prime: (2{^\F}^\prime)^+ \to A^\prime$.
\end{theorem}
\begin{proof}
($\longrightarrow$).
We start with a policy $\pi$ of the original problem that is winning by assumption. Given $\pi$, we can always build a new policy, which we call $\pi^\prime$, following the encoding presented in Section~3 of the paper.
The newly constructed policy will modify histories of $\pi$ by adding fluents and an auxiliary deterministic action $\pred{trans}$, both related to the \DFA associated with the \LTLf/\PLTLf formula $\varphi$. 
Now, we show that $\pi^\prime$ is an executable policy and that is winning for $\Gamma^\prime$.
To see the executability, we can just observe that, by construction of the new planning problem $\Gamma^\prime$, all action effects $\Eff_{a^\prime}$ of the original problem $\Gamma$ are modified in a way that all action effects of the original problem $\Gamma$ are not modified and that the auxiliary action $\pred{trans}$ only changes the truth value of additional fluents given by the \DFA $\A^p_\varphi$ (i.e., automaton states).
Therefore, the newly constructed policy $\pi^\prime$ can be executed.
To see that $\pi^\prime$ is winning and satisfies the \LTLf/\PLTLf goal formula $\varphi$, we reason about all possible executions. For all executions, every time the policy $\pi^\prime$ stops we can always extract an induced state trajectory of length $n$ such that its last state $s^\prime_n$ will contain one of the final states $F^p$ of the automaton $\A^p_\varphi$. This means that the induced state trajectory is accepted by the automaton $\A^p_\varphi$.
Then, by Theorem~\cite{DegVa13,ijcai2020surveyddfr}, we have that $\trace\models\varphi$. 

($\longleftarrow$).
From a winning policy $\pi^\prime$ for the compiled problem, we can always project out all automata auxiliary $\pred{trans}$ actions
obtaining a corresponding policy $\pi$. We need to show that the resulting policy $\pi$ is winning, namely, it can be successfully executed on the original problem $\Gamma$ and satisfies the \LTLf/\PLTLf goal formula $\varphi$. The executability follows from the fact that the deletion of $\pred{trans}$ actions and related auxiliary fluents from state trajectories induced by $\pi$ does not modify any precondition/effect of original domain actions (i.e., $a\in\A$). Hence, under the right preconditions, any domain action can be executed. Finally, the satisfaction of the \LTLf/\PLTLf formula $\varphi$ follows directly from Theorem~\cite{DegVa13,ijcai2020surveyddfr}. Indeed, every execution of the winning policy $\pi^\prime$ stops when reaching one of the final states $F^p$ of the automaton $\A^p_\varphi$ in the last state $s_n$,
thus every execution of $\pi$ would satisfy $\varphi$. Thus, the thesis holds.
\end{proof}

%--------------------------------------------------------------------
%--------------------------------------------------------------------
\section{Goal Recognition in \FONDtitle~Planning Domains with \LTLftitle and \LTLptitle Goals}\label{sec:goal_recognition_FOND_LTL}

We now introduce our recognition approach that is able to recognizing temporally extended (\LTLf and \PLTLf) goals in \FOND planning domains.
Our approach extends the probabilistic framework of \cite{RamirezG_AAAI2010} to compute posterior probabilities over temporally extended goal hypotheses, by reasoning over the set of possible executions of policies $\pi$ and the observations. Our goal recognition approach works in two stages: the \textit{compilation stage} and the \textit{recognition stage}.
In the next sections, we describe in detail how these two stages work. Figure~\ref{fig:Pipeline} illustrates how our approach works.

\begin{figure*}[!ht]
	\centering 	
	\includegraphics[width=0.7\linewidth]{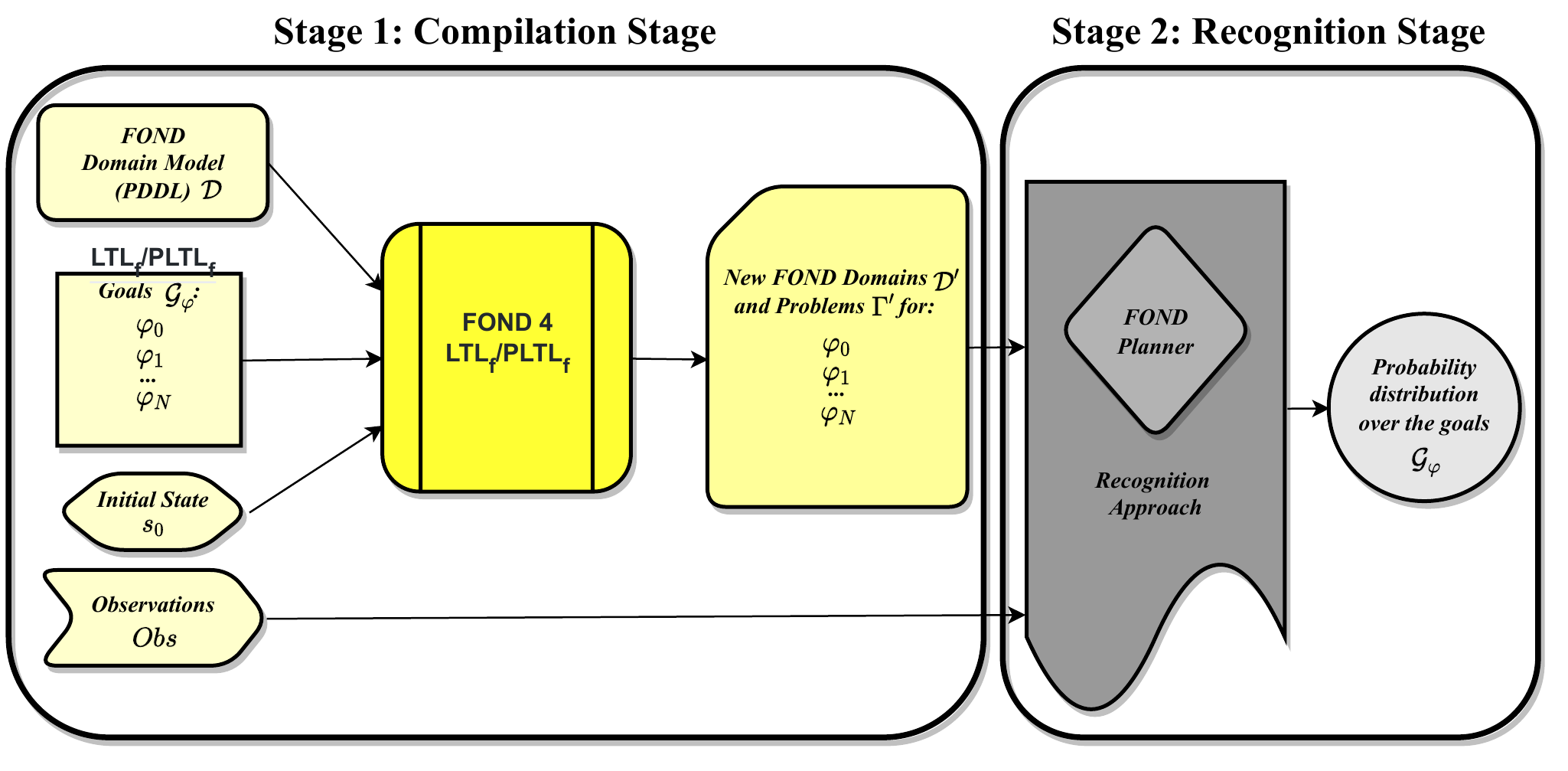}
	\caption{Overview of our solution approach.}
	\label{fig:Pipeline}
\end{figure*}

%--------------------------------------------------------------------
\subsection{Goal Recognition Problem} We define the task of goal recognition in \FOND planning domains with \LTLf and \LTLp goals by extending the standard definition of \textit{Plan Recognition as Planning}~\citep{RamirezG_IJCAI2009}, as follows.

\begin{definition}\label{def:goal_recognition}
\it
A goal recognition problem in a \FOND~planning setting with temporally extended goals (\LTLf and/or \LTLp) is a tuple $\T_{\varphi} = \langle \mathcal{D}, s_0, \mathcal{G}_{\varphi}, Obs\rangle$, where: $\D = \tup{2^{\F}, A, \alpha, tr}$ is a \FOND planning domain; $s_0$ is the initial state; $\mathcal{G}_{\varphi} = \lbrace \varphi_0, \varphi_1, ..., \varphi_n \rbrace$ is the set of goal hypotheses formalized in \LTLf or \LTLp, including the intended goal $\varphi^{*} \in \mathcal{G}_{\varphi}$; $Obs = \tup{o_0, o_1, ..., o_n}$ is a sequence of successfully executed (non-deterministic) actions of a policy $\pi_{\varphi^{*}}$ that achieves the intended goal $\varphi^{*}$, s.t. $o_i \in A$. 
\end{definition}

Since we deal with non-deterministic domain models, an observation sequence $Obs$ corresponds to a successful execution $\vec{e}$ in the set of all possible executions $\vec{E}$ of a \textit{strong-cyclic policy} $\pi$ that achieves the actual intended hidden goal  $\varphi^{*}$. 
In this work, we assume two recognition settings: \textit{Offline Keyhole Recognition}, and \textit{Online Recognition}.
In \textit{Offline Keyhole Recognition} the observed agent is completely unaware of the recognition process~\citep{Armentano_AIJ_2007}, the observation sequence $Obs$ is given at once, and it can be either \textit{full} or \textit{partial}---in a \textit{full observation sequence}, we observe all actions of an agent's plan, whereas, in a \textit{partial observation sequence}, only a sub-sequence thereof. 
By contrast, in \textit{Online Recognition}~\citep{vered2016online}, the observed agent is also unaware of the recognition process, but the observation sequence is revealed incrementally instead of being given in advance and at once, as in \textit{Offline Recognition}, thus making the recognition process an already much harder task.

An ``ideal'' solution for a goal recognition problem comprises a selection of the goal hypotheses containing only the single actual intended hidden goal $\varphi^{*} \in \mathcal{G}$ that the observation sequence $Obs$ of a plan execution achieves \citep{RamirezG_IJCAI2009,RamirezG_AAAI2010}. 
Fundamentally, there is no exact solution for a goal recognition problem, but it is possible to produce a probability distribution over the goal hypotheses and the observations, so that the goals that ``best'' explain the observation sequence are the most probable ones.
We formally define a solution to a goal recognition problem in \FOND planning with temporally extended goals in Definition~\ref{def:goal_recognition_solution}.

\begin{definition}\label{def:goal_recognition_solution}
\it
Solving a goal recognition problem $\T_{\varphi}$ requires selecting a temporally extended goal hypothesis $\hat{\varphi} \in \mathcal{G}_{\varphi}$ such that $\hat{\varphi} = \varphi^{*}$, and it represents how well $\hat{\varphi}$ predicts or explains what observation sequence $Obs$ aims to achieve.
\end{definition}

Existing recognition approaches often return either a probability distribution over the set of goals~\citep{RamirezG_AAAI2010,Sohrabi_IJCAI2016}, or scores associated with each possible goal hypothesis~\citep{PereiraOM_AIJ_2020}. Here, we return a probability distribution $\mathbb{P}$ over the set of temporally extended goals $\mathcal{G}_{\varphi}$ that ``best'' explains the observations sequence $Obs$.

%--------------------------------------------------------------------
\subsection{Probabilistic Goal Recognition}
We now recall the probabilistic framework for \textit{Plan Recognition as Planning} proposed in \cite{RamirezG_AAAI2010}. 
The framework sets the probability distribution for every goal $G$ in the set of goal hypotheses $\G$, and the observation sequence $Obs$ to be a Bayesian posterior conditional probability, as follows:
\begin{align}
\label{eq:posterior}
\mathbb{P}(G \mid Obs) = \eta * \mathbb{P}(Obs \mid G) * \mathbb{P}(G)
\end{align}
\noindent where $\mathbb{P}(G)$ is the \emph{a priori} probability assigned to goal $G$, $\eta$ is a normalization factor inversely proportional to the probability of $Obs$,
and $\mathbb{P}(Obs \mid G)$ is
\begin{align}
\mathbb{P}(Obs \mid G) = \sum_{\pi} \mathbb{P}(Obs \mid \pi) * \mathbb{P}(\pi \mid G)
\label{eq:likelihood}
\end{align}

\noindent $\mathbb{P}(Obs \mid \pi)$ is the probability of obtaining $Obs$ by executing a policy $\pi$ and $\mathbb{P}(\pi \mid G)$ is the probability of an agent pursuing $G$ to select $\pi$. 
Next, we extend the probabilistic framework above to recognize temporally extended goals in \FOND planning domain models.

%--------------------------------------------------------------------
\subsection{Compilation Stage}\label{subsec:compilation_stage}

We perform a \textit{compilation stage} that allows us to use any off-the-shelf \FOND planner to extract policies for temporally extended goals. To this end, we compile and generate new \FOND planning domain models $\Gamma'$ for the set of possible temporally extended goals $\mathcal{G}_{\varphi}$ using the compilation approach described in Section~\ref{sec:fond_planning_LTLPLTL}.  
Specifically, for every goal $\varphi \in \mathcal{G}_{\varphi}$, our compilation takes as input a \FOND planning problem $\Gamma$, where $\Gamma$ contains the \FOND planning domain $\D$ along with an initial state $s_0$ and a temporally extended goal $\varphi$. 
Finally, as a result, we obtain a new \FOND planning problem $\Gamma'$ associated with the new domain $\D'$.
Note that such a new \FOND planning domain $\Gamma'$ encodes new predicates and transitions that allow us to plan for temporally extended goals by using off-the-shelf \FOND planners. 

\begin{corollary}
Let $\mathcal{T}_\varphi$ be a goal recognition problem over a set of \LTLf/\PLTLf goals $\mathcal{G}_\varphi$  and let $\mathcal{T^\prime}$ be the compiled goal recognition problem over a set of propositional goals $\mathcal{G}$. Then, if $\mathcal{T^\prime}$ has a set of winning policies that solve the set of propositional goals in $\mathcal{G}$, then $\mathcal{T}_\varphi$ has a set of winning policies that solve its \LTLf/\PLTLf goals.
\end{corollary}
\begin{proof}
From Theorem~1 we have a bijective mapping between policies of \FOND planning for \LTLf/\PLTLf goals and policies of standard \FOND planning. Therefore, the thesis holds.
\end{proof}

%--------------------------------------------------------------------
\subsection{Recognition Stage}\label{subsec:recognition_stage}

The stage in which we perform the goal recognition task comprises extracting policies for every goal $\varphi \in \mathcal{G}_{\varphi}$. 
From such policies along with observations $Obs$, we compute posterior probabilities for the goals $\mathcal{G}_{\varphi}$ by matching the observations with all possible executions in the set of executions $\vec{E}$ of the policies. 
To ensure compatibility with the policies, we assume the recognizer knows the preference relation over actions for the observed agent when unrolling the policy during search.

\subsubsection*{Computing Policies and the Set of Executions $\vec{E}$ for $\mathcal{G}_{\varphi}$} 

We extract policies for every goal $\varphi \in \mathcal{G}_{\varphi}$ using the new \FOND planning domain models $\Gamma'$, and for each of these policies, we enumerate the set of possible executions $\vec{E}$. 
The aim of enumerating the possible executions $\vec{E}$ for a policy $\pi$ is to attempt to infer what execution $\vec{e} \in \vec{E}$ the observed agent is performing in the environment. 
Environmental non-determinism prevents the recognizer from determining the specific execution $\vec{e}$ the observed agent goes through to achieve its goals. 
The recognizer considers possible executions that are all paths to the goal with no repeated states.
This assumption is partially justified by the fact that the probability of entering loops multiple times is low, and relaxing it is an important research direction for future work.
% We assume that the recognizer considers possible executions that are all paths to the goal with no repeated states. This assumption is justified by the fact that the probability of entering loops multiple times is low.

% The recognizer considers the possible executions that are all paths with no repeated states, so it consider all simple paths because there is a very low probability of entering loops for strong-cyclic solution (this concerns the environment).

% Since the environment is non-deterministic, we do not know what execution $\vec{e}$ the observed agent will perform to achieve its temporally extended goals. 
%Additionally, we note that strong-cyclic solutions may have infinite possible executions, however, here we consider only those that enter a possible present loop \textit{at most} once. Note that possibly repeated actions present in loops do not affect the computation of the average distance. Indeed, even if the observed agent executes the same action repeatedly often, it does not change its distance to the goal.

% The probability of getting some loops from the environment that bring you back to a previously visited state is low. Therefore, we consider just all simple execution paths, which are the paths that are all paths with no repeated nodes. -- we consider all simple paths because there is a very low probability of entering loops for strong-cyclic solution (this concerns the environment).

After enumerating the set of possible executions $\vec{E}$ for a policy $\pi$, we compute the average distance of all actions in the set of executions $\vec{E}$ to the goal state $\varphi$ from initial state $s_{0}$.
We note that strong-cyclic solutions may have infinite possible executions. However, here
we consider executions that do not enter loops, and for those entering possible loops, we consider only the ones entering loops \textit{at most} once.
Indeed, the computation of the average distance is not affected by the occurrence of possibly repeated actions. In other words, if the observed agent executes the same action repeatedly often, it does not change its distance to the goal.
The average distance aims to estimate ``how far'' every observation $o \in Obs$ is to a goal state $\varphi$.
This average distance is computed because some executions $\vec{e} \in \vec{E}$ may share the same action in execution sequences but at different time steps. We refer to this average distance as $\bfmath{d}$. 
For example, consider the policy $\pi$ depicted in Figure~\ref{fig:Triangle_Tireworld_Policy}. This policy $\pi$ has two possible executions for achieving the goal state from the initial state, and these two executions share some actions, such as \pred{(move 11 21)}. In particular, this action appears twice in Figure~\ref{fig:Triangle_Tireworld_Policy} due to its uncertain outcome. Therefore, this action has two different distances (if we count the number of remaining actions towards the goal state) to the goal state: $distance = 1$, if the outcome of this action generates the state $s_2$; and $distance = 2$, if the outcome of this action generates the state $s_3$. Hence, since this policy $\pi$ has two possible executions, and the sum of the distances is 3, the average distance for this action to the goal state is $\bfmath{d} = 1.5$. The average distances for the other actions in this policy are: $\bfmath{d} = 1$ for \pred{(changetire 21)}, because it appears only in one execution; and $\bfmath{d} = 0$ for \pred{(move 21 22)}, because the execution of this action achieves the goal state.

We use $\bfmath{d}$ to compute an \emph{estimated score} that expresses ``how far'' every observed action in the observation sequence $Obs$ is to a temporally extended goal $\varphi$ in comparison to the other goals in the set of goal hypotheses $\mathcal{G}_{\varphi}$. This means that the goal(s) with the lowest score(s) along the execution of the observed actions $o \in Obs$ is (are) the one(s) that, most likely, the observation sequence $Obs$ aims to achieve.  
We note that, the average distance $\bfmath{d}$ for those observations $o \in Obs$ that are not in the set of executions $\vec{E}$ of a policy $\pi$, is set to a large constant number, i.e., to $\bfmath{d} = e^{5}$. As part of the computation of this \textit{estimated score}, we compute a \emph{penalty value} that directly affects the \textit{estimated score}. This \emph{penalty value} represents a penalization that aims to increase the \textit{estimated score} for those goals in which each pair of subsequent observations $\langle o_{i-1}, o_{i} \rangle$ in $Obs$ does not have any relation of order in the set of executions $\vec{E}$ of these goals. We use the Euler constant $e$ to compute this \textit{penalty value}, formally defined as 
$e^{\bfmath{p}(o_{i-1}, o_{i})}$, in which we use $\R(\vec{e})$ as the set of order relation of an execution $\vec{e}$, where
\begin{equation}
\centering
\bfmath{p}(o_{i-1}, o_{i}) =
\begin{cases}
    1,		& \text{if } \lbrace\forall \vec{e} \in E | \langle o_{i-1} \prec o_{i} \rangle \notin \R(\vec{e})\rbrace \\
	% 0,		& \text{if } i=0\\
    0,      & \text{otherwise}
\end{cases}
\end{equation}

Equation~\ref{eq:estimated_score} formally defines the computation of the \emph{estimated score} for every goal $\varphi \in \mathcal{G}_{\varphi}$ given a pair of subsequent observations $\langle o_{i-1}, o_{i} \rangle$, and the set of goal hypotheses $\mathcal{G}_{\varphi}$. 
% It is possible to see that we use the \emph{penalty value} $e^{\bfmath{p}(o_{i-1}, o_{i})}$ as a way to increase the average distance $\bfmath{d}$ of the goals in which subsequent observations $\langle o_{i-1}, o_{i} \rangle$ do not have any relation of order in one of the executions $\vec{e}$ in the set of possible executions $E$, as we formally describe below.
\begin{equation}
\label{eq:estimated_score}
\frac{\mathit{e^{\bfmath{p}(o_{i-1}, o_{i})}} * \bfmath{d}(o_{i}, \varphi)}
{\sum_{\varphi' \in \mathcal{G}_{\varphi}} \bfmath{d}(o_{i}, \varphi')}
\end{equation}

\begin{figure}[!ht]
	\centering 	\includegraphics[width=0.6\linewidth]{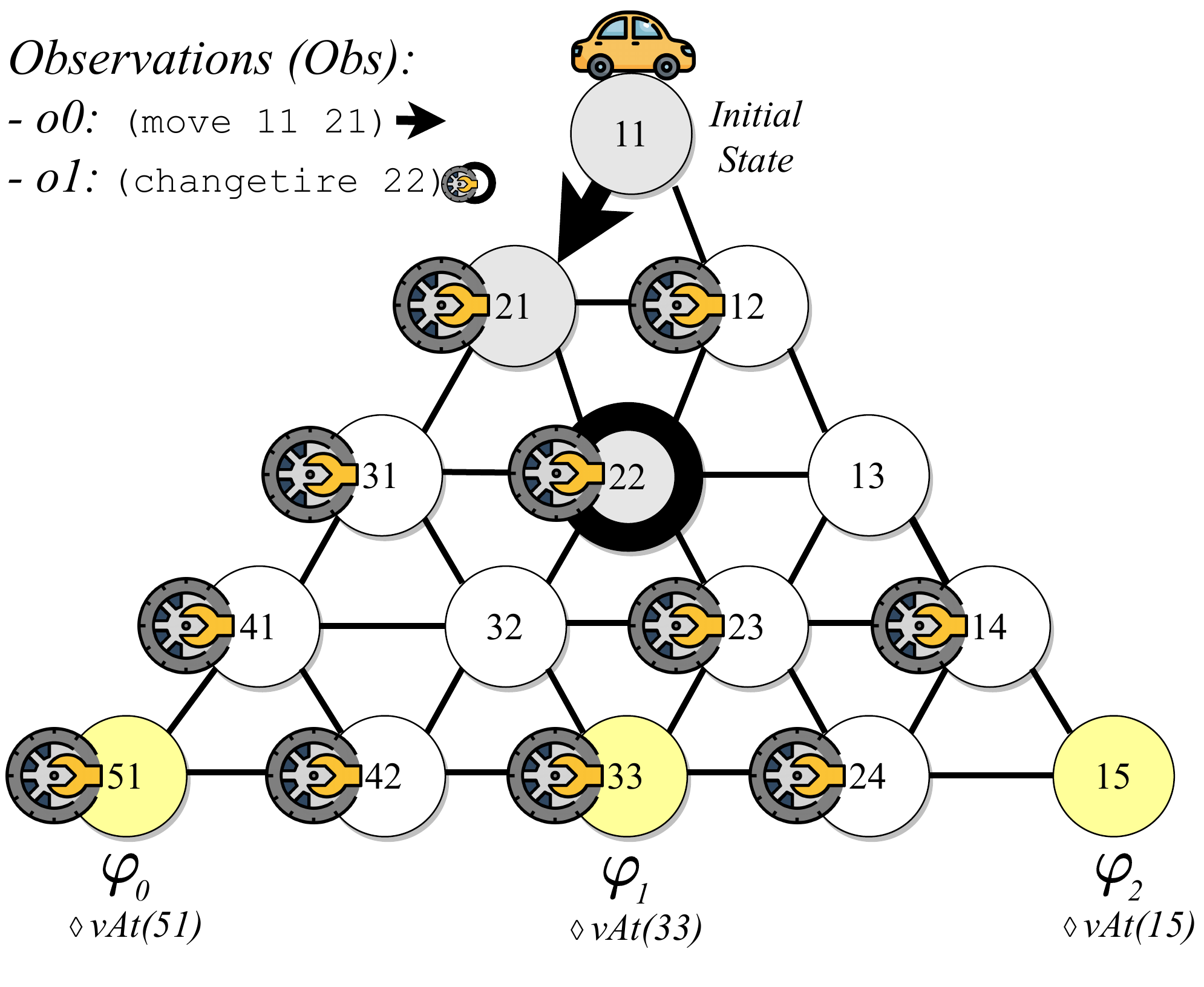}
	\caption{Recognition problem example.}
	\label{fig:Recognition_Example}
\end{figure}

\begin{examplebold}\label{exemp:recognition}
\it
To exemplify how we compute the \textit{estimated score} for every goal $\varphi \in \mathcal{G}_{\varphi}$, consider the recognition problem in Figure~\ref{fig:Recognition_Example}: $s_0$ is $vAt(11)$; the goal hypotheses $\mathcal{G}_{\varphi}$ are expressed as \LTLf goals, such that $\varphi_0 = \Diamond vAt(51), \varphi_1 = \Diamond vAt(33)$, and $\varphi_2 = \Diamond vAt(15)$; $Obs = \lbrace o_0: {\normalfont \pred{(move 11 21)}}, o_1: {\normalfont \pred{(changetire 22)}} \rbrace$. The intended goal $\varphi^{*}$ is $\varphi_1$. 
Before computing the \textit{estimated score} for the goals, we first perform the compilation process presented before. 
Afterward, we extract policies for every goal $\varphi \in \mathcal{G}_{\varphi}$, enumerate the possible executions $\vec{E}$ for the goals $\mathcal{G}_{\varphi}$ from the extracted policies, and then compute the average distance $\bfmath{d}$ of all actions in the set of executions $\vec{E}$ for the goals $\mathcal{G}_{\varphi}$ from $s_{0}$. 
The number of possible executions $\vec{E}$ for the goals are: $\varphi_0: |\vec{E}| = 8, \varphi_1: |\vec{E}| = 8$, and $\varphi_2 = |\vec{E}| = 16$. The average distances $\bfmath{d}$ of all actions in $\vec{E}$ for the goals are as follows: 
\begin{itemize}
	\item $\varphi_0$: {\normalfont\pred{(move 11 21)}} = 4.5, {\normalfont\pred{(changetire 21)}} = 4, {\normalfont\pred{(move 21 31)}} = 3, {\normalfont\pred{(changetire 31)}} = 2.5, {\normalfont\pred{(move 31 41)}} = 1.5, {\normalfont\pred{(changetire 41)}} = 1, {\normalfont\pred{(move 41 51)}} = 0;
	
	\item $\varphi_1$: {\normalfont\pred{(move 11 21)}} = 4.5, {\normalfont\pred{(changetire 21)}} = 4, {\normalfont\pred{(move 21 22)}} = 3, {\normalfont\pred{(changetire 22)}} = 2.5, {\normalfont\pred{(move 22 23)}} = 1.5, {\normalfont\pred{(changetire 23)}} = 1, {\normalfont\pred{(move 23 33)}}: 0; 
	
	\item $\varphi_2$: {\normalfont\pred{(move 11 21)}} = 6, {\normalfont\pred{changetire 21)}} = 5.5, {\normalfont\pred{(move 21 22)}} = 4.5, {\normalfont\pred{(changetire 22)}} = 4, {\normalfont\pred{(move 22 23)}} = 3, {\normalfont\pred{(changetire 23)}} = 2.5, {\normalfont\pred{(changetire 24)}} = 1, {\normalfont\pred{(move 23 24)}} = 1.5, {\normalfont\pred{(move 24 15)}} = 0.
\end{itemize}	

Once we have the average distances $\bfmath{d}$ of the actions in $\vec{E}$ for all goals, we can then compute the \textit{estimated score} for $\mathcal{G}_{\varphi}$ for every observation $o \in Obs$: $o_0 \pred{(move 11 21)}: \varphi_0 = \frac{4.5}{4.5 + 6} = $ 0.43, $\varphi_1 = \frac{4.5}{4.5 + 6} =$ 0.43, $\varphi_2 = \frac{6}{4.5 + 6} =$ 0.57; and $o_1 \pred{(changetire 22)}: \varphi_0 = \frac{e^1 * e^5}{6.5} =$ 61.87, $\varphi_1 = \frac{2.5}{e^5 + 2.5} =$ 0.016, $\varphi_2 = \frac{4}{e^5 + 4} =$ 0.026. Note that for the observation $o_1$, the average distance $\bfmath{d}$ for $\varphi_0$ is $e^5 = 148.4$ because this observation is not an action for one of the executions in the set of executions for this goal ($Obs$ aims to achieve the intended goal $\varphi^{*} = \varphi_1$). Furthermore, the \textit{penalty value} is applied to $\varphi_0$, i.e., $e^1 = 2.71$. We can see that the \textit{estimated score} of the intended goal $\varphi_1$ is always the lowest for all observations $Obs$, especially when we observe the second observation $o_1$. Note that our approach correctly infers the intended goal $\varphi^{*}$, even when observing with just few actions.
\end{examplebold}

% % FRM - This bit of glue is a bit unnecessary
% Next, we present how we use the \textit{estimated score} to compute posterior probabilities for the set of possible temporally extended goals $\mathcal{G}_{\varphi}$.

\subsubsection*{Computing Posterior Probabilities for $\mathcal{G}_{\varphi}$} To compute the posterior probabilities over the set of possible temporally extended goals $\mathcal{G}_{\varphi}$, we start by  computing the \textit{average estimated score} for every goal $\varphi \in \mathcal{G}_{\varphi}$ for every observation $o \in Obs$, and we formally define this computation as $\E(\varphi, Obs, \mathcal{G}_{\varphi})$, as follows:

\begin{equation}
\label{eq:approach_obs}
\E(\varphi, Obs, \mathcal{G}_{\varphi}) = 
\left(
\frac{\displaystyle\sum_{i=0}^{|Obs|} 
\frac{\mathit{e^{\bfmath{p}(o_{i-1}, o_{i})}} * \bfmath{d}(o_{i}, \varphi)}
{\sum_{\varphi' \in \mathcal{G}_{\varphi}} \bfmath{d}(o_{i}, \varphi')}}
{ |Obs| }
\right)
\end{equation}

The \textit{average estimated score} $\E$ aims to estimate ``how far'' a goal $\varphi$ is to be achieved compared to other goals ($\mathcal{G}_{\varphi} \setminus \{\varphi\}$) \emph{averaging} among all the observations in $Obs$. 
The lower the \textit{average estimated score} $\E$ to a goal $\varphi$, the more likely such a goal is to be the one that the observed agent aims to achieve.
Consequently, $\E$ has two important properties defined in Equation~\ref{eq:approach_obs}, as follows.

\begin{proposition}
Given that the sequence of observations $Obs$ corresponds to an execution $\vec{e} \in \vec{E}$ that aims to achieve the actual intended hidden goal ${\varphi}^{*} \in \mathcal{G}_{\varphi}$, the \textit{average estimated score} outputted by $\E$ will tend to be the lowest for ${\varphi}^{*}$ in comparison to the scores of the other goals ($\mathcal{G}_{\varphi} \setminus \{ {\varphi}^{*} \}$), as observations increase in length.
\end{proposition}

\begin{proposition}
If we restrict the recognition setting and define that the goal hypotheses $\mathcal{G}_{\varphi}$ are not sub-goals of each other, and observe all observations in $Obs$ (i.e., full observability), we will have the intended goal ${\varphi}^{*}$ with the lowest score among all goals, i.e., $\forall \varphi \in \mathcal{G}_{\varphi}$ is the case that $\E({\varphi}^{*}, Obs, \mathcal{G}_{\varphi}) \leq \E(\varphi, Obs, \mathcal{G}_{\varphi})$.	
\end{proposition}

After defining how we compute the \textit{average estimated score} $\E$ for the goals using Equation~\ref{eq:approach_obs}, we can define how our approach tries to maximize the probability of observing a sequence of observations $Obs$ for a given goal $\varphi$, as follows:
\begin{equation}
\centering
\label{eq:posterior_prob}
\mathbb{P}(Obs \mid \varphi) = [1 + \E(\varphi, Obs, \mathcal{G}_{\varphi})]^{-1}
\end{equation}

Thus, by using the \textit{estimated score} in Equation~\ref{eq:posterior_prob}, we can infer that the goals $\varphi \in \mathcal{G}_{\varphi}$ with the lowest \textit{estimated score} will be the most likely to be achieved according to the probability interpretation we propose in Equation~\ref{eq:approach_obs}. For instance, consider the goal recognition problem presented in Example~\ref{exemp:recognition}, and the \textit{estimated scores} we computed for the temporally extended goals $\varphi_0$, $\varphi_1$, and $\varphi_2$ based on the observation sequence $Obs$. From this, we have the following probabilities $\mathbb{P}(Obs \mid \varphi)$ for the goals:

\begin{itemize}
	\item $\mathbb{P}(Obs \mid \varphi_0) = [1 + (31.15)]^{-1} = 0.03$
	\item $\mathbb{P}(Obs \mid \varphi_1) = [1 + (0.216)]^{-1} = 0.82$
	\item $\mathbb{P}(Obs \mid \varphi_2) = [1 + (0.343)]^{-1} = 0.74$
\end{itemize}

After normalizing these computed probabilities using the normalization factor $\eta$\footnote{$\eta = [\sum_{\varphi \in \mathcal{G}_{\varphi}} \mathbb{P}(Obs \mid \varphi) * \mathbb{P}(\varphi)]^{-1}$}, and assuming that the prior probability $\mathbb{P}(\varphi)$ is equal to every goal in the set of goals $\mathcal{G}_{\varphi}$, we can use Equation~\ref{eq:posterior_prob} to compute the posterior probabilities (Equation~\ref{eq:posterior}) for the temporally extended goals $\mathcal{G}_{\varphi}$. We define the \textit{solution} to a recognition problem $\T_{\varphi}$ (Definition~\ref{def:goal_recognition}) as a set of temporally extended goals $\mathcal{G}_{\varphi}^{*}$ with the \textit{maximum probability}, formally: $\mathcal{G}_{\varphi}^{*} = \argmax_{\varphi \in \mathcal{G}_{\varphi}} \mathbb{P}(\varphi \mid Obs)$. Hence, considering the normalizing factor $\eta$ and the probabilities $\mathbb{P}(Obs \mid \varphi)$ computed before, we then have the following posterior probabilities for the goals in Example~\ref{exemp:recognition}: $\mathbb{P}(\varphi_0 \mid Obs) = 0.001$; $\mathbb{P}(\varphi_1 \mid Obs) = 0.524$; and $\mathbb{P}(\varphi_2 \mid Obs) = 0.475$. Recall that in Example~\ref{exemp:recognition}, $\varphi^{*}$ is $\varphi_1$, and according to the computed posterior probabilities, we then have $\mathcal{G}_{\varphi}^{*} = \lbrace \varphi_1 \rbrace$, so our approach yields only the correct intended goal by observing just two observations.

Using the \textit{average distance} $\bfmath{d}$ and the \textit{penalty value} $\bfmath{p}$ allows our approach to disambiguate similar goals during the recognition stage. 
For instance, consider the following possible temporally extended goals: $\varphi_0 = \phi_1 \lUntil \phi_2$ and $\varphi_1 = \phi_2 \lUntil \phi_1$.
Here, both goals have the same formulas to be achieved, i.e., $\phi_1$ and $\phi_2$, but in a different order. 
Thus, even having the same formulas to be achieved, the sequences of their policies' executions are different. Therefore, the average distances are also different, possibly a smaller value for the temporally extended goal that the agent aims to achieve, and the penalty value may also be applied to the other goal if two subsequent observations do not have any order relation in the set of executions for this goal.

\subsubsection*{Computational Analysis}

% In essence, 
The most expensive computational part of our recognition approach is computing the policies $\pi$ for the goal hypotheses $\mathcal{G}_{\varphi}$.
% from an initial state $s_0$. 
Thus, we can say that our approach requires $|\mathcal{G}_{\varphi}|$ calls to an off-the-shelf \FOND planner. Hence, the computational complexity of our recognition approach is linear in the number of goal hypotheses $|\mathcal{G}_{\varphi}|$.
% , namely $O(|\mathcal{G}_{\varphi}|)$. 
In contrast, to recognize goals and plans in \textit{Classical Planning} settings, the approach of \cite{RamirezG_AAAI2010} requires $2 * |\mathcal{G}|$ calls to an off-the-shelf \textit{Classical} planner.
% , i.e., $O(2 * |\mathcal{G}_{\varphi}|)$. 
Concretely,
to compute $\mathbb{P}(Obs \mid G)$, Ramirez and Geffner's approach computes two plans for every goal and based on these two plans, they compute a \textit{cost-difference} between these plans and plug it into a Boltzmann equation. For computing these two plans, this approach requires a non-trivial transformation process that modifies both the domain and problem, i.e., an augmented domain and problem that compute a plan that \textit{complies} with the observations, and another augmented domain and problem to compute a plan that \textit{does not comply} with the observations. Essentially, the intuition of Ramirez and Geffner's approach is that the lower the \textit{cost-difference} for a goal, the higher the probability for this goal, much similar to the intuition of our \textit{estimated score} $\E$.

%--------------------------------------------------------------------
% \input{arxiv_sections/05-solution_approach}
%--------------------------------------------------------------------
\section{Experiments and Evaluation}\label{sec:experiments_evaluation}

We now present experiments and evaluations carried out to validate the effectiveness of our recognition approach. We empirically evaluate our approach over thousands of goal recognition problems using well-known \FOND planning domain models with different types of temporally extended goals expressed in \LTLf and \PLTLf.

The source code of our PDDL encoding for \LTLf and \PLTLf goals\footnote{\url{https://github.com/whitemech/FOND4LTLf}} and our temporally extended goal recognition approach\footnote{\url{https://github.com/ramonpereira/goal-recognition-ltlf_pltlf-fond}}, as well as the recognition datasets and results are available on GitHub.

%--------------------------------------------------------------------
\subsection{Domains, Recognition Datasets, and Setup}

For experiments and evaluation, we use six different \FOND planning domain models, in which most of them are commonly used in the AI Planning community to evaluate \FOND planners~\citep{MyND_MattmullerOHB10,Muise12ICAPSFond}, such as: \textsc{Blocks-World}, \textsc{Logistics}, \textsc{Tidy-up}, \textsc{Tireworld}, \textsc{Triangle-Tireworld}, and \textsc{Zeno-Travel}. 
The domain models involve practical real-world applications, such as navigating, stacking, picking up and putting down objects, loading and unloading objects, loading and unloading objects, and etc. Some of the domains combine more than one of the characteristics we just described, namely, \textsc{Logistics}, \textsc{Tidy-up}~\citep{nebel13_tidyup_aaaiirs}, and \textsc{Zeno-Travel}, which involve navigating and manipulating objects in the environment. In practice, our recognition approach is capable of recognizing not only the set of facts of a goal that an observed agent aims to achieve from a sequence of observations, but also the \textit{temporal order} (e.g., \textit{exact order}) in which the agent aims to achieve this set of facts. For instance, for \textsc{Tidy-up}, is a real-world application domain, in which the purpose is defining planning tasks for a household robot that could assist elder people in smart-home application, our approach would be able to monitor and assist the household robot to achieve its goals in a specific order.

Based on these \FOND planning domain models, we build different recognition datasets: a \textit{baseline} dataset using conjunctive goals ($\phi_1\land \phi_2$) and datasets with \LTLf and \PLTLf goals. 

For the \LTLf datasets, we use three types of goals: 
\begin{itemize}
	\item $\Diamond\phi$, where $\phi$ is a propositional formula expressing that \textit{eventually} $\phi$ will be achieved. This temporal formula is analogous to a conjunctive goal;
	\item $\Diamond(\phi_1 \land \Next(\Diamond\phi_2))$, expressing that $\phi_1$ must hold before $\phi_2$ holds. For instance, we can define a temporal goal that expresses the order in which a set of packages in \textsc{Logistics} domain should be delivered;
	\item $\phi_1 \lUntil \phi_2$: $\phi_1$ must hold \textit{until} $\phi_2$ is achieved. For the \textsc{Tidy-up} domain, we can define a temporal goal that no one can be in the kitchen until the robot cleans the kitchen.
\end{itemize}

For the \PLTLf datasets, we use two types of goals: 
\begin{itemize}
	\item $\phi_1 \land \past \phi_2$, expressing that $\phi_1$ holds and $\phi_2$ held once. For instance, in the \textsc{Blocks-World} domain, we can define a past temporal goal that only allows stacking a set of blocks (\pred{a}, \pred{b}, \pred{c}) once another set of blocks has been stacked (\pred{d}, \pred{e});
	\item $\phi_1 \land (\lnot\phi_2 \Since \phi_3)$, expressing that the formula $\phi_1$ holds and \textit{since} $\phi_3$ held $\phi_2$ was not true anymore. For instance, in \textsc{Zeno-Travel}, we can define a past temporal goal expressing that person$1$ is at city$1$ and since the person$2$ is at city$1$, the aircraft must not pass through city$2$ anymore. 
\end{itemize}

Thus, in total, we have six different recognition data\-sets over the six \FOND planning domains and temporal formulas presented above. Each of these datasets contains hundreds of recognition problems ($\approx$ 390 recognition problems per dataset), such that each recognition problem $\T_{\varphi}$ in these datasets is comprised of a \FOND planning domain model $\D$, an initial state $s_0$, a set of possible goals $\mathcal{G}_{\varphi}$ (expressed in either \LTLf or \PLTLf), the actual intended hidden goal in the set of possible goals $\varphi^{*} \in \mathcal{G}_{\varphi}$, and the observation sequence $Obs$. We note that the set of possible goals $\mathcal{G}_{\varphi}$ contains very similar goals (i.e., $\varphi_0 = \phi_1 \lUntil \phi_2$ and $\varphi_1 = \phi_2 \lUntil \phi_1$), and all possible goals can be achieved from the initial state by a strong-cyclic policy. For instance, for the \textsc{Tidy-up} domain, we define the following \LTLf goals as possible goals $\mathcal{G}_{\varphi}$:

\begin{itemize}
	\item $\varphi_0 = \Diamond (\pred{(wiped desk1)} \land \Next(\Diamond \pred{(on book1 desk1)}))$;
	\item $\varphi_1 = \Diamond (\pred{(on book1 desk1)} \land \Next(\Diamond \pred{(wiped desk1)}))$;
	\item $\varphi_2 = \Diamond (\pred{(on cup1 desk2)} \land \Next(\Diamond \pred{(wiped desk2)}))$;
	\item $\varphi_3 = \Diamond (\pred{(wiped desk2)} \land \Next(\Diamond \pred{(on cup1 desk2)}))$;		
\end{itemize}

Note that some of the goals described above share the same formulas and fluents, but some of these formulas must be achieved in a different order, e.g., $\varphi_0$ and $\varphi_1$, and $\varphi_2$ and $\varphi_3$. We note that the recognition approach we developed in the paper is very accurate in discerning (Table~\ref{tab:gr_results}) the order that the intended goal aims to be achieved based on few observations (executions of the agent in the environment).

As we mentioned earlier in the paper, an observation sequence contains a sequence of actions that represent an execution $\vec{e}$ in the set of possible executions $\vec{E}$ of policy $\pi$ that achieves the actual intended hidden goal $\varphi^{*}$, and as we stated before, this observation sequence $Obs$ can be full or partial. To generate the observations $Obs$ for $\varphi^{*}$ and build the recognition problems, we extract strong-cyclic policies using different \FOND planners, such as PRP and MyND.
A full observation sequence represents an execution (a sequence of executed actions) of a strong-cyclic policy that achieves the actual intended hidden goal $\varphi^{*}$, i.e., 100\% of the actions of $\vec{e}$ being observed. 
A partial observation sequence is represented by a sub-sequence of actions of a full execution that aims to achieve the actual intended hidden goal $\varphi^{*}$ (e.g., an execution with ``missing'' actions, due to a sensor malfunction). In our recognition datasets, we define four levels of observability for a partial observation sequence: 10\%, 30\%, 50\%, or 70\% of its actions being observed. 
For instance, for a full observation sequence $Obs$ with 10 actions (100\% of observability), a corresponding partial observations sequence with 10\% of observability would have only one observed action, and for 30\% of observability three observed actions, and so on for the other levels of observability.

We ran all experiments using PRP~\citep{Muise12ICAPSFond} planner with a single core of a 12 core Intel(R) Xeon(R) CPU E5-2620 v3 @ 2.40GHz with 16GB of RAM, set a maximum memory usage limit of 8GB, and set a 10-minute timeout for each recognition problem.
We note that we are unable to provide a \textit{direct comparison} of our approach against existing recognition approaches in the literature because most of these approaches perform a non-trivial process that transforms a recognition problem into planning problems to be solved by a planner~\citep{RamirezG_AAAI2010,Sohrabi_IJCAI2016}. Even adapting such a transformation to work in \FOND settings with temporally extended goals, we cannot guarantee that it will work properly in the problem setting we propose in this paper.

%--------------------------------------------------------------------
\subsection{Evaluation Metrics}

We evaluate our goal recognition approach using widely known metrics in the \textit{Goal and Plan Recognition} literature~\citep{RamirezG_IJCAI2009,vered2016online,PereiraOM_AIJ_2020}.
To evaluate our approach in the \textit{Offline Keyhole Recognition} setting, we use four metrics, as follows:

\begin{itemize}
	\item \textit{True Positive Rate} (\textit{TPR}) measures the fraction of times that the intended hidden goal $\varphi^{*}$ was correctly recognized, e.g., the percentage of recognition problems that our approach correctly recognized the intended goal. A \textbf{higher} \textit{TPR} indicates better accuracy, measuring how often the intended hidden goal had the highest probability $P(\varphi \mid Obs)$ among the possible goals. \textit{TPR} (Equation~\ref{eq:tpr}) is the ratio between true positive results\footnote{\textit{True positive results} represent the number of goals that has been recognized correctly.}, and the sum of true positive and false negative results\footnote{\textit{False negative results} represent the number of correct goals that has not been recognized.};
	\begin{equation}
	\label{eq:tpr}
	TPR = \frac{TP}{TP + FN} = 1 - FNR
	\end{equation}
	
	\item \textit{False Positive Rate} (\textit{FPR}) is a metric that measures how often goals other than the intended goal are recognized (wrongly) as the intended ones. A \textbf{lower} \textit{FPR} indicates better accuracy. \textit{FPR} is the ratio between false positive results\footnote{\textit{False positive results} are the number of incorrect goals that has been recognized as the correct ones.}, and the sum of false positive and true negative results\footnote{\textit{True negative results} represent the number of incorrect goals has been recognized correctly as the incorrect ones.};
	\begin{equation}
	\label{eq:fpr}
	FPR = \frac{FP}{FP + TN}
	\end{equation}
	
	\item \textit{False Negative Rate} (\textit{FNR}) aims to measure the fraction of times in which the intended correct goal was recognized incorrectly. A \textbf{lower} \textit{FNR} indicates better accuracy. \textit{FNR} (Equation~\ref{eq:fnr}) is the ratio between false negative results and the sum of false negative and true positive results;
	\begin{equation}
	\label{eq:fnr}
	FNR = \frac{FN}{FN + TP} = 1 - TPR
	\end{equation}
	
	\item \textit{F1-Score} (Equation~\ref{eq:f1}) is the harmonic mean of precision and sensitivity (i.e., \textit{TPR}), representing the trade-off between true positive and false positive results. The \textbf{highest possible value} of an \textit{F1-Score} is 1.0, indicating perfect precision and sensitivity, and the \textbf{lowest possible value} is 0. Thus, \textbf{higher} \textit{F1-Score} values indicate better accuracy.
	\begin{equation}
	\label{eq:f1}
	F1-Score = \frac{2*TP}{2TP + FP + FN}
	\end{equation}
\end{itemize}

In contrast, to evaluate our approach in the \textit{Online Recognition} setting, we use the following metric:

\begin{itemize}
    \item \textit{Ranked First} is a metric that measures the number of times the intended goal hypothesis $\varphi^{*}$ has been correctly ranked first as the most likely intended goal, and \textbf{higher} values for this metric indicate better accuracy for performing online recognition.
\end{itemize}

In addition to the metrics mentioned above, we also evaluate our recognition approach in terms of \textit{recognition time} (\textit{Time}), which is the average time in seconds to perform the recognition process (including the calls to a \FOND planner);

%--------------------------------------------------------------------
\subsection{Offline Keyhole Recognition Results}

We now assess how accurate our recognition approach is in the \textit{Keyhole Recognition} setting.
Table~\ref{tab:gr_results} shows three inner tables that summarize and aggregate the average results of all the six datasets for four different metrics, such as \textit{Time}, \textit{TPR}, \textit{FPR}, and \textit{FNR}.
$|\mathcal{G}_{\varphi}|$ represents the average number of goals in the datasets, and $|Obs|$ the average number of observations. 
Each row in these inner tables represents the observation level, varying from 10\% to 100\%. Figure~\ref{fig:f1-comparison} shows the performance of our approach by comparing the results using \textit{F1-Score} for the six types of temporal formulas we used for evaluation. Table~\ref{tab:gr_results_separately} shows in much more detail the results for each of the six datasets we used for evaluating of our recognition approach.

\subsubsection*{Offline Results for Conjunctive and Eventuality Goals}

The first inner table shows the average results comparing the performance of our approach between conjunctive goals and temporally extended goals using the \textit{eventually} temporal operator $\Diamond$. We refer to this comparison as the \textit{baseline} since these two types of goals have the same semantics. We can see that the results for these two types of goals are very similar for all metrics. Moreover, it is also possible to see that our recognition approach is very accurate and performs well at all levels of observability, yielding high \textit{TPR} values and low \textit{FPR} and \textit{FNR} values for more than 10\% of observability. Note that for 10\% of observability, and \LTLf goals for $\Diamond\varphi$, the \textit{TPR} average value is 0.74, and it means for 74\% of the recognition problems our approach recognized correctly the intended temporally extended goal when observing, on average, only 3.85 actions. Figure~\ref{fig:f1-baseline} shows that our approach yields higher \textit{F1-Score} values (i.e., greater than 0.79) for these types of formulas when dealing with more than 50\% of observability.

\subsubsection*{Offline Results for \LTLf Goals} 

Regarding the results for the two types of \LTLf goals (second inner table), it is possible to see that our approach shows to be accurate for all metrics at all levels of observability, apart from the results for 10\% of observability for \LTLf goals in which the formulas must be recognized in a certain order. Note that our approach is accurate even when observing just a few actions (2.1 for 10\% and 5.4 for 30\%), but not as accurate as for more than 30\% of observability.
Figure~\ref{fig:f1-ltl} shows that our approach yields higher \textit{F1-Score} values (i.e., greater than 0.75) when dealing with more than 30\% of observability. 

\subsubsection*{Offline Results for \PLTLf Goals}

Finally, as for the results for the two types of \PLTLf goals, it is possible to observe in the last inner table that the overall average number of observations $|Obs|$ is less than the average for the other datasets, making the task of goal recognition more difficult for the \PLTLf datasets. Yet, we can see that our recognition approach remains accurate when dealing with fewer observations. We can also see that the values of \textit{FNR} increase for low observability, but the \textit{FPR} values are, on average, inferior to $\approx$ 0.15. Figure~\ref{fig:f1-pltl} shows that our approach gradually increases the \textit{F1-Score} values when also increases the percentage of observability.

\begin{figure*}[!ht]
	\centering
	\begin{subfigure}[b]{0.32\textwidth}
 	    \includegraphics[width=\textwidth]{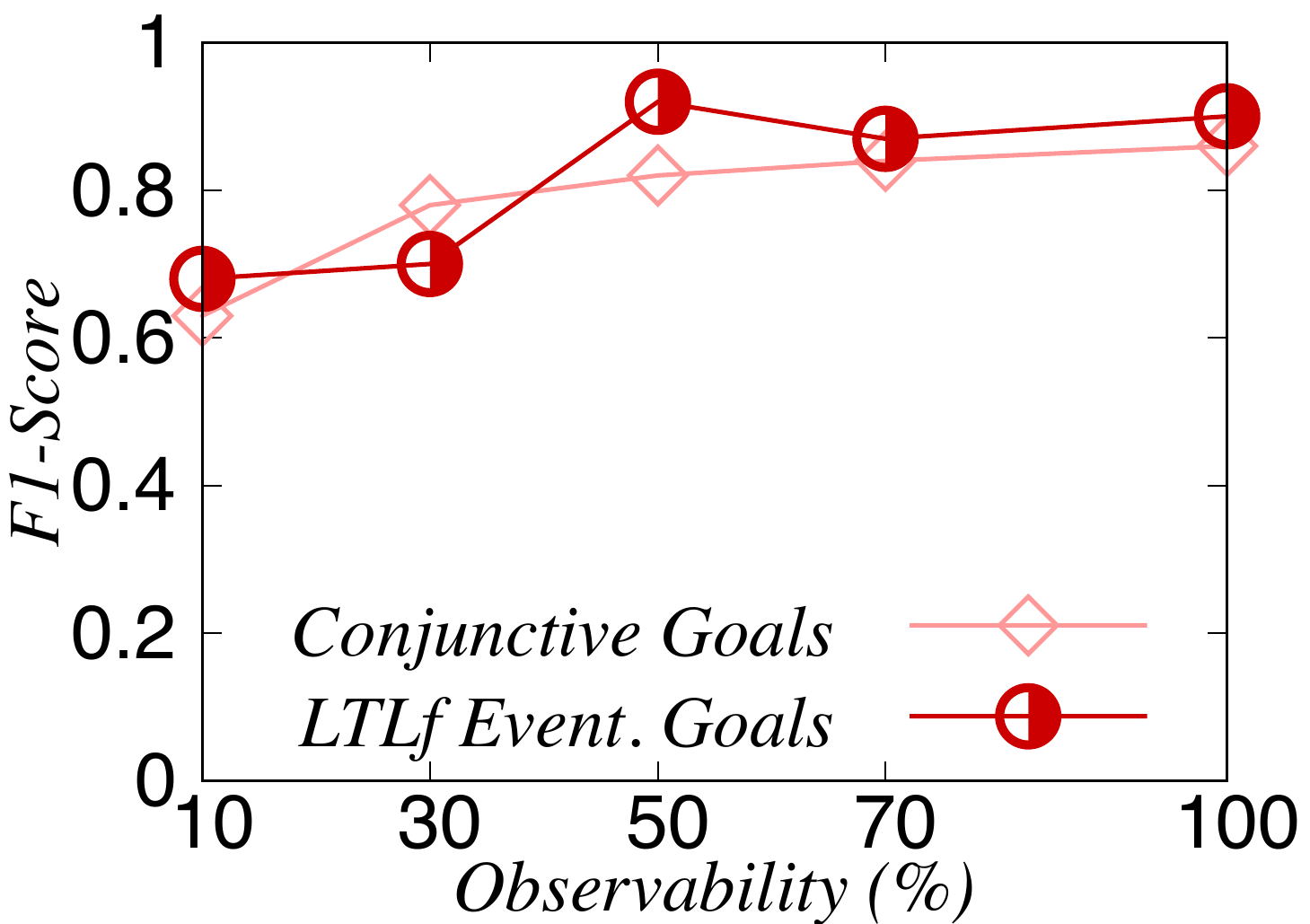}
		\caption{Conj. \textit{vs.} \LTLf Event.}
		\label{fig:f1-baseline}
	\end{subfigure}
	~
	\begin{subfigure}[b]{0.32\textwidth}
		\includegraphics[width=\textwidth]{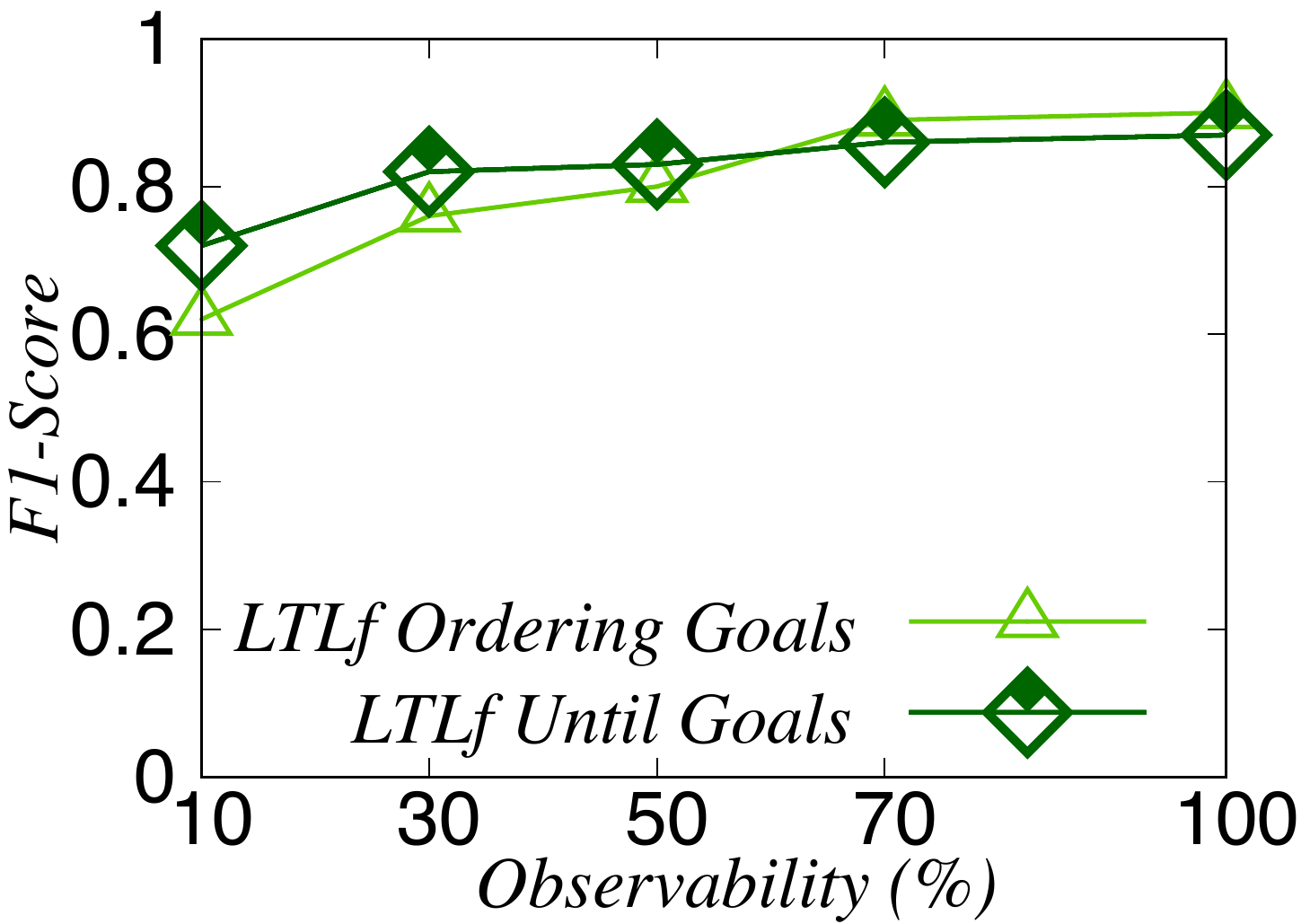}
		\caption{\LTLf: Ordering \textit{vs.} Until.}
		\label{fig:f1-ltl}
	\end{subfigure}
	~
	\begin{subfigure}[b]{0.32\textwidth}
		\includegraphics[width=\textwidth]{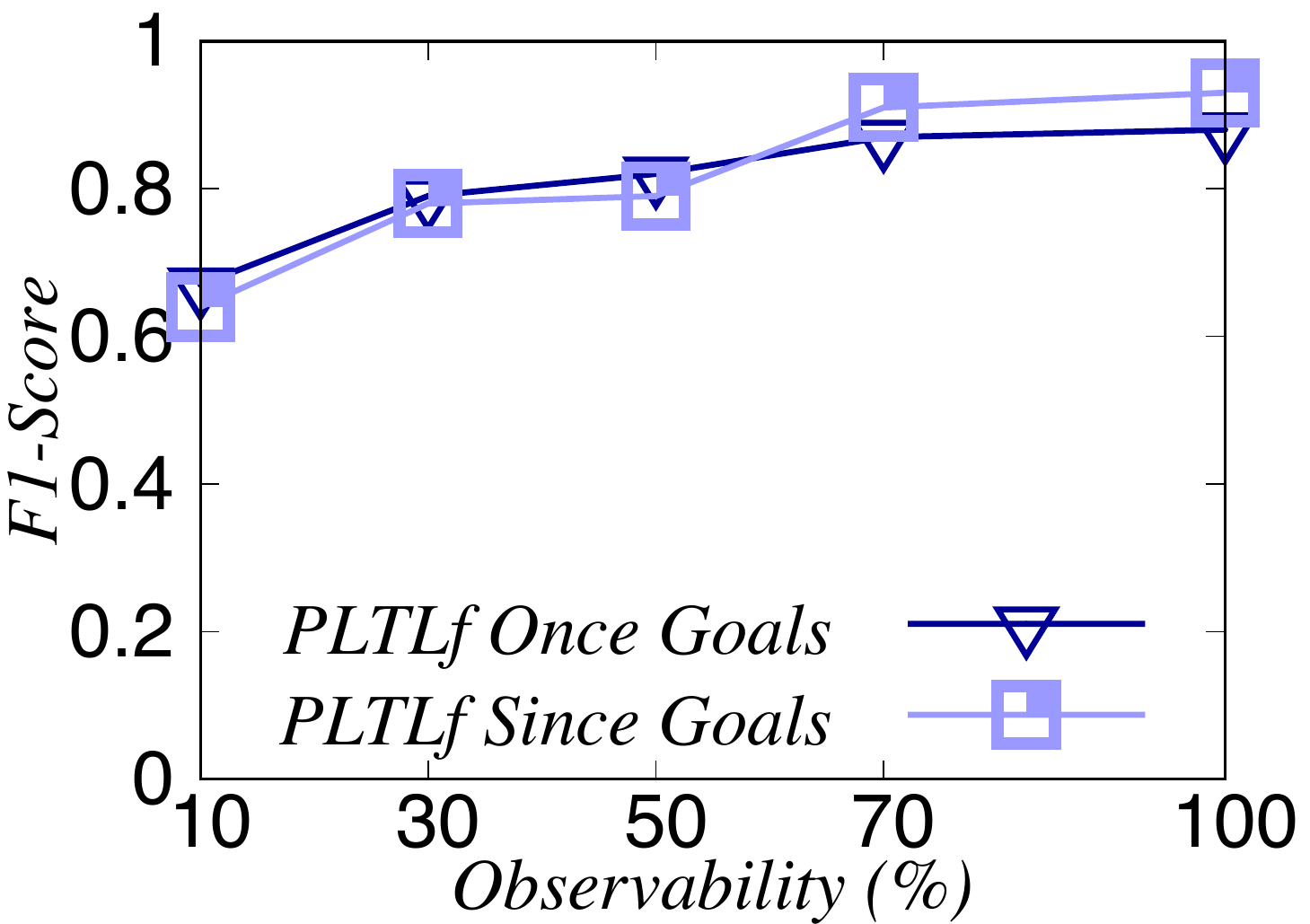}
		\caption{\PLTLf: Once \textit{vs.} Since.}
		\label{fig:f1-pltl}
	\end{subfigure}

	\caption{F1-Score comparison.}
	\label{fig:f1-comparison}
\end{figure*}

\begin{table*}[!ht]
\centering
\fontsize{10}{11}\selectfont
\setlength\tabcolsep{5pt}
\begin{tabular}{ccccccccccccccc}
\toprule
    &                      &         &  & \multicolumn{5}{c}{\begin{tabular}[c]{@{}c@{}} Conjunctive Goals \\ $\phi_1 \land \phi_2$ \end{tabular}} &  & \multicolumn{5}{c}{\begin{tabular}[c]{@{}c@{}} \LTLf Eventuality Goals \\ $\Diamond\phi$ \end{tabular}} \\
    \hline
    & $|\mathcal{G}_{\phi}|$            & $|Obs|$ &  & \textit{Time}    & \textit{TPR}     & \textit{FPR}     & \textit{FNR}	& \textit{F1-Score}     &  & \textit{Time}      & \textit{TPR}      & \textit{FPR}      & \textit{FNR}   & \textit{F1-Score}   \\ \hline
10  & \multirow{5}{*}{5.2} & 3.85     &  & 189.1     & 0.75    & 0.15    & 0.25  &  0.63  &  & 243.8       & 0.74     & 0.11     & 0.26  &  0.60 \\
30  &                      & 10.7     &  & 187.2     & 0.85    & 0.08    & 0.15  &  0.78  &  & 235.1       & 0.86     & 0.10     & 0.14  &  0.78 \\
50  &                      & 17.4     &  & 188.4     & 0.83    & 0.09    & 0.17  &  0.82  &  & 242.1       & 0.89     & 0.07     & 0.11  &  0.92 \\
70  &                      & 24.3     &  & 187.8     & 0.86    & 0.08    & 0.14  &  0.84  &  & 232.1       & 0.92     & 0.08     & 0.08  &  0.87 \\
100 &                      & 34.7     &  & 190.4     & 0.85    & 0.09    & 0.15  &  0.86  &  & 272.8       & 0.95     & 0.09     & 0.05  & 0.90 \\
\bottomrule
\end{tabular}

\begin{tabular}{ccccccccccccccc}
\toprule
    &                      &         &  & \multicolumn{5}{c}{\begin{tabular}[c]{@{}c@{}} \LTLf Ordering Goals \\ $\Diamond(\phi_1 \land \Next(\Diamond\phi_2))$ \end{tabular}} &  & \multicolumn{5}{c}{\begin{tabular}[c]{@{}c@{}} \LTLf Goals Until \\ $\phi_1 \lUntil \phi_2$ \end{tabular}} \\
    \hline
    & $|\mathcal{G}_{\phi}|$            & $|Obs|$ &  & \textit{Time}    & \textit{TPR}     & \textit{FPR}     & \textit{FNR}	& \textit{F1-Score}     &  & \textit{Time}      & \textit{TPR}      & \textit{FPR}      & \textit{FNR}   & \textit{F1-Score}   \\ \hline
10  & \multirow{5}{*}{4.0} & 2.1     &  & 136.1     & 0.68    & 0.15    & 0.32  & 0.62 &  & 217.9       & 0.79     & 0.11     & 0.21   &  0.72 \\
30  &                      & 5.4     &  & 130.9     & 0.84    & 0.13    & 0.16  & 0.76 &  & 215.8       & 0.91     & 0.12     & 0.09   &  0.82 \\
50  &                      & 8.8     &  & 132.1     & 0.88    & 0.10    & 0.12  & 0.80 &  & 210.1       & 0.93     & 0.10     & 0.07   &  0.83 \\
70  &                      & 12.5    &  & 129.2     & 0.95    & 0.06    & 0.05  & 0.89 &  & 211.5       & 0.97     & 0.09     & 0.03   &  0.86 \\
100 &                      & 17.1    &  & 126.6     & 0.94    & 0.05    & 0.06  & 0.90 &  & 207.7       & 0.97     & 0.07     & 0.03   &  0.87 \\
\bottomrule
\end{tabular}

\begin{tabular}{ccccccccccccccc}
\toprule
    &                      &         &  & \multicolumn{5}{c}{\begin{tabular}[c]{@{}c@{}} \PLTLf Goals Once \\ $\phi_1 \land \past \phi_2$ \end{tabular}} &  & \multicolumn{5}{c}{\begin{tabular}[c]{@{}c@{}} \PLTLf Goals Since \\ $\phi_1 \land (\lnot\phi_2 \Since \phi_3)$ \end{tabular}} \\
    \hline
    & $|\mathcal{G}_{\phi}|$            & $|Obs|$ &  & \textit{Time}    & \textit{TPR}     & \textit{FPR}     & \textit{FNR}	& \textit{F1-Score}     &  & \textit{Time}      & \textit{TPR}      & \textit{FPR}      & \textit{FNR}   & \textit{F1-Score}   \\ \hline
10  & \multirow{5}{*}{4.0} & 1.7     &  & 144.8     & 0.73    & 0.11    & 0.27  & 0.67 &  & 173.5     & 0.76    & 0.18    & 0.24  & 0.64 \\
30  &                      & 4.6     &  & 141.3     & 0.84    & 0.07    & 0.16  & 0.79 &  & 173.3     & 0.87    & 0.12    & 0.13  & 0.78 \\
50  &                      & 7.3     &  & 141.9     & 0.89    & 0.08    & 0.11  & 0.82 &  & 172.9     & 0.85    & 0.09    & 0.15  & 0.79 \\
70  &                      & 10.3    &  & 142.9     & 0.95    & 0.07    & 0.05  & 0.87 &  & 171.1     & 0.97    & 0.07    & 0.03  & 0.91 \\
100 &                      & 14.2    &  & 155.8     & 0.97    & 0.07    & 0.03  & 0.88 &  & 169.3     & 0.94    & 0.02    & 0.06  & 0.93 \\
\bottomrule
\end{tabular}

\caption{Offline Recognition results for Conjunctive, \LTLf, and \PLTLf goals.}
\label{tab:gr_results}
\end{table*}

%--------------------------------------------------------------------
\subsection{Online Recognition Results}

With the experiments and evaluation in the \textit{Keyhole Offline} recognition setting in place, we now proceed to present the experiments and evaluation in the \textit{Online} recognition setting.
As before, performing the recognition task in the \textit{Online} recognition setting is usually harder than in the offline setting, as the recognition task has to be performed incrementally and gradually, and we see to the observations step-by-step, rather than performing the recognition task by analyzing all observations at once, as in the offline recognition setting. % FRM - This is a bit repetitious.

\begin{figure}[!ht]
	\centering 	\includegraphics[width=0.75\linewidth]{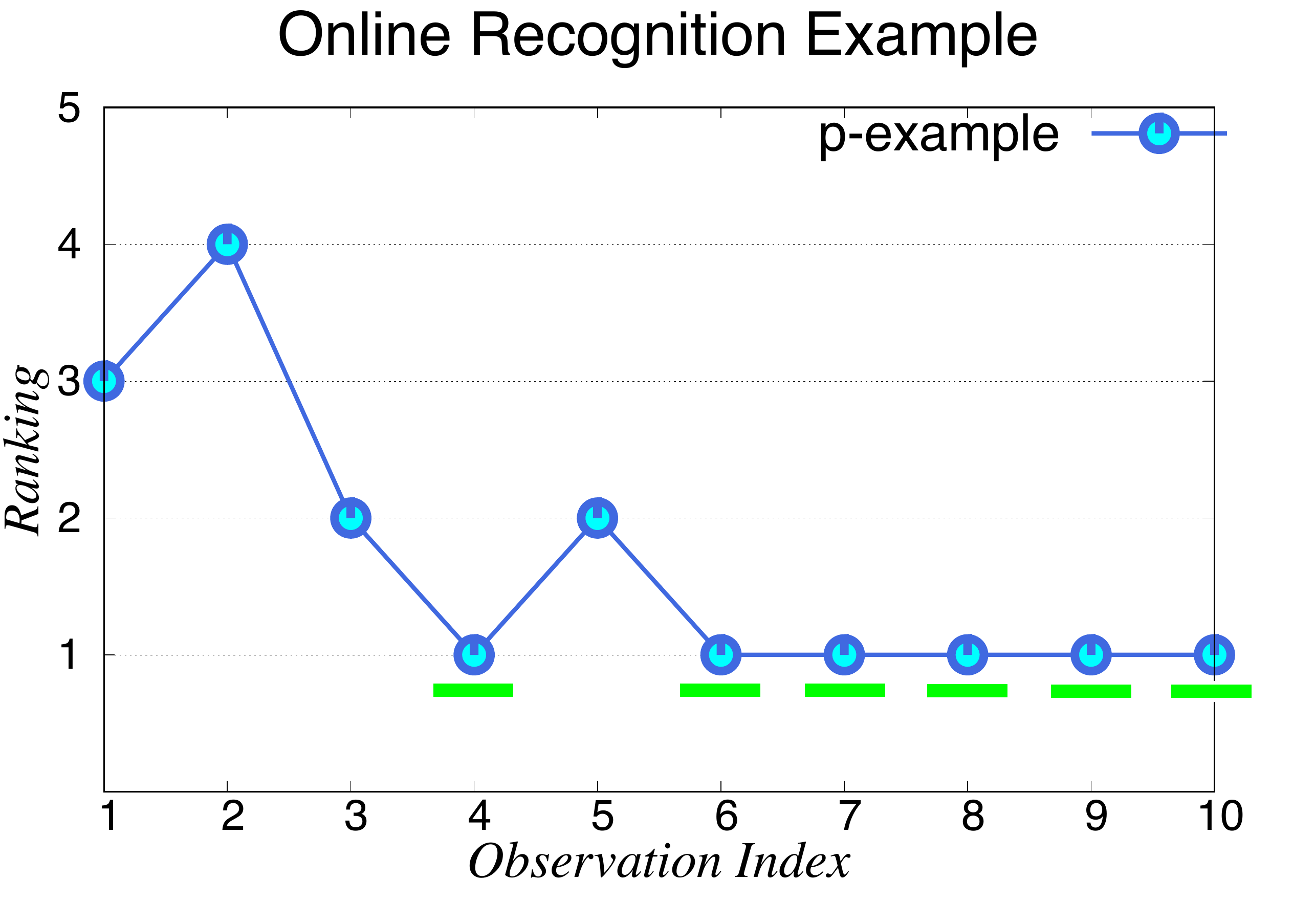}
	\caption{Online Recognition example.}
	\label{fig:onlie_recognition-example}
\end{figure}

\begin{figure}[!ht]
	\centering 	\includegraphics[width=0.8\linewidth]{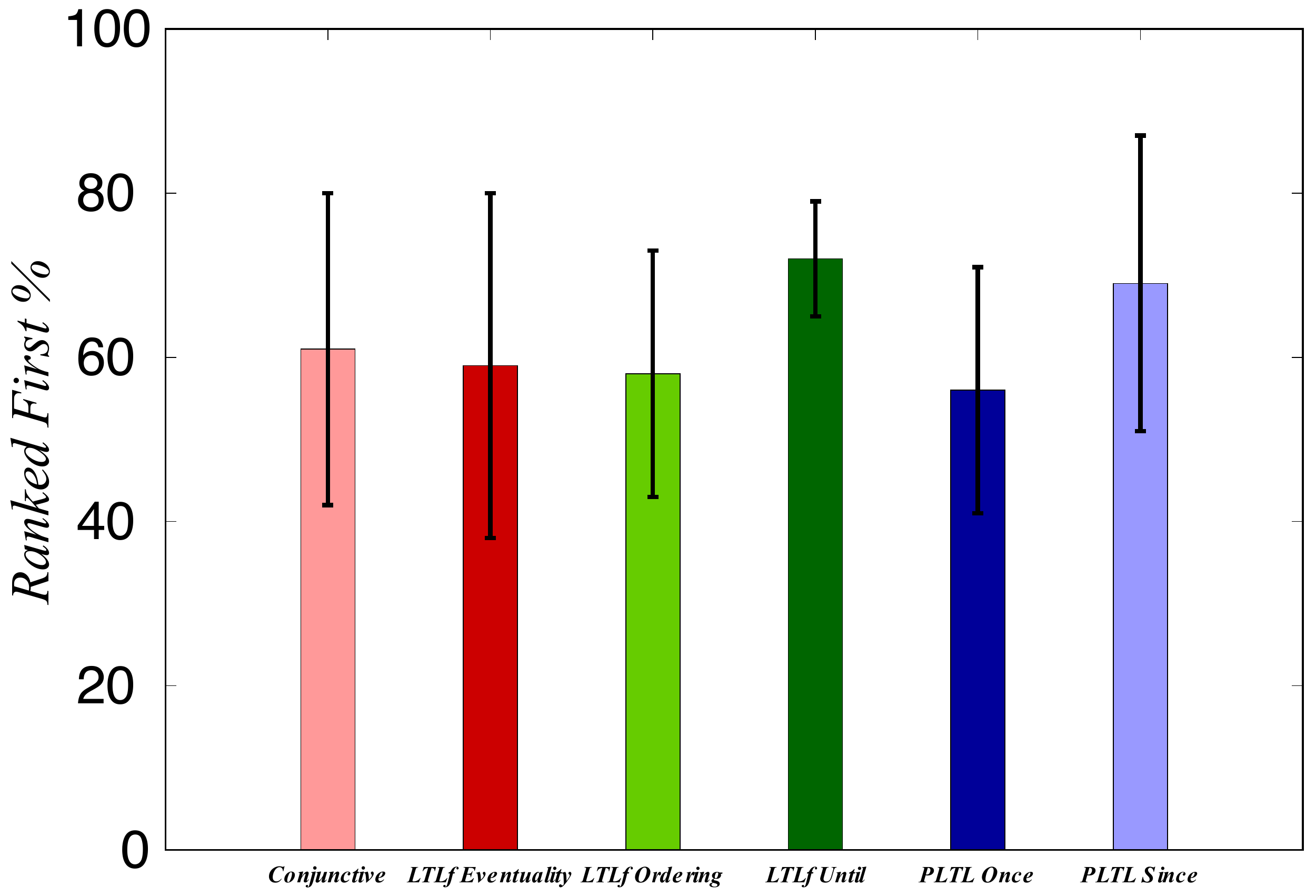}
	\caption{Online Recognition Histogram.}
	\label{fig:recognition_histogram}
\end{figure}

Figure~\ref{fig:onlie_recognition-example} exemplifies how we evaluate our approach in the \textit{Online} recognition setting. To do so, we use the \textit{Ranked First} metric, which measures how many times over the observation sequence the correct intended goal $\varphi^{*}$ has been ranked first as the \textit{top-1} goal over the goal hypotheses $\mathcal{G}_{\varphi}$.	
The recognition problem example depicted in Figure~\ref{fig:onlie_recognition-example} has five goal hypotheses (y-axis), and ten actions in the observation sequence (x-axis). As stated before, the recognition task in the \textit{Online} setting is done gradually, step-by-step, so at every step our approach essentially ranks the goals according to the probability distribution over the goal hypotheses $\mathcal{G}_{\varphi}$.
We can see that in the example in Figure~\ref{fig:onlie_recognition-example} the correct goal $\varphi^{*}$ is \textit{Ranked First} six times (at the observation indexes: 4, 6, 7, 8, 9, and 10) over the observation sequence with ten observation, so it means that the goal correct intended goal $\mathcal{G}_{\varphi}$ is \textit{Ranked First} (i.e., as the \textit{top-1}, with the highest probability among the goal hypotheses $\mathcal{G}_{\varphi}$) 60\% of the time in the observation sequence for this recognition example.

We aggregate the average recognition results of all the six datasets for the \textit{Ranked First} metric as a histogram, by considering full observation sequences that represent executions (sequences of executed actions) of strong-cyclic policies that achieves the actual intended goal $\varphi^{*}$, and we show such results in Figure~\ref{fig:recognition_histogram}. 
The results represent the overall percentage (including the standard deviation -- black bars) of how many times the of time that
the correct intended goal $\varphi^{*}$ has been ranked first over the observations. 
% It is possible to see from the average results that our approach has shown 
The average results indicated our approach to be in general accurate to recognize correctly the \textit{temporal order} of the facts in the goals in the \textit{Online} recognition setting, yielding \textit{Ranked First} percentage values greater than 58\%.

Figures~\ref{fig:blocks_ranking}, \ref{fig:logistics_ranking}, \ref{fig:tidyup_ranking}, \ref{fig:tireworld_ranking}, \ref{fig:tidyup_ranking}, \ref{fig:fig:triangle_tireworld_ranking}, and \ref{fig:zenotravel_ranking} shows the \textit{Online} recognition results separately for all six domains models and the different types of temporally extended goals. 
By analyzing the \textit{Online} recognition results more closely, we see that our approach converges to rank the correct goal as the \textit{top-1} mostly after a few observations. 
This means that it is commonly hard to disambiguate among the goals at the beginning of the execution, which, in turn, directly affects the overall \textit{Ranked First} percentage values (as we can see in Figure~\ref{fig:recognition_histogram}). 
We can observe our approach struggles to disambiguate and recognize correctly the intended goal for some recognition problems and some types of temporal formulas. Namely, our approach has struggled to disambiguate when dealing with \LTLf Eventuality goals in \textsc{Blocks-World} (see Figure~\ref{fig:blocks-ltl0}), for most temporal extended goals in \textsc{Tidy-Up} (see Figure~\ref{fig:tidyup_ranking}), and for \LTLf Eventuality goals in \textsc{Zeno-Travel} (see Figure~\ref{fig:zenotravel-ltl0}).

%--------------------------------------------------------------------
%--------------------------------------------------------------------
\section{Related Work and Discussion}\label{sec:related_work}

To the best of our knowledge, existing approaches to \textit{Goal and Plan Recognition as Planning} cannot explicitly recognize temporally extended goals in non-deter\-ministic environments. Seminal and recent work on \textit{Goal Recognition as Planning} relies on deterministic planning techniques~\citep{RamirezG_IJCAI2009,Sohrabi_IJCAI2016,PereiraOM_AIJ_2020} for recognizing conjunctive goals.
By contrast, we propose a novel problem formalization for goal recognition, addressing temporally extended goals (\LTLf or \PLTLf goals) in \FOND planning domain models.
While our probabilistic approach relies on the probabilistic framework of \cite{RamirezG_AAAI2010}, we address the challenge of computing $\mathbb{P}(Obs \mid G)$ in a completely different way. 

There exist different techniques to \textit{Goal and Plan Recognition} in the literature, including approaches that rely on plan libraries~\citep{AvrahamiZilberbrand2005}, context-free grammars~\citep{Geib_PPR_AIJ2009}, and Hierarchical Task Network (HTN)~\citep{HollerHTNRec_2018}. 
Such approaches rely on hierarchical structures that represent the knowledge of how to achieve the possible goals, and this knowledge can be seen as potential strategies for achieving the set of possible goals. 
Note that the temporal constraints of temporally extended goals can be adapted and translated to such hierarchical knowledge. 
For instance, context-free grammars are expressive enough to encode temporally extended goals~\citep{Chiari2020}. 
\LTLf has the expressive power of the star-free fragment of regular expressions and hence captured by context-free grammars. 
However, unlike regular expressions, \LTLf uses negation and conjunction liberally, and the translation to regular expression is computationally costly. 
Note, being equally expressive is not a meaningful indication of the complexity of transforming one formalism into another. 
\cite{ijcai2020surveyddfr} show that, while \LTLf and \PLTLf have the same expressive power, the best translation techniques known are worst-case 3EXPTIME. 

As far as we know, there are no encodings of \LTLf-like specification languages into HTN, and its difficulty is unclear. 
Nevertheless, combining HTN and \LTLf could be interesting for further study.
HTN techniques focus on the knowledge about the decomposition property of traces, whereas \LTLf-like solutions focus on the knowledge about dynamic properties of traces, similar to what is done in verification settings.

Most recently, \cite{ICAPS_2023_PLTL_Planning} develop a novel Pure-Past Linear Temporal Logic PDDL encoding for planning in the \textit{Classical Planning} setting. % FRM - Repetition with the intro

%--------------------------------------------------------------------
%--------------------------------------------------------------------
\section{Conclusions}\label{sec:conclusions}

We have introduced a novel problem formalization for recognizing \textit{temporally extended goals}, specified in either \LTLf or \PLTLf, in \FOND planning domain models. 
We have also developed a novel probabilistic framework for goal recognition in such settings, and implemented a compilation of temporally extended goals that allows us to reduce the problem of \FOND planning for \LTLf/\PLTLf goals to standard \FOND planning.
We have shown that our recognition approach yields high accuracy for recognizing temporally extended goals (\LTLf/\PLTLf) in different recognition settings (\textit{Keyhole Offline} and \textit{Online} recognition) at several levels of observability. 

As future work, we intend to extend and adapt our recognition approach for being able to deal with spurious (noisy) observations, and recognize not only the temporal extended goals but also anticipate the policy that the agent is executing to achieve its goals.

%--------------------------------------------------------------------

\begin{acknowledgements}
This work has been partially supported by the EU H2020 project AIPlan4EU (No. 101016442), the ERC Advanced Grant WhiteMech (No. 834228), the EU ICT-48 2020 project TAILOR (No. 952215), the PRIN project RIPER (No. 20203FFYLK), and the PNRR MUR project FAIR (No. PE0000013).
\end{acknowledgements}
% \vspace{-10mm}

%--------------------------------------------------------------------

\afterpage{
\begin{landscape}
	
\begin{table}[!ht]
\centering
\fontsize{5.2}{6}\selectfont
\setlength\tabcolsep{1pt}
\begin{tabular}{cccccccccccccclcccccclcccccclcccccclcccccc}
\toprule	
	\multicolumn{42}{c}{\begin{tabular}[c]{@{}c@{}} \textsc{Blocks-World} \end{tabular}}	\\ \hline
    & \multicolumn{6}{c}{\begin{tabular}[c]{@{}c@{}} Conjunctive Goals \\ $\phi_1 \land \phi_2$ \end{tabular}}                  &  
	& \multicolumn{6}{c}{\begin{tabular}[c]{@{}c@{}} \LTLf Eventuality Goals \\ $\Diamond\phi$ \end{tabular}}                    &  
	& \multicolumn{6}{c}{\begin{tabular}[c]{@{}c@{}} \LTLf Ordering \\ $\Diamond(\phi_1 \land \Next(\Diamond\phi_2))$ \end{tabular}}                       & \multicolumn{1}{c}{} 
	& \multicolumn{6}{c}{\begin{tabular}[c]{@{}c@{}} \LTLf Goals Until \\ $\phi_1 \lUntil \phi_2$ \end{tabular}}                          & \multicolumn{1}{c}{} 
	& \multicolumn{6}{c}{\begin{tabular}[c]{@{}c@{}} \PLTLf Goals Once \\ $\phi_1 \land \past \phi_2$ \end{tabular}}                         & \multicolumn{1}{c}{} 
	& \multicolumn{6}{c}{\begin{tabular}[c]{@{}c@{}} \PLTLf Goals Since \\ $\phi_1 \land (\lnot\phi_2 \Since \phi_3)$ \end{tabular}}                          \\
	\cline{2-7} \cline{9-14} \cline{16-21} \cline{23-28} \cline{30-35} \cline{37-42}
    & $|\mathcal{G}_{\varphi}|$ & $|Obs|$ & \textit{Time} & \textit{TPR}  & \textit{FPR}  & \textit{FNR}  &  
	& $|\mathcal{G}_{\varphi}|$ & $|Obs|$ & \textit{Time} & \textit{TPR}  & \textit{FPR}  & \textit{FNR}  &  
	& $|\mathcal{G}_{\varphi}|$ & $|Obs|$ & \textit{Time} & \textit{TPR}  & \textit{FPR}  & \textit{FNR}  &                      
	& $|\mathcal{G}_{\varphi}|$ & $|Obs|$ & \textit{Time} & \textit{TPR}  & \textit{FPR}  & \textit{FNR}  &                      
	& $|\mathcal{G}_{\varphi}|$ & $|Obs|$ & \textit{Time} & \textit{TPR}  & \textit{FPR}  & \textit{FNR}  &                      
	& $|\mathcal{G}_{\varphi}|$ & $|Obs|$ & \textit{Time} & \textit{TPR}  & \textit{FPR}  & \textit{FNR}  \\
	\cline{2-7} \cline{9-14} \cline{16-21} \cline{23-28} \cline{30-35} \cline{37-42}

10  
	& \multirow{5}{*}{6.0} & 3.92 & 33.53 & 0.81 & 0.10 & 0.19 &  
	& \multirow{5}{*}{6.0} & 3.92 & 81.16 & 0.81 & 0.06 & 0.19 &  
	& \multirow{5}{*}{4.0} & 1.33 & 38.05 & 0.67 & 0.10 & 0.33 &                      
	& \multirow{5}{*}{4.0} & 1.00 & 342.01 & 0.72 & 0.08 & 0.28 &                      
	& \multirow{5}{*}{4.0} & 1.17 & 38.97 & 0.89 & 0.08 & 0.11 &                      
	& \multirow{5}{*}{4.0} & 1.00 & 50.98 & 0.89 & 0.14 & 0.11 \\

30  &                      
	& 10.33 & 33.91 & 0.86 & 0.09 & 0.14 &  &                      
	& 10.58 & 69.65 & 0.81 & 0.10 & 0.19 &  &                      
	& 3.17 & 39.49 & 0.83 & 0.07 & 0.17 &                      &                      
	& 2.83 & 357.48 & 0.94 & 0.01 & 0.06 &                      &                      
	& 3.17 & 39.15 & 1.00 & 0.03 & 0.00 &                      &                      
	& 2.83 & 53.42 & 1.00 & 0.00 & 0.00 \\

50  &                      
	& 16.67 & 33.90 & 0.75 & 0.14 & 0.25 &  &                      
	& 17.08 & 67.16 & 0.72 & 0.09 & 0.28 &  &                      
	& 4.67 & 39.63 & 0.89 & 0.07 & 0.11 &                      &                      
	& 3.83 & 349.43 & 1.00 & 0.00 & 0.00 &                      &                      
	& 4.67 & 38.32 & 1.00 & 0.06 & 0.00 &                      &                      
	& 3.83 & 50.67 & 1.00 & 0.00 & 0.00 \\

70  &                      
	& 23.58 & 33.98 & 0.75 & 0.13 & 0.25 &  &                      
	& 23.92 & 67.71 & 0.78 & 0.09 & 0.22 &  &                      
	& 6.67 & 38.92 & 0.89 & 0.07 & 0.11 &                      &                      
	& 5.83 & 355.37 & 1.00 & 0.00 & 0.00 &                      &                      
	& 6.67 & 38.16 & 1.00 & 0.07 & 0.00 &                      &                      
	& 5.83 & 48.33 & 1.00 & 0.00 & 0.00 \\

100 &                      
	& 33.00 & 34.01 & 0.75 & 0.17 & 0.25 &  &                      
	& 33.58 & 68.82 & 0.75 & 0.14 & 0.25 &  &                      
	& 9.00 & 38.28 & 0.83 & 0.08 & 0.17 &                      &                      
	& 7.67 & 393.83 & 1.00 & 0.00 & 0.00 &                      &                      
	& 9.17 & 37.92 & 1.00 & 0.08 & 0.00 &                      &                      
	& 7.67 & 47.21 & 1.00 & 0.00 & 0.00 \\
\bottomrule
\end{tabular}

\begin{tabular}{cccccccccccccclcccccclcccccclcccccclcccccc}
\toprule
	\multicolumn{42}{c}{\begin{tabular}[c]{@{}c@{}} \textsc{Logistics} \end{tabular}}	\\ \hline
    & \multicolumn{6}{c}{\begin{tabular}[c]{@{}c@{}} Conjunctive Goals \\ $\phi_1 \land \phi_2$ \end{tabular}}                  &
	& \multicolumn{6}{c}{\begin{tabular}[c]{@{}c@{}} \LTLf Eventuality Goals \\ $\Diamond\phi$ \end{tabular}}                    &
	& \multicolumn{6}{c}{\begin{tabular}[c]{@{}c@{}} \LTLf Ordering \\ $\Diamond(\phi_1 \land \Next(\Diamond\phi_2))$ \end{tabular}}                       & \multicolumn{1}{c}{}
	& \multicolumn{6}{c}{\begin{tabular}[c]{@{}c@{}} \LTLf Goals Until \\ $\phi_1 \lUntil \phi_2$ \end{tabular}}                          & \multicolumn{1}{c}{}
	& \multicolumn{6}{c}{\begin{tabular}[c]{@{}c@{}} \PLTLf Goals Once \\ $\phi_1 \land \past \phi_2$ \end{tabular}}                         & \multicolumn{1}{c}{}
	& \multicolumn{6}{c}{\begin{tabular}[c]{@{}c@{}} \PLTLf Goals Since \\ $\phi_1 \land (\lnot\phi_2 \Since \phi_3)$ \end{tabular}}
	\\ \cline{2-7} \cline{9-14} \cline{16-21} \cline{23-28} \cline{30-35} \cline{37-42}
    & $|\mathcal{G}_{\varphi}|$ & $|Obs|$ & \textit{Time} & \textit{TPR}  & \textit{FPR}  & \textit{FNR}  &
	& $|\mathcal{G}_{\varphi}|$ & $|Obs|$ & \textit{Time} & \textit{TPR}  & \textit{FPR}  & \textit{FNR}  &
	& $|\mathcal{G}_{\varphi}|$ & $|Obs|$ & \textit{Time} & \textit{TPR}  & \textit{FPR}  & \textit{FNR}  &
	& $|\mathcal{G}_{\varphi}|$ & $|Obs|$ & \textit{Time} & \textit{TPR}  & \textit{FPR}  & \textit{FNR}  &
	& $|\mathcal{G}_{\varphi}|$ & $|Obs|$ & \textit{Time} & \textit{TPR}  & \textit{FPR}  & \textit{FNR}  &
	& $|\mathcal{G}_{\varphi}|$ & $|Obs|$ & \textit{Time} & \textit{TPR}  & \textit{FPR}  & \textit{FNR} \\
	\cline{2-7} \cline{9-14} \cline{16-21} \cline{23-28} \cline{30-35} \cline{37-42}

10
	& \multirow{5}{*}{4.0} & 3.00 & 260.92 & 0.85 & 0.11 & 0.15 &
	& \multirow{5}{*}{4.0} & 3.00 & 412.34 & 0.70 & 0.11 & 0.30 &
	& \multirow{5}{*}{4.0} & 2.33 & 498.70 & 0.72 & 0.22 & 0.28 &
	& \multirow{5}{*}{4.0} & 1.83 & 310.16 & 1.00 & 0.12 & 0.00 &
	& \multirow{5}{*}{4.0} & 1.67 & 554.45 & 0.78 & 0.14 & 0.22 &
	& \multirow{5}{*}{4.0} & 1.00 & 643.01 & 0.83 & 0.14 & 0.17 \\

30  &
	& 8.11 & 258.60 & 0.89 & 0.08 & 0.11 &  &
	& 8.11 & 360.49 & 0.96 & 0.10 & 0.04 &  &
	& 5.83 & 468.36 & 0.89 & 0.17 & 0.11 &                      &
	& 4.67 & 292.87 & 0.94 & 0.32 & 0.06 &                      &
	& 4.17 & 541.31 & 0.94 & 0.12 & 0.06 &                      &
	& 2.33 & 645.55 & 0.83 & 0.18 & 0.17 \\

50  &
	& 13.11 & 258.58 & 0.89 & 0.09 & 0.11 &  &
	& 13.00 & 383.45 & 1.00 & 0.09 & 0.00 &  &
	& 9.17 & 480.20 & 0.89 & 0.19 & 0.11 &                      &
	& 7.67 & 282.77 & 1.00 & 0.31 & 0.00 &                      &
	& 6.33 & 542.76 & 0.94 & 0.12 & 0.06 &                      &
	& 3.17 & 652.51 & 0.89 & 0.15 & 0.11 \\

70  &
	& 18.33 & 251.51 & 0.96 & 0.09 & 0.04 &  &
	& 18.11 & 380.51 & 1.00 & 0.09 & 0.00 &  &
	& 13.00 & 466.86 & 1.00 & 0.11 & 0.00 &                      &
	& 11.00 & 285.97 & 1.00 & 0.28 & 0.00 &                      &
	& 9.00 & 552.88 & 1.00 & 0.08 & 0.00 &                      &
	& 4.50 & 648.61 & 1.00 & 0.05 & 0.00 \\

100 &
	& 25.44 & 251.27 & 1.00 & 0.08 & 0.00 &  &
	& 25.22 & 444.65 & 1.00 & 0.08 & 0.00 &  &
	& 17.83 & 450.51 & 1.00 & 0.08 & 0.00 &                      &
	& 14.83 & 232.00 & 1.00 & 0.17 & 0.00 &                      &
	& 12.50 & 630.17 & 1.00 & 0.08 & 0.00 &                      &
	& 6.00 & 644.07 & 1.00 & 0.00 & 0.00 \\
\bottomrule
\end{tabular}

\begin{tabular}{cccccccccccccclcccccclcccccclcccccclcccccc}
\toprule
	\multicolumn{42}{c}{\begin{tabular}[c]{@{}c@{}} \textsc{Tidyup} \end{tabular}}	\\ \hline
    & \multicolumn{6}{c}{\begin{tabular}[c]{@{}c@{}} Conjunctive Goals \\ $\phi_1 \land \phi_2$ \end{tabular}}                  &
	& \multicolumn{6}{c}{\begin{tabular}[c]{@{}c@{}} \LTLf Eventuality Goals \\ $\Diamond\phi$ \end{tabular}}                    &
	& \multicolumn{6}{c}{\begin{tabular}[c]{@{}c@{}} \LTLf Ordering \\ $\Diamond(\phi_1 \land \Next(\Diamond\phi_2))$ \end{tabular}}                       & \multicolumn{1}{c}{}
	& \multicolumn{6}{c}{\begin{tabular}[c]{@{}c@{}} \LTLf Goals Until \\ $\phi_1 \lUntil \phi_2$ \end{tabular}}                          & \multicolumn{1}{c}{}
	& \multicolumn{6}{c}{\begin{tabular}[c]{@{}c@{}} \PLTLf Goals Once \\ $\phi_1 \land \past \phi_2$ \end{tabular}}                          & \multicolumn{1}{c}{}
	& \multicolumn{6}{c}{\begin{tabular}[c]{@{}c@{}} \PLTLf Goals Since \\ $\phi_1 \land (\lnot\phi_2 \Since \phi_3)$ \end{tabular}}
	\\ \cline{2-7} \cline{9-14} \cline{16-21} \cline{23-28} \cline{30-35} \cline{37-42}
    & $|\mathcal{G}_{\varphi}|$ & $|Obs|$ & \textit{Time} & \textit{TPR}  & \textit{FPR}  & \textit{FNR}  &
	& $|\mathcal{G}_{\varphi}|$ & $|Obs|$ & \textit{Time} & \textit{TPR}  & \textit{FPR}  & \textit{FNR}  &
	& $|\mathcal{G}_{\varphi}|$ & $|Obs|$ & \textit{Time} & \textit{TPR}  & \textit{FPR}  & \textit{FNR}  &
	& $|\mathcal{G}_{\varphi}|$ & $|Obs|$ & \textit{Time} & \textit{TPR}  & \textit{FPR}  & \textit{FNR}  &
	& $|\mathcal{G}_{\varphi}|$ & $|Obs|$ & \textit{Time} & \textit{TPR}  & \textit{FPR}  & \textit{FNR}  &
	& $|\mathcal{G}_{\varphi}|$ & $|Obs|$ & \textit{Time} & \textit{TPR}  & \textit{FPR}  & \textit{FNR} \\
	\cline{2-7} \cline{9-14} \cline{16-21} \cline{23-28} \cline{30-35} \cline{37-42}

10
	& \multirow{5}{*}{4.0} & 6.56 & 180.17 & 0.37 & 0.27 & 0.63 &
	& \multirow{5}{*}{4.0} & 7.00 & 230.36 & 0.52 & 0.20 & 0.48 &
	& \multirow{5}{*}{4.0} & 3.50 & 106.87 & 0.67 & 0.28 & 0.33 &
	& \multirow{5}{*}{4.0} & 3.17 & 45.20 & 0.72 & 0.22 & 0.28 &
	& \multirow{5}{*}{4.0} & 3.33 & 108.24 & 0.67 & 0.11 & 0.33 &
	& \multirow{5}{*}{4.0} & 3.50 & 47.25 & 0.50 & 0.31 & 0.50 \\

30  &
	& 18.78 & 178.05 & 0.48 & 0.19 & 0.52 &  &
	& 20.00 & 228.34 & 0.63 & 0.32 & 0.37 &  &
	& 9.50 & 105.93 & 0.67 & 0.36 & 0.33 &                      &
	& 8.33 & 45.53 & 0.78 & 0.26 & 0.22 &                      &
	& 9.00 & 104.46 & 0.61 & 0.14 & 0.39 &                      &
	& 10.33 & 46.47 & 0.56 & 0.21 & 0.44 \\

50  &
	& 31.00 & 179.45 & 0.44 & 0.22 & 0.56 &  &
	& 32.89 & 191.61 & 0.81 & 0.22 & 0.19 &  &
	& 15.50 & 105.84 & 0.78 & 0.28 & 0.22 &                      &
	& 13.50 & 43.40 & 0.89 & 0.25 & 0.11 &                      &
	& 14.50 & 105.02 & 0.72 & 0.17 & 0.28 &                      &
	& 17.00 & 45.71 & 0.33 & 0.24 & 0.67 \\

70  &
	& 43.56 & 178.79 & 0.41 & 0.22 & 0.59 &  &
	& 46.11 & 193.81 & 0.81 & 0.31 & 0.19 &  &
	& 21.83 & 104.01 & 0.89 & 0.19 & 0.11 &                      &
	& 19.33 & 43.17 & 0.94 & 0.24 & 0.06 &                      &
	& 20.33 & 106.25 & 0.89 & 0.19 & 0.11 &                      &
	& 23.67 & 48.27 & 0.83 & 0.08 & 0.17 \\

100 &
	& 61.56 & 179.52 & 0.33 & 0.25 & 0.67 &  &
	& 65.33 & 247.09 & 1.00 & 0.28 & 0.00 &  &
	& 30.50 & 100.16 & 0.83 & 0.12 & 0.17 &                      &
	& 26.83 & 43.97 & 1.00 & 0.17 & 0.00 &                      &
	& 28.50 & 107.56 & 1.00 & 0.21 & 0.00 &                      &
	& 33.50 & 48.22 & 0.67 & 0.08 & 0.33 \\
\bottomrule
\end{tabular}

\begin{tabular}{cccccccccccccclcccccclcccccclcccccclcccccc}
\toprule
	\multicolumn{42}{c}{\begin{tabular}[c]{@{}c@{}} \textsc{Tireworld} \end{tabular}}	\\ \hline
    & \multicolumn{6}{c}{\begin{tabular}[c]{@{}c@{}} Conjunctive Goals \\ $\phi_1 \land \phi_2$ \end{tabular}}                  &
	& \multicolumn{6}{c}{\begin{tabular}[c]{@{}c@{}} \LTLf Eventuality Goals \\ $\Diamond\phi$ \end{tabular}}                    &
	& \multicolumn{6}{c}{\begin{tabular}[c]{@{}c@{}} \LTLf Ordering \\ $\Diamond(\phi_1 \land \Next(\Diamond\phi_2))$ \end{tabular}}                       & \multicolumn{1}{c}{}
	& \multicolumn{6}{c}{\begin{tabular}[c]{@{}c@{}} \LTLf Goals Until \\ $\phi_1 \lUntil \phi_2$ \end{tabular}}                          & \multicolumn{1}{c}{}
	& \multicolumn{6}{c}{\begin{tabular}[c]{@{}c@{}} \PLTLf Goals Once \\ $\phi_1 \land \past \phi_2$ \end{tabular}}                          & \multicolumn{1}{c}{}
	& \multicolumn{6}{c}{\begin{tabular}[c]{@{}c@{}} \PLTLf Goals Since \\ $\phi_1 \land (\lnot\phi_2 \Since \phi_3)$ \end{tabular}}
	\\ \cline{2-7} \cline{9-14} \cline{16-21} \cline{23-28} \cline{30-35} \cline{37-42}
    & $|\mathcal{G}_{\varphi}|$ & $|Obs|$ & \textit{Time} & \textit{TPR}  & \textit{FPR}  & \textit{FNR}  &
	& $|\mathcal{G}_{\varphi}|$ & $|Obs|$ & \textit{Time} & \textit{TPR}  & \textit{FPR}  & \textit{FNR}  &
	& $|\mathcal{G}_{\varphi}|$ & $|Obs|$ & \textit{Time} & \textit{TPR}  & \textit{FPR}  & \textit{FNR}  &
	& $|\mathcal{G}_{\varphi}|$ & $|Obs|$ & \textit{Time} & \textit{TPR}  & \textit{FPR}  & \textit{FNR}  &
	& $|\mathcal{G}_{\varphi}|$ & $|Obs|$ & \textit{Time} & \textit{TPR}  & \textit{FPR}  & \textit{FNR}  &
	& $|\mathcal{G}_{\varphi}|$ & $|Obs|$ & \textit{Time} & \textit{TPR}  & \textit{FPR}  & \textit{FNR} \\
	\cline{2-7} \cline{9-14} \cline{16-21} \cline{23-28} \cline{30-35} \cline{37-42}

10
	& \multirow{5}{*}{5.5} & 1.50 & 16.88 & 1.00 & 0.29 & 0.00 &
	& \multirow{5}{*}{5.5} & 1.50 & 29.17 & 1.00 & 0.19 & 0.00 &
	& \multirow{5}{*}{3.5} & 1.50 & 12.59 & 0.67 & 0.16 & 0.33 &
	& \multirow{5}{*}{3.5} & 1.17 &  5.11 & 0.72 & 0.08 & 0.28 &
	& \multirow{5}{*}{4.0} & 1.33 &  6.85 & 0.56 & 0.19 & 0.44 &
	& \multirow{5}{*}{4.0} & 1.17 & 18.00 & 0.67 & 0.12 & 0.33 \\

30  &
	& 3.50 & 17.19 & 1.00 & 0.04 & 0.00 &  &
	& 3.50 & 26.39 & 1.00 & 0.07 & 0.00 &  &
	& 3.50 & 11.90 & 0.72 & 0.16 & 0.28 &                      &
	& 3.17 &  5.17 & 0.94 & 0.03 & 0.06 &                      &
	& 3.50 &  6.95 & 0.83 & 0.10 & 0.17 &                      &
	& 3.17 & 18.69 & 0.89 & 0.07 & 0.11 \\

50  &
	& 6.00 & 17.35 & 1.00 & 0.01 & 0.00 &  &
	& 6.00 & 21.97 & 1.00 & 0.04 & 0.00 &  &
	& 5.67 & 10.18 & 0.89 & 0.05 & 0.11 &                      &
	& 4.83 &  5.10 & 0.94 & 0.01 & 0.06 &                      &
	& 5.50 &  6.87 & 0.83 & 0.10 & 0.17 &                      &
	& 4.83 & 19.09 & 0.94 & 0.04 & 0.06 \\

70  &
	& 8.50 & 17.31 & 1.00 & 0.01 & 0.00 &  &
	& 8.50 & 20.30 & 1.00 & 0.00 & 0.00 &  &
	& 7.83 & 10.20 & 0.94 & 0.02 & 0.06 &                      &
	& 6.50 &  5.16 & 1.00 & 0.06 & 0.00 &                      &
	& 7.67 &  6.82 & 0.83 & 0.07 & 0.17 &                      &
	& 6.50 & 19.23 & 1.00 & 0.04 & 0.00 \\

100 &
	& 11.50 & 17.28 & 1.00 & 0.00 & 0.00 &  &
	& 11.50 & 20.97 & 1.00 & 0.00 & 0.00 &  &
	& 10.50 & 10.15 & 1.00 & 0.00 & 0.00 &                      &
	& 9.00 &  5.55 & 1.00 & 0.08 & 0.00 &                      &
	& 9.00 &  5.55 & 1.00 & 0.08 & 0.00 &                      &
	& 9.00 & 19.17 & 1.00 & 0.04 & 0.00\\
\bottomrule
\end{tabular}

\begin{tabular}{cccccccccccccclcccccclcccccclcccccclcccccc}
\toprule
	\multicolumn{42}{c}{\begin{tabular}[c]{@{}c@{}} \textsc{Triangle-Tireworld} \end{tabular}}	\\ \hline
    & \multicolumn{6}{c}{\begin{tabular}[c]{@{}c@{}} Conjunctive Goals \\ $\phi_1 \land \phi_2$ \end{tabular}}                  &
	& \multicolumn{6}{c}{\begin{tabular}[c]{@{}c@{}} \LTLf Eventuality Goals \\ $\Diamond\phi$ \end{tabular}}                    &
	& \multicolumn{6}{c}{\begin{tabular}[c]{@{}c@{}} \LTLf Ordering \\ $\Diamond(\phi_1 \land \Next(\Diamond\phi_2))$ \end{tabular}}                       & \multicolumn{1}{c}{}
	& \multicolumn{6}{c}{\begin{tabular}[c]{@{}c@{}} \LTLf Goals Until \\ $\phi_1 \lUntil \phi_2$ \end{tabular}}                          & \multicolumn{1}{c}{}
	& \multicolumn{6}{c}{\begin{tabular}[c]{@{}c@{}} \PLTLf Goals Once \\ $\phi_1 \land \past \phi_2$ \end{tabular}}                          & \multicolumn{1}{c}{}
	& \multicolumn{6}{c}{\begin{tabular}[c]{@{}c@{}} \PLTLf Goals Since \\ $\phi_1 \land (\lnot\phi_2 \Since \phi_3)$ \end{tabular}}
	\\ \cline{2-7} \cline{9-14} \cline{16-21} \cline{23-28} \cline{30-35} \cline{37-42}
    & $|\mathcal{G}_{\varphi}|$ & $|Obs|$ & \textit{Time} & \textit{TPR}  & \textit{FPR}  & \textit{FNR}  &
	& $|\mathcal{G}_{\varphi}|$ & $|Obs|$ & \textit{Time} & \textit{TPR}  & \textit{FPR}  & \textit{FNR}  &
	& $|\mathcal{G}_{\varphi}|$ & $|Obs|$ & \textit{Time} & \textit{TPR}  & \textit{FPR}  & \textit{FNR}  &
	& $|\mathcal{G}_{\varphi}|$ & $|Obs|$ & \textit{Time} & \textit{TPR}  & \textit{FPR}  & \textit{FNR}  &
	& $|\mathcal{G}_{\varphi}|$ & $|Obs|$ & \textit{Time} & \textit{TPR}  & \textit{FPR}  & \textit{FNR}  &
	& $|\mathcal{G}_{\varphi}|$ & $|Obs|$ & \textit{Time} & \textit{TPR}  & \textit{FPR}  & \textit{FNR} \\
	\cline{2-7} \cline{9-14} \cline{16-21} \cline{23-28} \cline{30-35} \cline{37-42}

10
	& \multirow{5}{*}{3.75} & 1.67 & 16.56 & 0.64 & 0.16 & 0.36 &
	& \multirow{5}{*}{3.75} & 2.08 & 34.57 & 0.69 & 0.13 & 0.31 &
	& \multirow{5}{*}{4.0} & 1.67 & 14.42 & 0.44 & 0.14 & 0.56 &
	& \multirow{5}{*}{4.0} & 1.17 & 11.04 & 0.72 & 0.11 & 0.28 &
	& \multirow{5}{*}{4.0} & 1.67 & 15.91 & 0.67 & 0.11 & 0.33 &
	& \multirow{5}{*}{4.0} & 1.00 & 106.56 & 0.78 & 0.26 & 0.22 \\

30  &
	& 4.67 & 16.90 & 0.86 & 0.03 & 0.14 &  &
	& 5.58 & 31.76 & 0.86 & 0.03 & 0.14 &  &
	& 3.83 & 14.57 & 0.94 & 0.01 & 0.06 &                      &
	& 2.83 & 10.55 & 1.00 & 0.03 & 0.00 &                      &
	& 3.83 & 15.87 & 0.72 & 0.07 & 0.28 &                      &
	& 2.33 & 104.83 & 0.94 & 0.17 & 0.06 \\

50  &
	& 7.33 & 17.08 & 0.89 & 0.03 & 0.11 &  &
	& 8.83 & 30.83 & 0.92 & 0.02 & 0.08 &  &
	& 6.17 & 14.76 & 0.83 & 0.04 & 0.17 &                      &
	& 4.50 & 10.55 & 0.94 & 0.01 & 0.06 &                      &
	& 6.17 & 16.13 & 0.89 & 0.03 & 0.11 &                      &
	& 3.50 & 102.41 & 1.00 & 0.12 & 0.00 \\

70 &
	& 10.00 & 17.12 & 1.00 & 0.00 & 0.00 &  &
	& 12.08 & 32.81 & 1.00 & 0.00 & 0.00 &  &
	& 8.50 & 17.43 & 1.00 & 0.00 & 0.00 &                      &
	& 6.17 & 10.58 & 1.00 & 0.00 & 0.00 &                      &
	& 8.50 & 14.87 & 1.00 & 0.00 & 0.00 &                      &
	& 4.67 & 102.36 & 1.00 & 0.25 & 0.00 \\

100  &
	& 13.67 & 17.16 & 1.00 & 0.00 & 0.00 &  &
	& 16.67 & 30.93 & 1.00 & 0.00 & 0.00 &  &
	& 11.33 & 24.17 & 1.00 & 0.00 & 0.00 &                      &
	& 8.00 & 11.89 & 1.00 & 0.00 & 0.00 &                      &
	& 11.33 & 15.16 & 1.00 & 0.00 & 0.00 &                      &
	& 6.00 & 102.91 & 1.00 & 0.00 & 0.00 \\
\bottomrule
\end{tabular}

\begin{tabular}{cccccccccccccclcccccclcccccclcccccclcccccc}
\toprule
	\multicolumn{42}{c}{\begin{tabular}[c]{@{}c@{}} \textsc{Zeno-Travel} \end{tabular}}	\\ \hline
    & \multicolumn{6}{c}{\begin{tabular}[c]{@{}c@{}} Conjunctive Goals \\ $\phi_1 \land \phi_2$ \end{tabular}}                  &
	& \multicolumn{6}{c}{\begin{tabular}[c]{@{}c@{}} \LTLf Eventuality Goals \\ $\Diamond\phi$ \end{tabular}}                    &
	& \multicolumn{6}{c}{\begin{tabular}[c]{@{}c@{}} \LTLf Ordering \\ $\Diamond(\phi_1 \land \Next(\Diamond\phi_2))$ \end{tabular}}                       & \multicolumn{1}{c}{}
	& \multicolumn{6}{c}{\begin{tabular}[c]{@{}c@{}} \LTLf Goals Until \\ $\phi_1 \lUntil \phi_2$ \end{tabular}}                          & \multicolumn{1}{c}{}
	& \multicolumn{6}{c}{\begin{tabular}[c]{@{}c@{}} \PLTLf Goals Once \\ $\phi_1 \land \past \phi_2$ \end{tabular}}                          & \multicolumn{1}{c}{}
	& \multicolumn{6}{c}{\begin{tabular}[c]{@{}c@{}} \PLTLf Goals Since \\ $\phi_1 \land (\lnot\phi_2 \Since \phi_3)$ \end{tabular}}
	\\ \cline{2-7} \cline{9-14} \cline{16-21} \cline{23-28} \cline{30-35} \cline{37-42}
    & $|\mathcal{G}_{\varphi}|$ & $|Obs|$ & \textit{Time} & \textit{TPR}  & \textit{FPR}  & \textit{FNR}  &
	& $|\mathcal{G}_{\varphi}|$ & $|Obs|$ & \textit{Time} & \textit{TPR}  & \textit{FPR}  & \textit{FNR}  &
	& $|\mathcal{G}_{\varphi}|$ & $|Obs|$ & \textit{Time} & \textit{TPR}  & \textit{FPR}  & \textit{FNR}  &
	& $|\mathcal{G}_{\varphi}|$ & $|Obs|$ & \textit{Time} & \textit{TPR}  & \textit{FPR}  & \textit{FNR}  &
	& $|\mathcal{G}_{\varphi}|$ & $|Obs|$ & \textit{Time} & \textit{TPR}  & \textit{FPR}  & \textit{FNR}  &
	& $|\mathcal{G}_{\varphi}|$ & $|Obs|$ & \textit{Time} & \textit{TPR}  & \textit{FPR}  & \textit{FNR} \\
	\cline{2-7} \cline{9-14} \cline{16-21} \cline{23-28} \cline{30-35} \cline{37-42}

10
	& \multirow{5}{*}{7.5} & 5.67 & 556.36 & 0.89 & 0.08 & 0.11 &
	& \multirow{5}{*}{7.5} & 5.33 & 607.20 & 0.81 & 0.05 & 0.19 &
	& \multirow{5}{*}{4.0} & 2.67 & 145.81 & 0.94 & 0.01 & 0.06 &
	& \multirow{5}{*}{4.0} & 2.50 & 174.38 & 0.89 & 0.03 & 0.11 &
	& \multirow{5}{*}{4.0} & 2.17 & 144.87 & 0.83 & 0.04 & 0.17 &
	& \multirow{5}{*}{4.0} & 2.00 & 175.58 & 0.89 & 0.14 & 0.11 \\

30  &
	& 16.25 & 549.34 & 1.00 & 0.04 & 0.00 &  &
	& 15.17 & 619.13 & 0.94 & 0.02 & 0.06 &  &
	& 7.50 & 145.33 & 1.00 & 0.00 & 0.00 &                      &
	& 6.50 & 167.33 & 0.89 & 0.03 & 0.11 &                      &
	& 6.17 & 140.15 & 0.94 & 0.01 & 0.06 &                      &
	& 5.33 & 171.27 & 1.00 & 0.10 & 0.00 \\

50  &
	& 26.50 & 554.09 & 1.00 & 0.02 & 0.00 &  &
	& 24.75 & 670.07 & 0.97 & 0.02 & 0.03 &  &
	& 11.67 & 141.81 & 1.00 & 0.00 & 0.00 &                      &
	& 10.17 & 164.68 & 0.83 & 0.04 & 0.17 &                      &
	& 9.83 & 142.72 & 0.94 & 0.01 & 0.06 &                      &
	& 8.67 & 167.39 & 0.94 & 0.04 & 0.06 \\

70  &
	& 37.50 & 556.76 & 1.00 & 0.03 & 0.00 &  &
	& 34.92 & 619.24 & 1.00 & 0.02 & 0.00 &  &
	& 16.67 & 138.04 & 1.00 & 0.00 & 0.00 &                      &
	& 14.33 & 161.67 & 0.89 & 0.03 & 0.11 &                      &
	& 13.83 & 138.64 & 1.00 & 0.00 & 0.00 &                      &
	& 12.67 & 159.45 & 1.00 & 0.00 & 0.00 \\

100 &
	& 53.00 & 569.43 & 1.00 & 0.02 & 0.00 &  &
	& 49.42 & 735.20 & 1.00 & 0.02 & 0.00 &  &
	& 23.17 & 136.37 & 1.00 & 0.00 & 0.00 &                      &
	& 20.00 & 161.05 & 0.83 & 0.04 & 0.17 &                      &
	& 19.50 & 137.53 & 1.00 & 0.00 & 0.00 &                      &
	& 17.33 & 154.69 & 1.00 & 0.00 & 0.00 \\
\bottomrule
\end{tabular}
\caption{Offline Recognition experimental results for all six \FOND domains separately.}
\label{tab:gr_results_separately}
\end{table}
\end{landscape}
}

%--------------------------------------------------------------------

\newpage
\begin{figure*}[!ht]
	\centering
	\begin{subfigure}[b]{0.185\textwidth}
 	    \includegraphics[width=\textwidth]{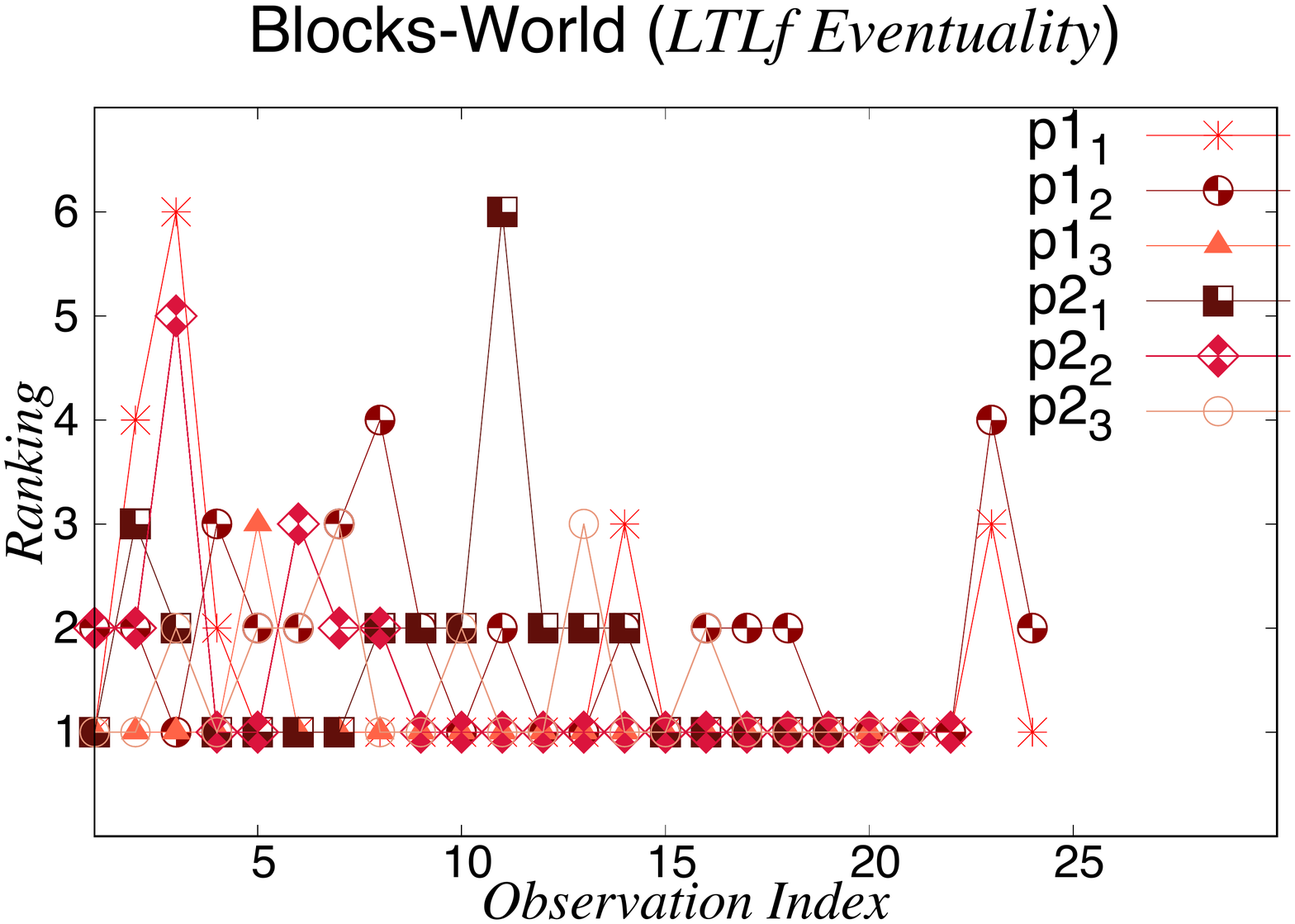}
		\caption{\LTLf Eventuality.}
		\label{fig:blocks-ltl0}
	\end{subfigure}
	~
	\begin{subfigure}[b]{0.185\textwidth}
 	    \includegraphics[width=\textwidth]{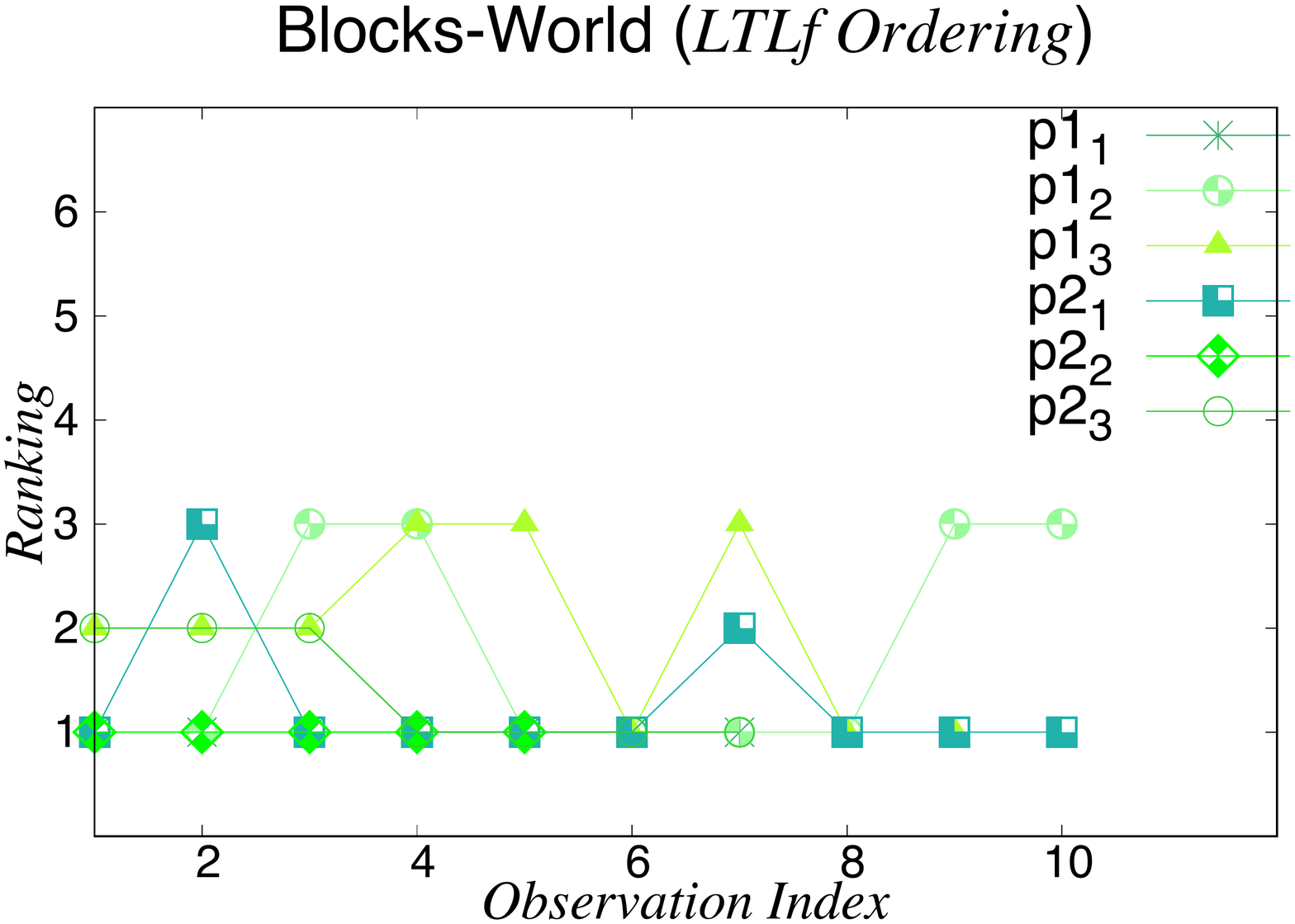}
		\caption{\LTLf Ordering.}
		\label{fig:blocks-ltl1}
	\end{subfigure}
	~
	\begin{subfigure}[b]{0.185\textwidth}
 	    \includegraphics[width=\textwidth]{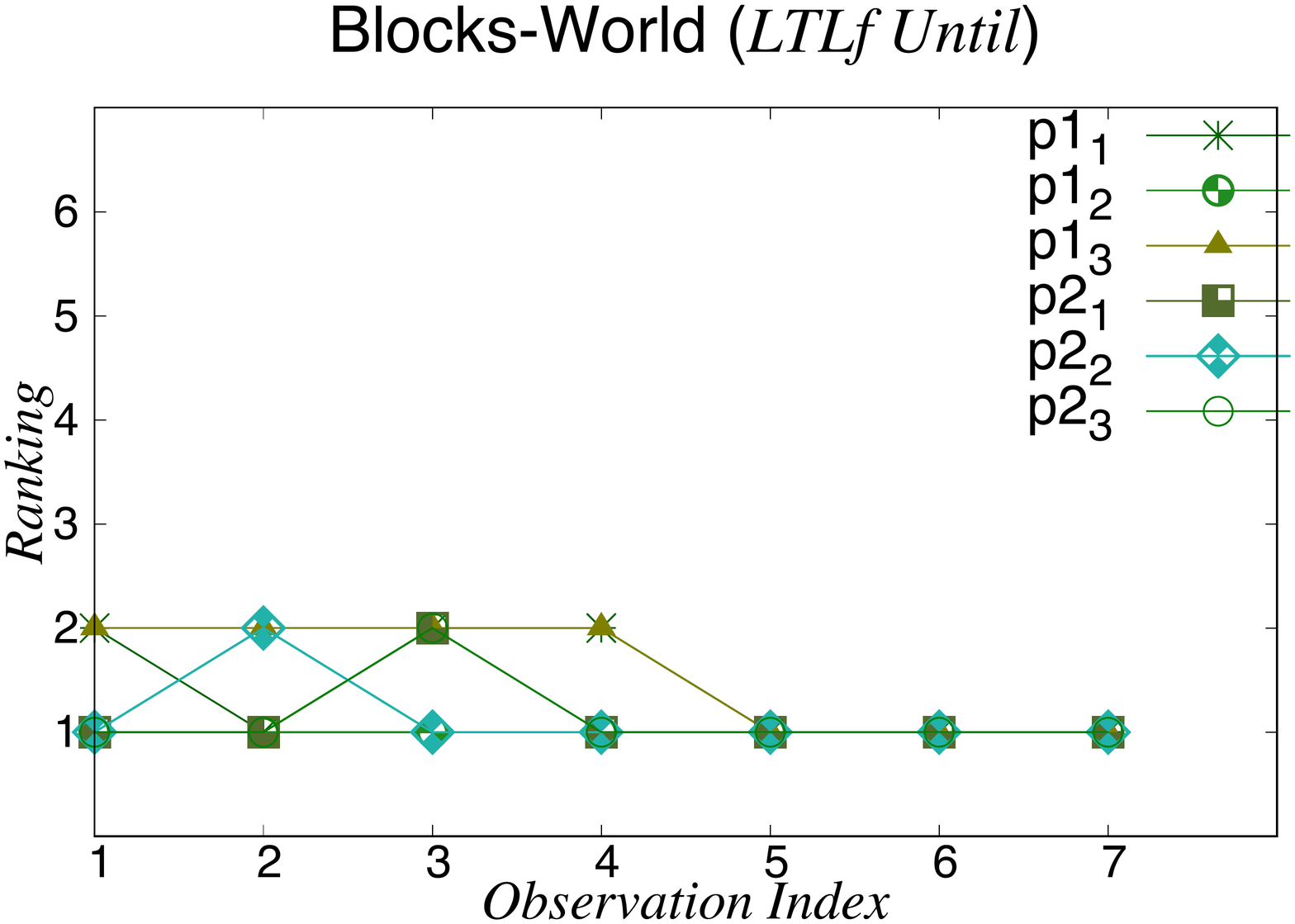}
		\caption{\LTLf Until.}
		\label{fig:blocks-ltl2}
	\end{subfigure}
	~
	\begin{subfigure}[b]{0.185\textwidth}
 	    \includegraphics[width=\textwidth]{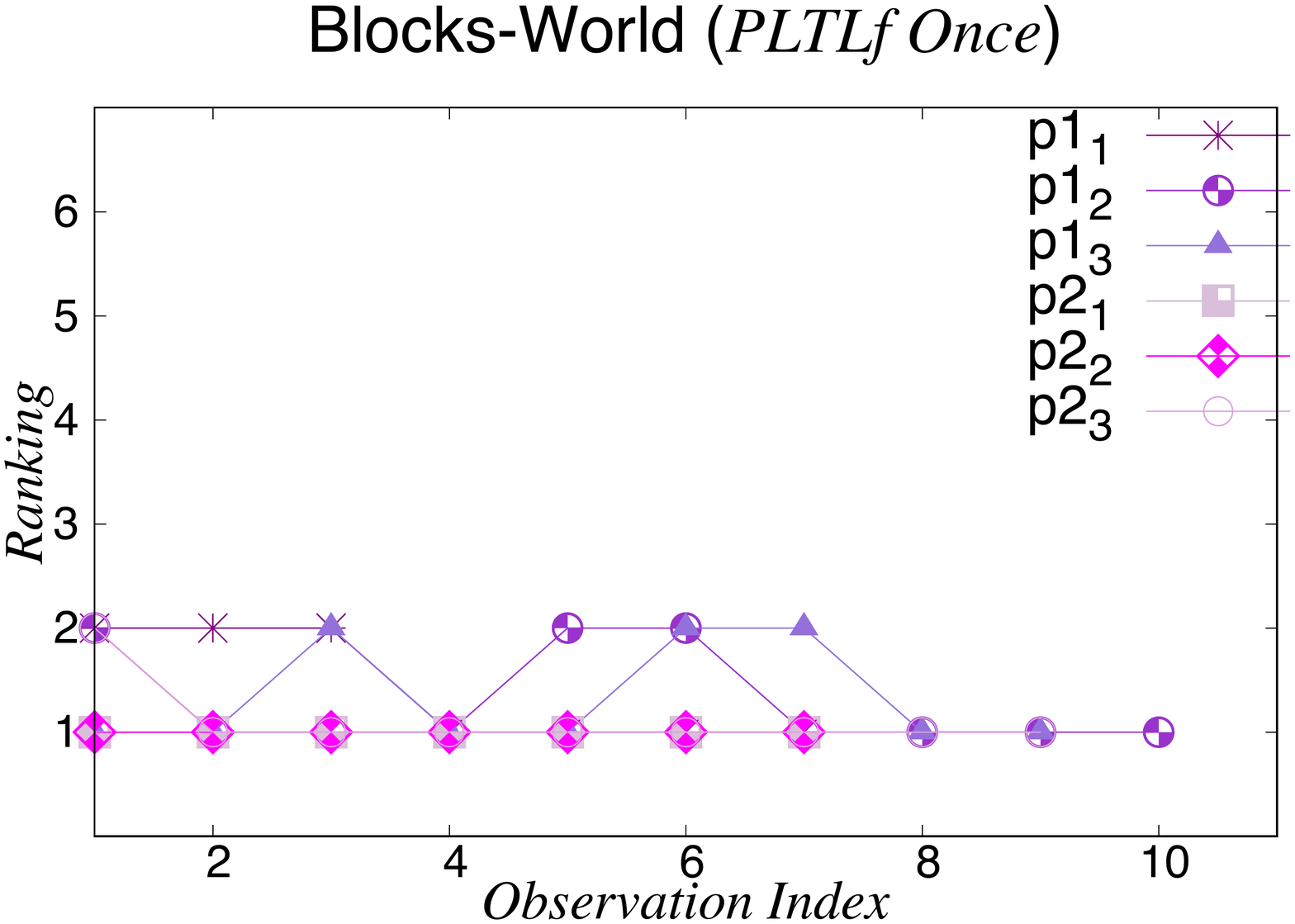}
		\caption{\PLTLf Once.}
		\label{fig:blocks-pltl0}
	\end{subfigure}
	~
	\begin{subfigure}[b]{0.185\textwidth}
 	    \includegraphics[width=\textwidth]{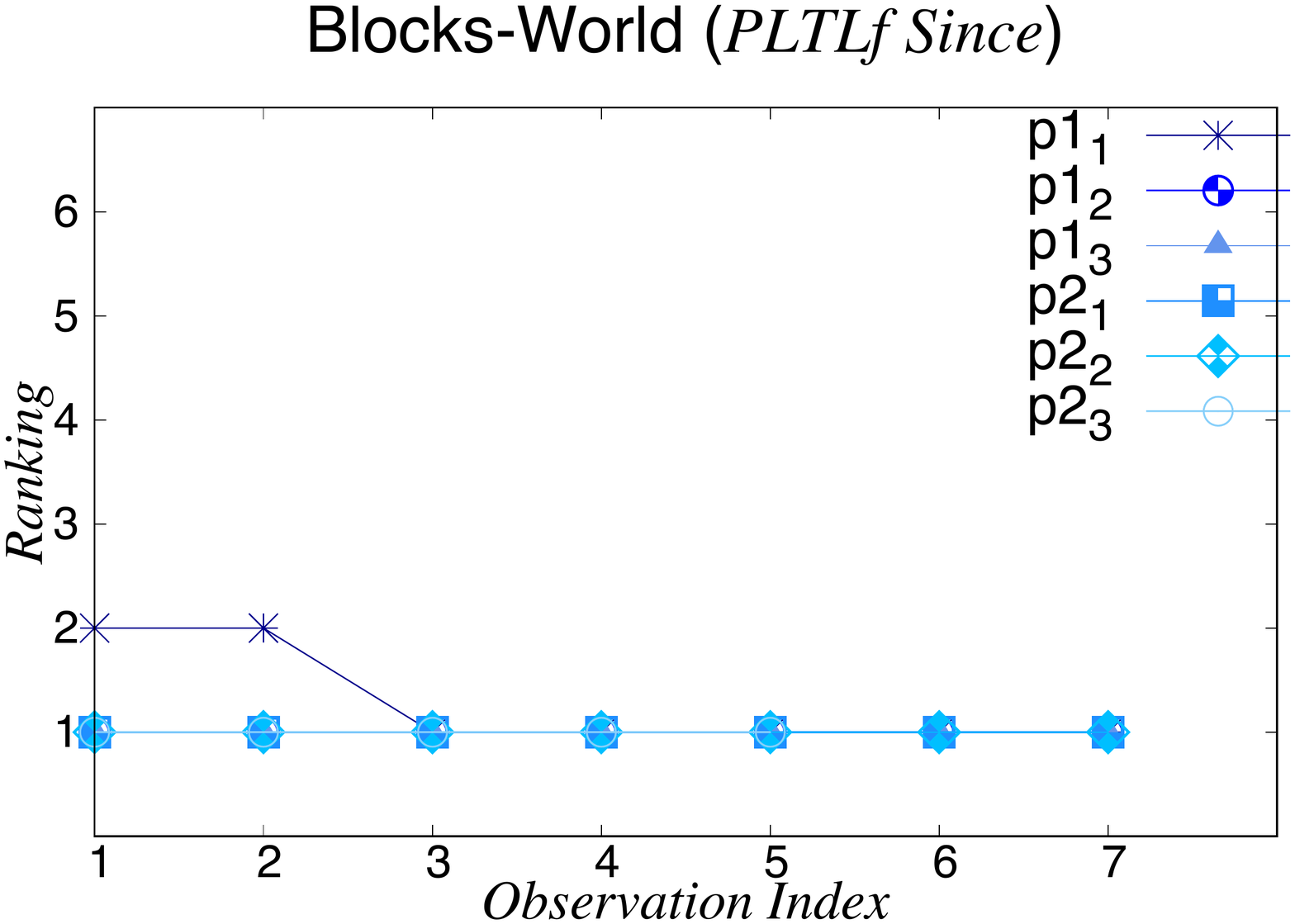}
		\caption{\PLTLf Since.}
		\label{fig:blocks-pltl1}
	\end{subfigure}	
	\vspace{-3mm}
	\caption{Online recognition ranking over the observations for \textsc{Blocks-World}.}
	\label{fig:blocks_ranking}
\end{figure*}

%--------------------------------------------------------------------

\begin{figure*}[!ht]
	\centering
	\begin{subfigure}[b]{0.185\textwidth}
 	    \includegraphics[width=\textwidth]{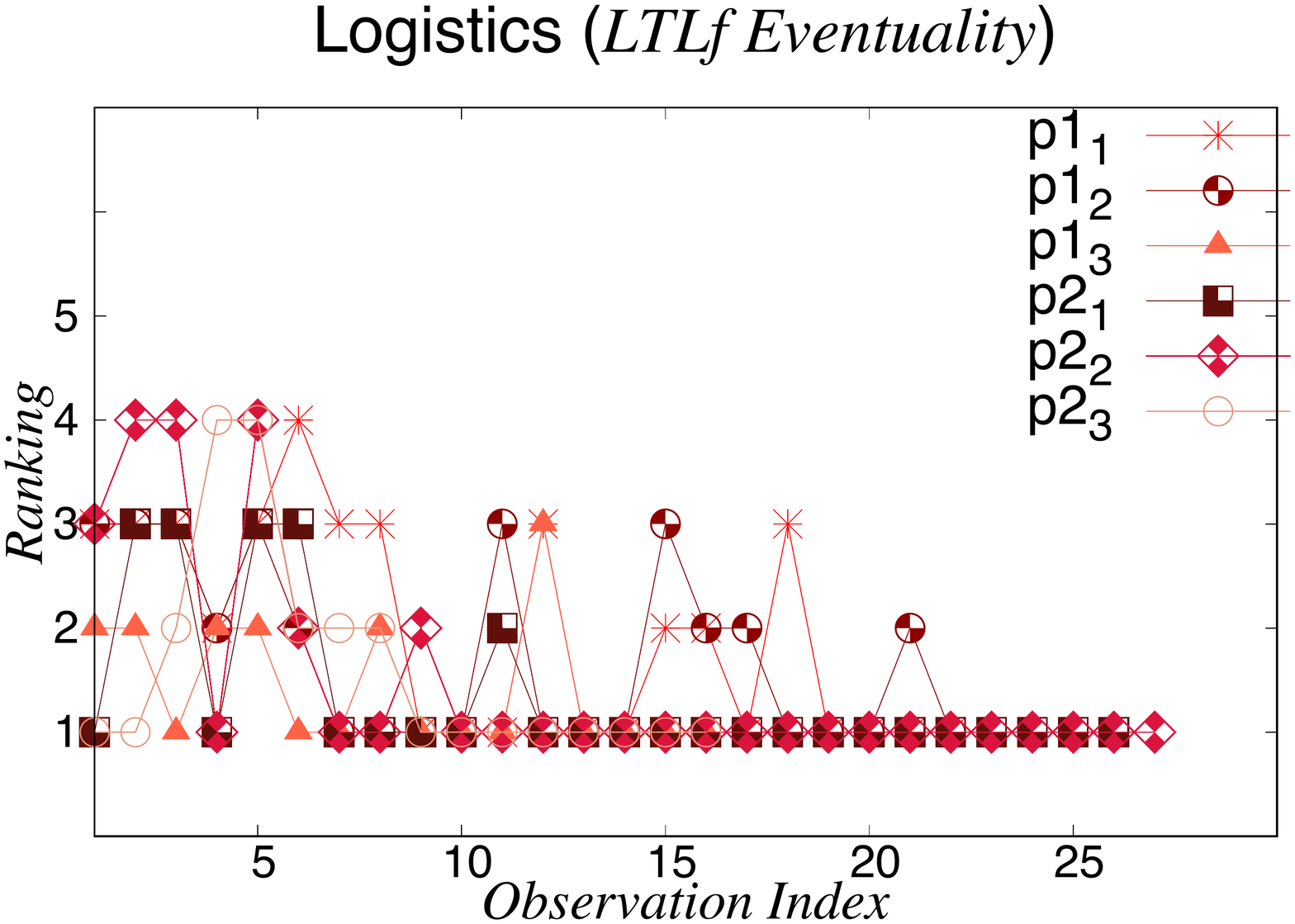}
		\caption{\LTLf Eventuality.}
		\label{fig:logistics-ltl0}
	\end{subfigure}
	~
	\begin{subfigure}[b]{0.185\textwidth}
 	    \includegraphics[width=\textwidth]{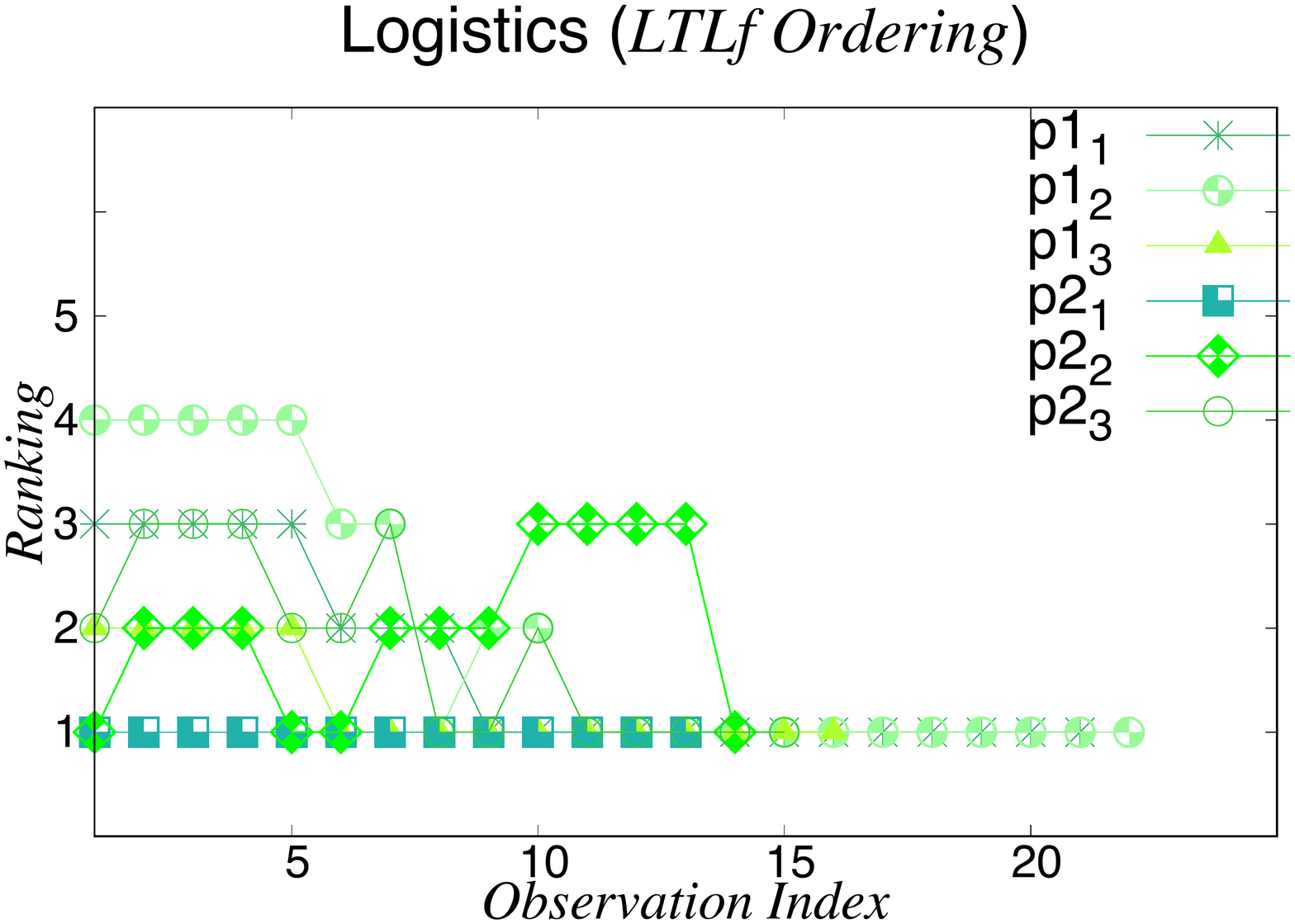}
		\caption{\LTLf Ordering.}
		\label{fig:logistics-ltl1}
	\end{subfigure}
	~
	\begin{subfigure}[b]{0.185\textwidth}
 	    \includegraphics[width=\textwidth]{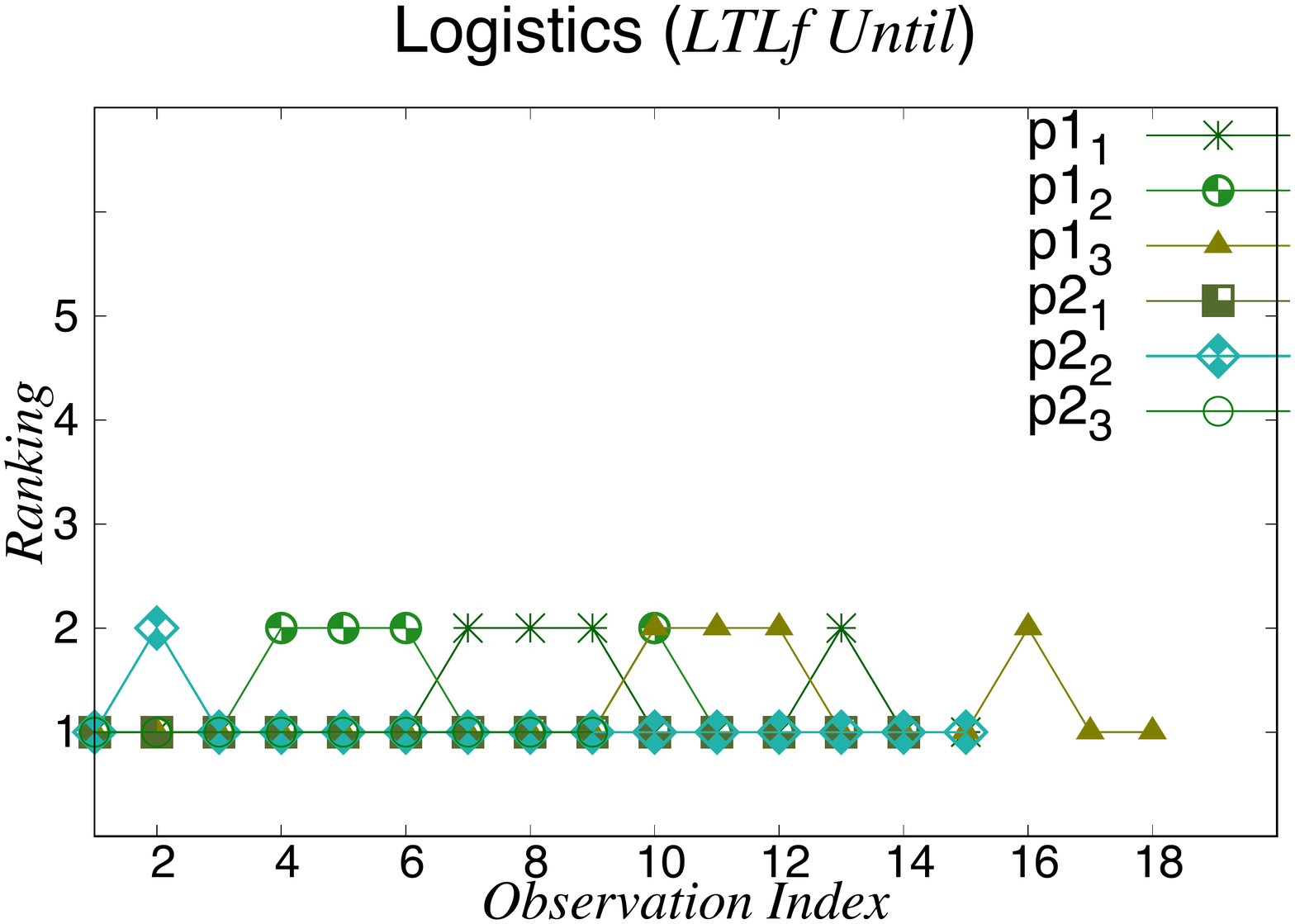}
		\caption{\LTLf Until.}
		\label{fig:logistics-ltl2}
	\end{subfigure}
	~
	\begin{subfigure}[b]{0.185\textwidth}
 	    \includegraphics[width=\textwidth]{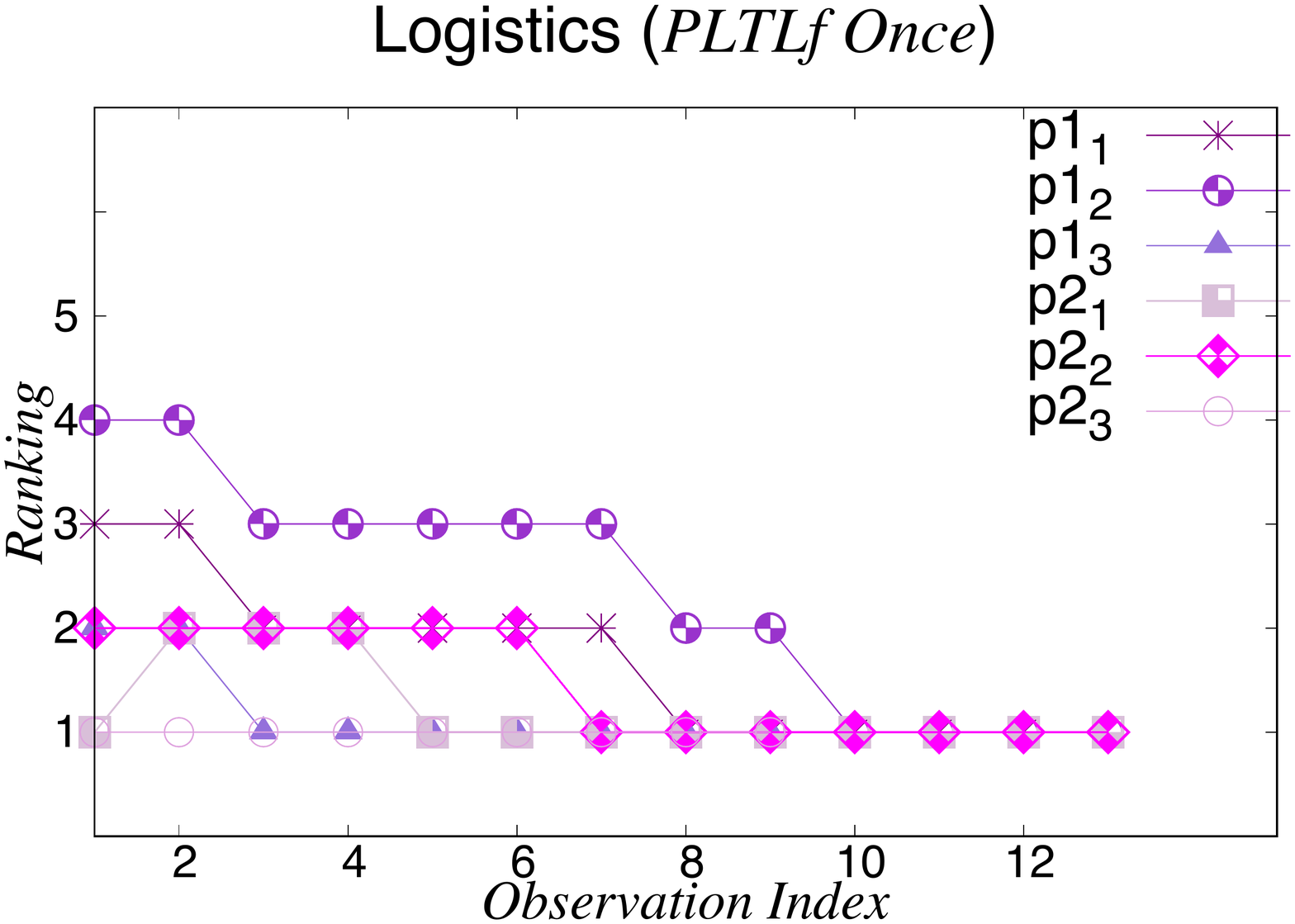}
		\caption{\PLTLf Once.}
		\label{fig:logistics-pltl0}
	\end{subfigure}
	~
	\begin{subfigure}[b]{0.185\textwidth}
 	    \includegraphics[width=\textwidth]{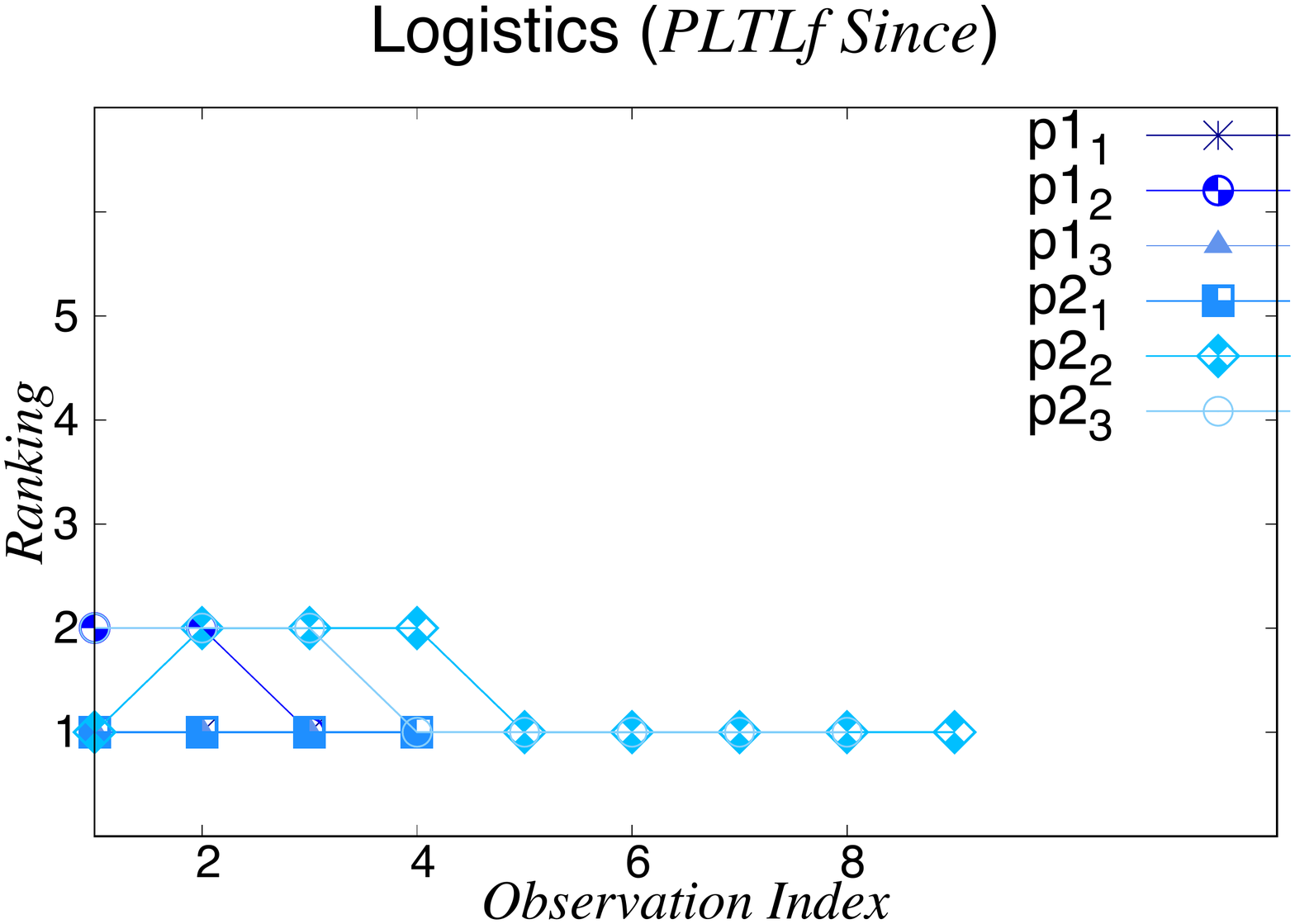}
		\caption{\PLTLf Since.}
		\label{fig:logistics-pltl1}
	\end{subfigure}
	\vspace{-3mm}
	\caption{Online recognition ranking over the observations for \textsc{Logistics}.}
	\label{fig:logistics_ranking}
\end{figure*}

%--------------------------------------------------------------------

\begin{figure*}[!ht]
	\centering
	\begin{subfigure}[b]{0.185\textwidth}
 	    \includegraphics[width=\textwidth]{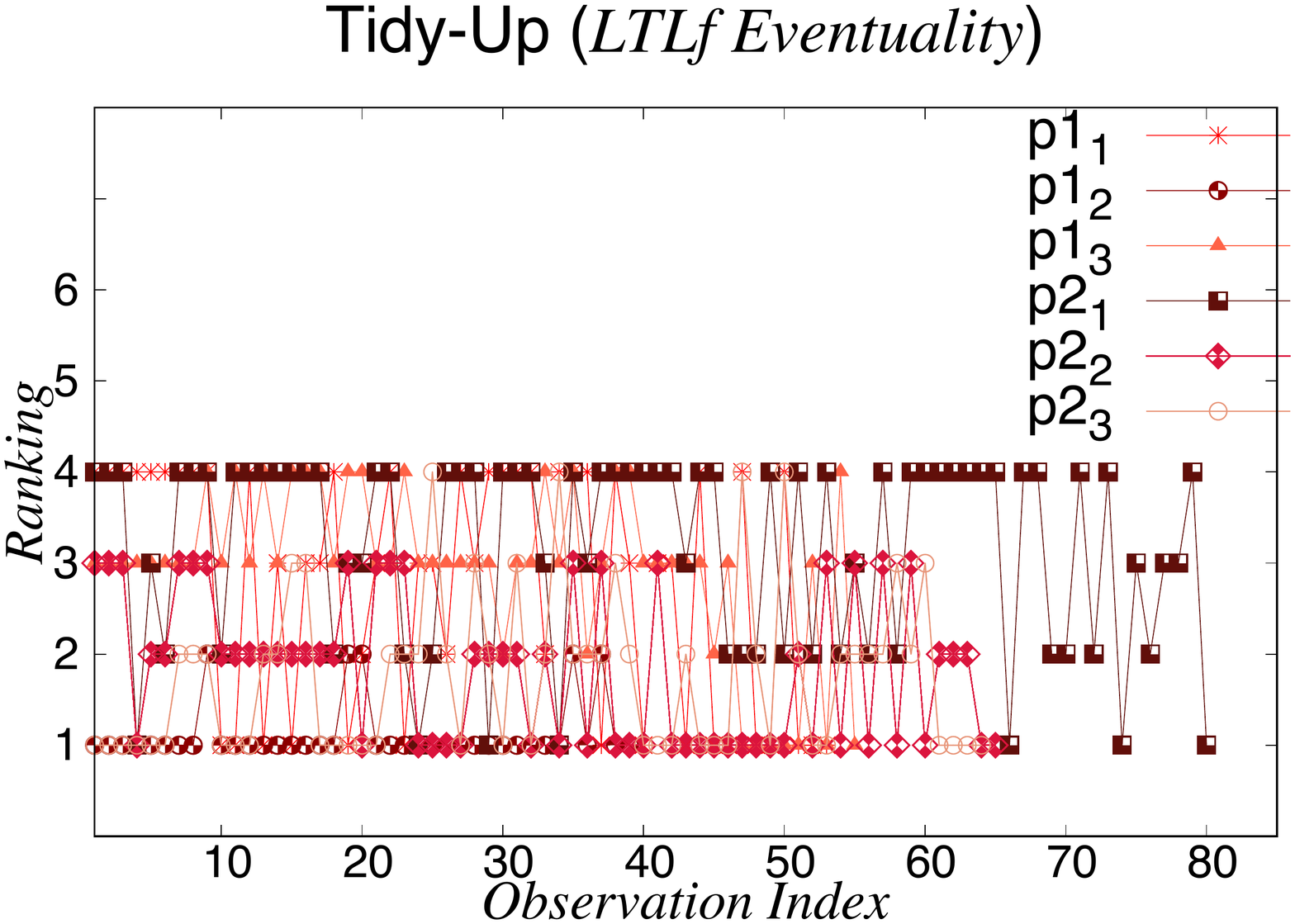}
		\caption{\LTLf Eventuality.}
		\label{fig:tidyup-ltl0}
	\end{subfigure}
	~
	\begin{subfigure}[b]{0.185\textwidth}
 	    \includegraphics[width=\textwidth]{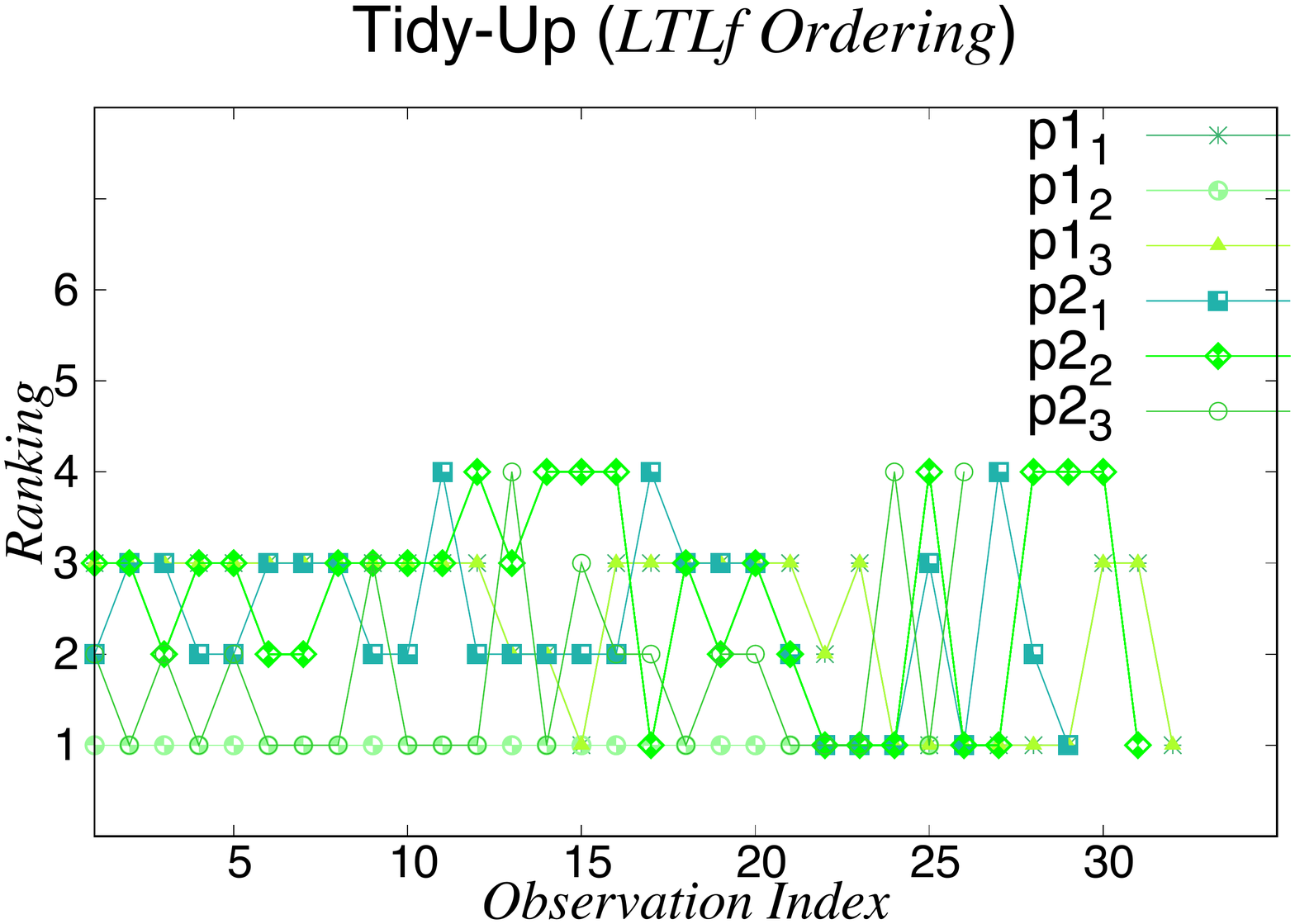}
		\caption{\LTLf Ordering.}
		\label{fig:tidyup-ltl1}
	\end{subfigure}
	~
	\begin{subfigure}[b]{0.185\textwidth}
 	    \includegraphics[width=\textwidth]{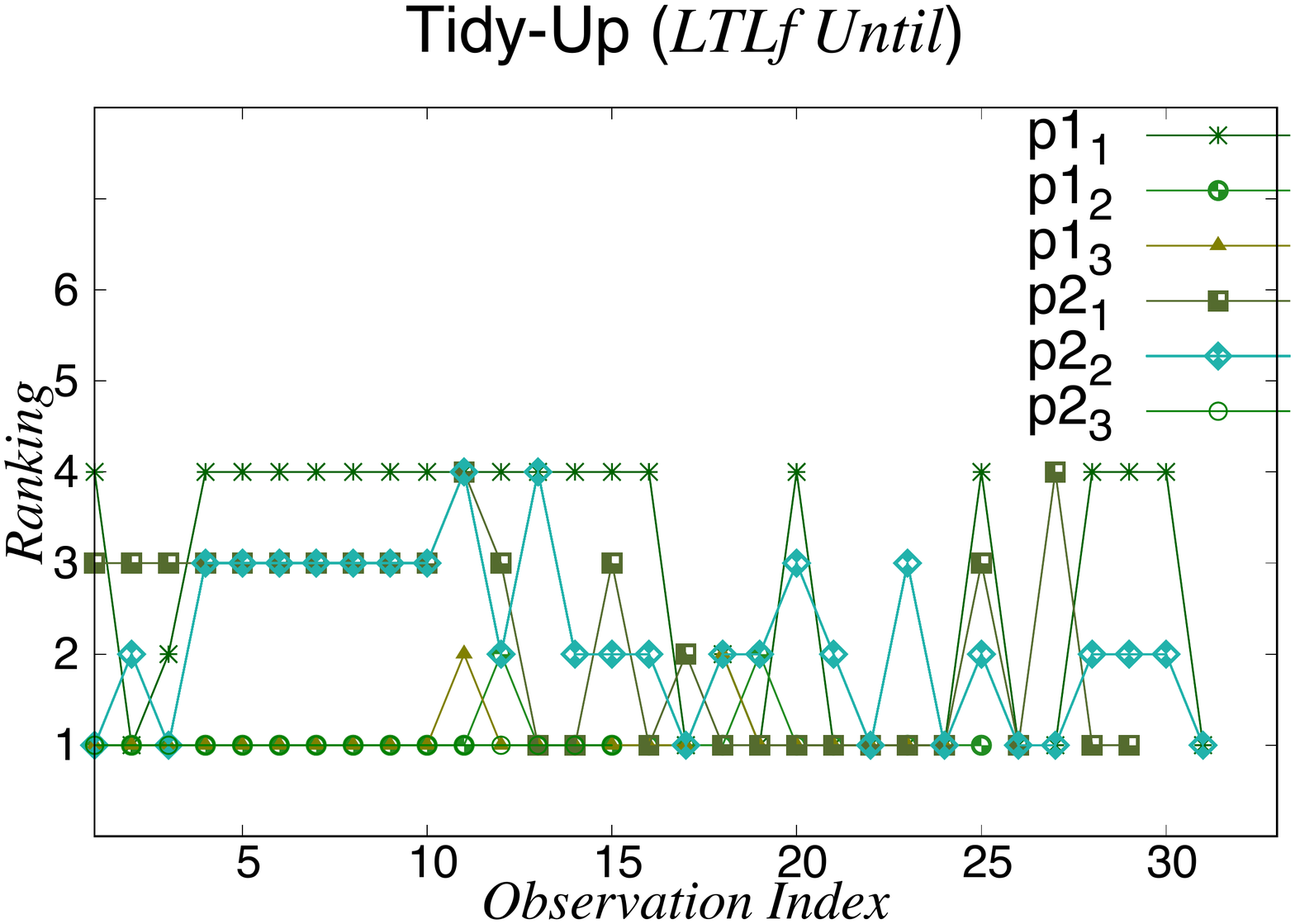}
		\caption{\LTLf Until.}
		\label{fig:tidyup-ltl2}
	\end{subfigure}
	~
	\begin{subfigure}[b]{0.185\textwidth}
 	    \includegraphics[width=\textwidth]{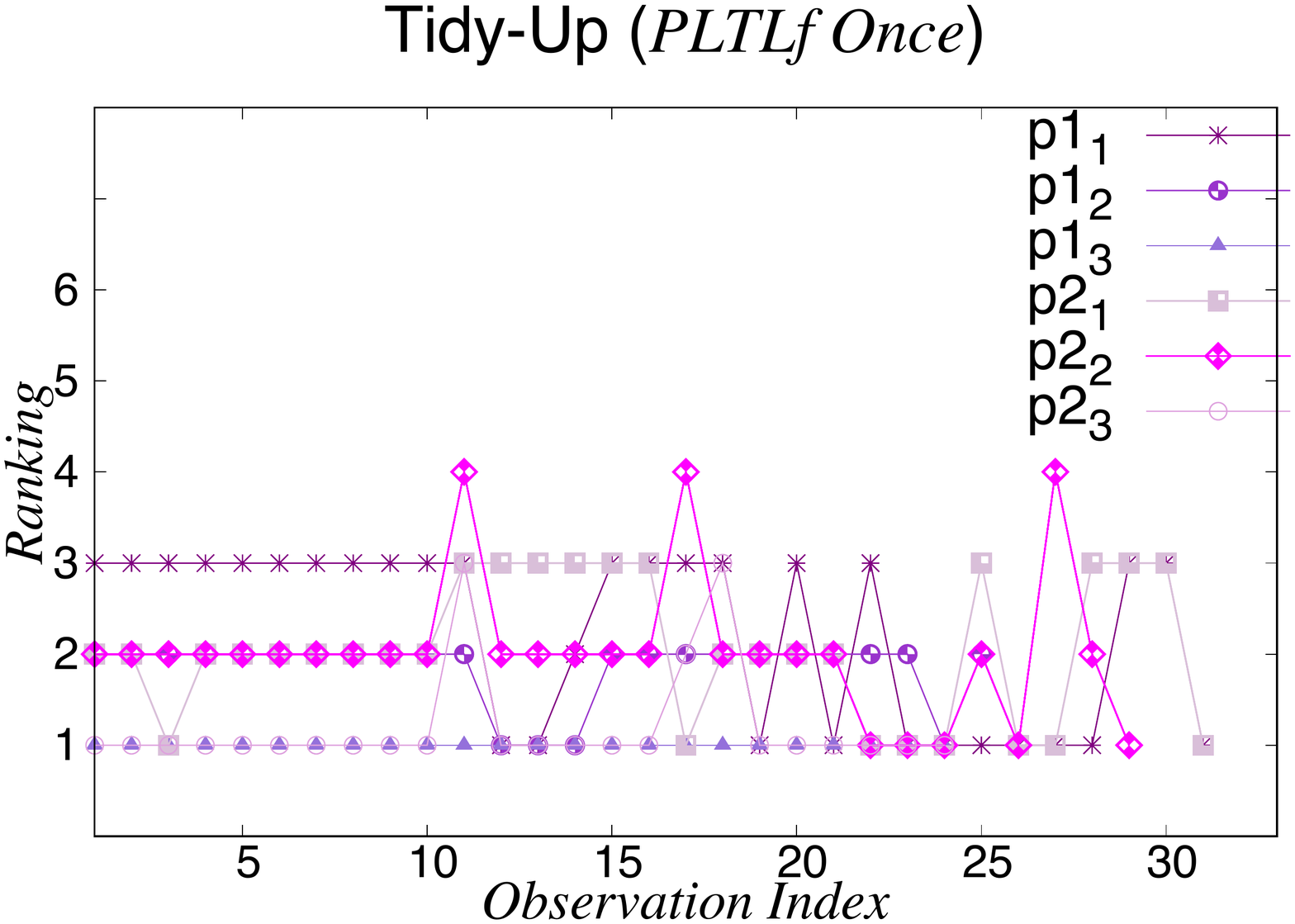}
		\caption{\PLTLf Once.}
		\label{fig:tidyup-pltl0}
	\end{subfigure}
	~
	\begin{subfigure}[b]{0.185\textwidth}
 	    \includegraphics[width=\textwidth]{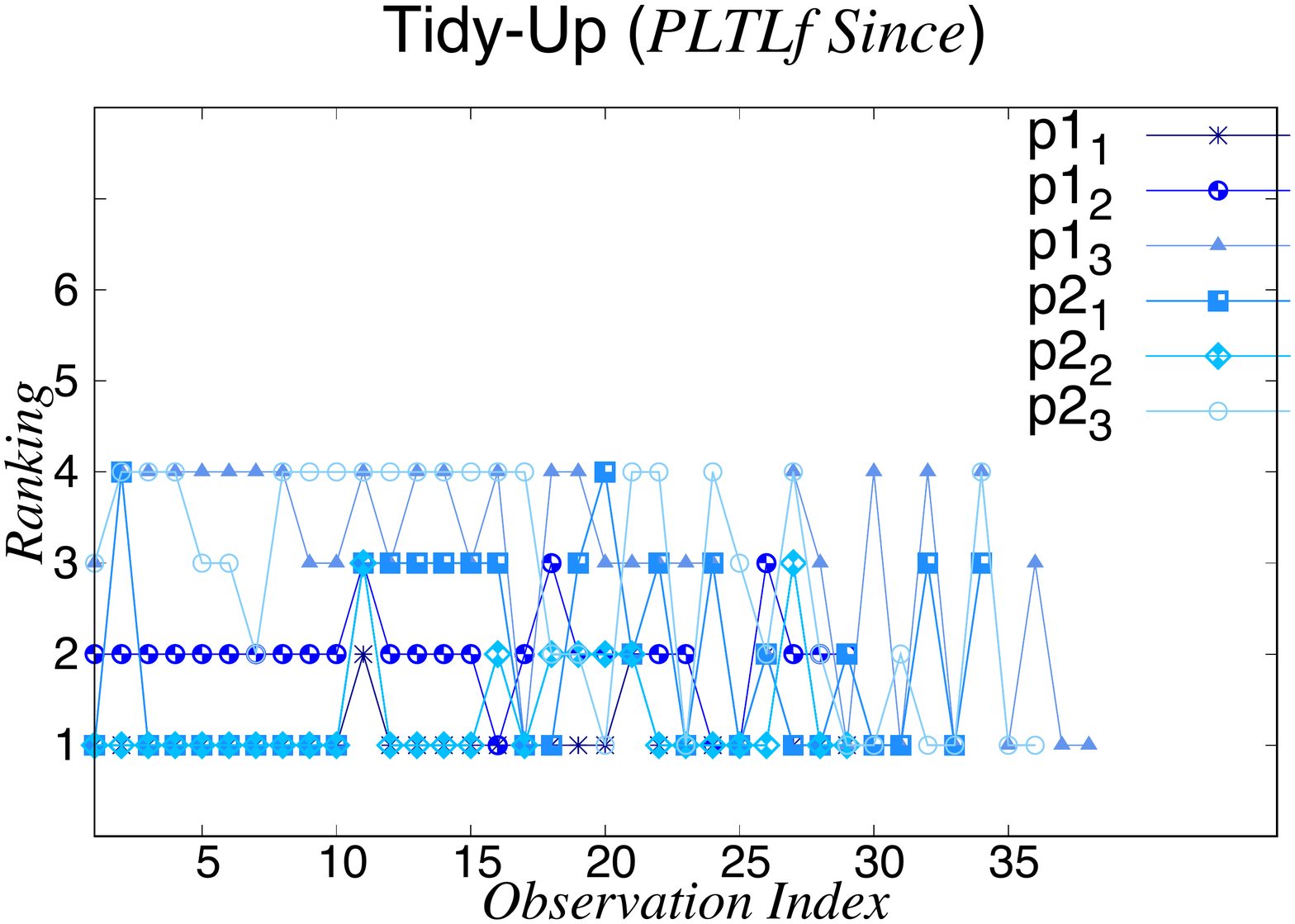}
		\caption{\PLTLf Since.}
		\label{fig:tidyup-pltl1}
	\end{subfigure}
	\vspace{-3mm}
	\caption{Online recognition ranking over the observations for \textsc{Tidy-Up}.}
	\label{fig:tidyup_ranking}
\end{figure*}

%--------------------------------------------------------------------

\begin{figure*}[!ht]
	\centering
	\begin{subfigure}[b]{0.185\textwidth}
 	    \includegraphics[width=\textwidth]{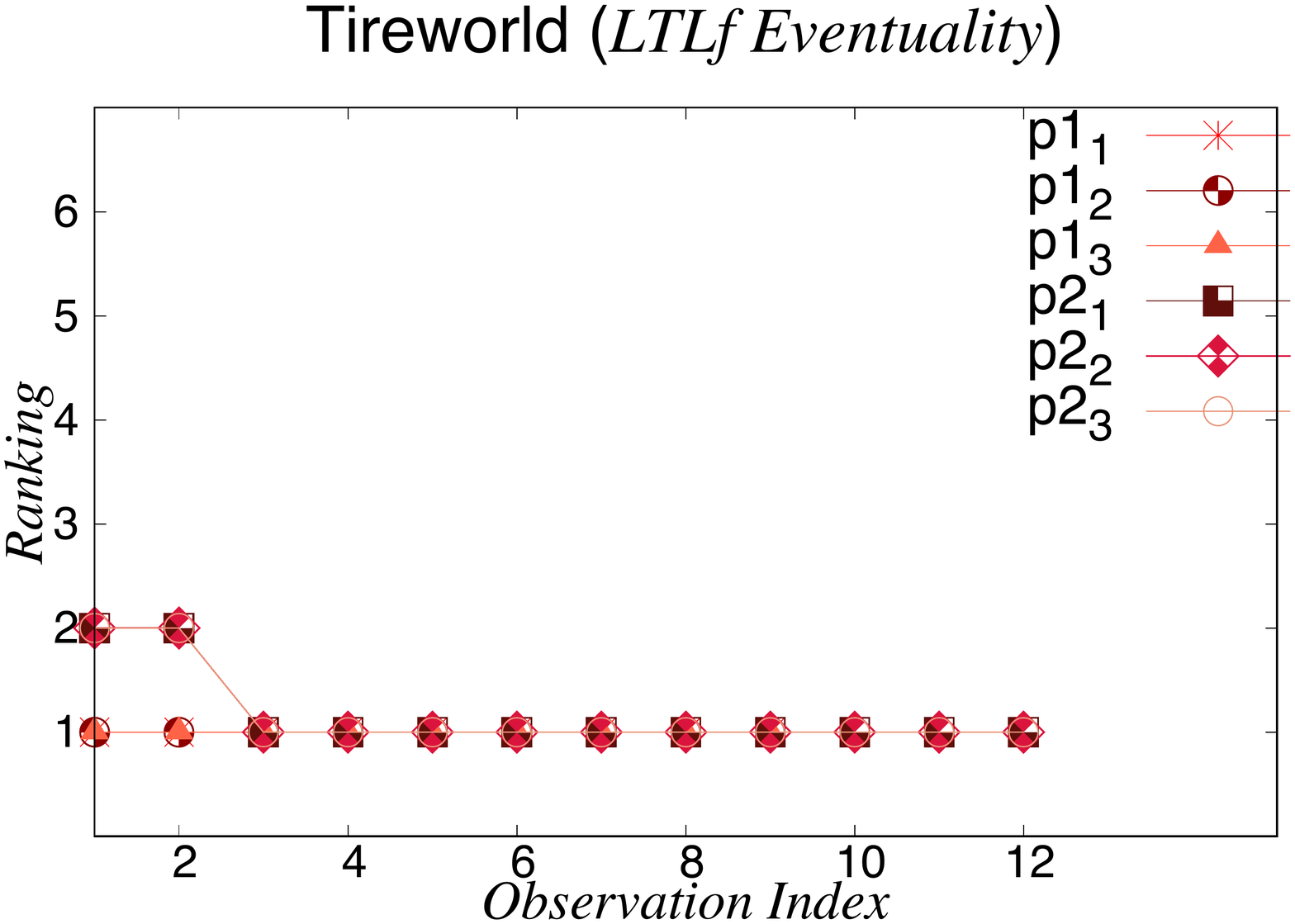}
		\caption{\LTLf Eventuality.}
		\label{fig:tireworld-ltl0}
	\end{subfigure}
	~
	\begin{subfigure}[b]{0.185\textwidth}
 	    \includegraphics[width=\textwidth]{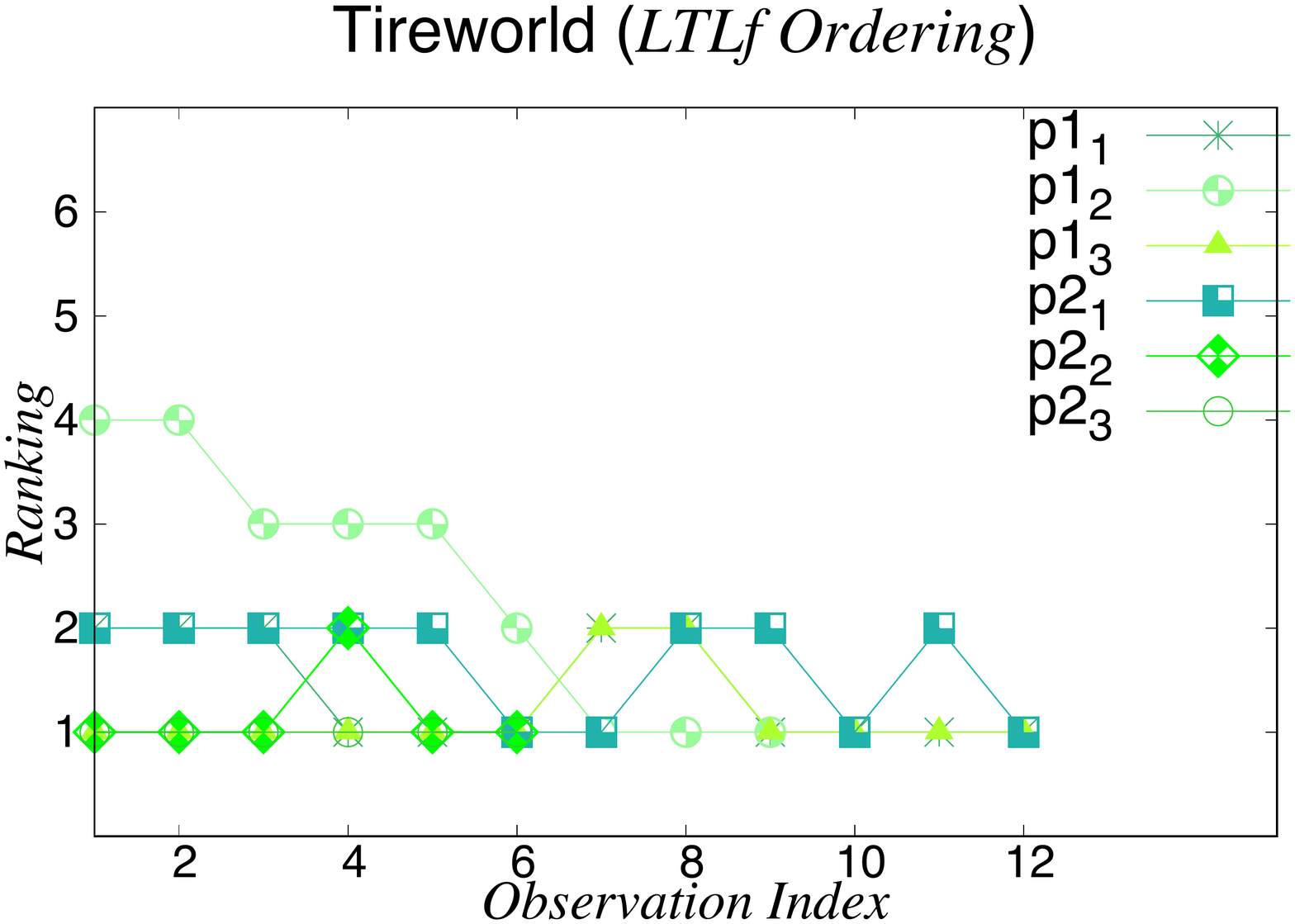}
		\caption{\LTLf Ordering.}
		\label{fig:tireworld-ltl1}
	\end{subfigure}
	~
	\begin{subfigure}[b]{0.185\textwidth}
 	    \includegraphics[width=\textwidth]{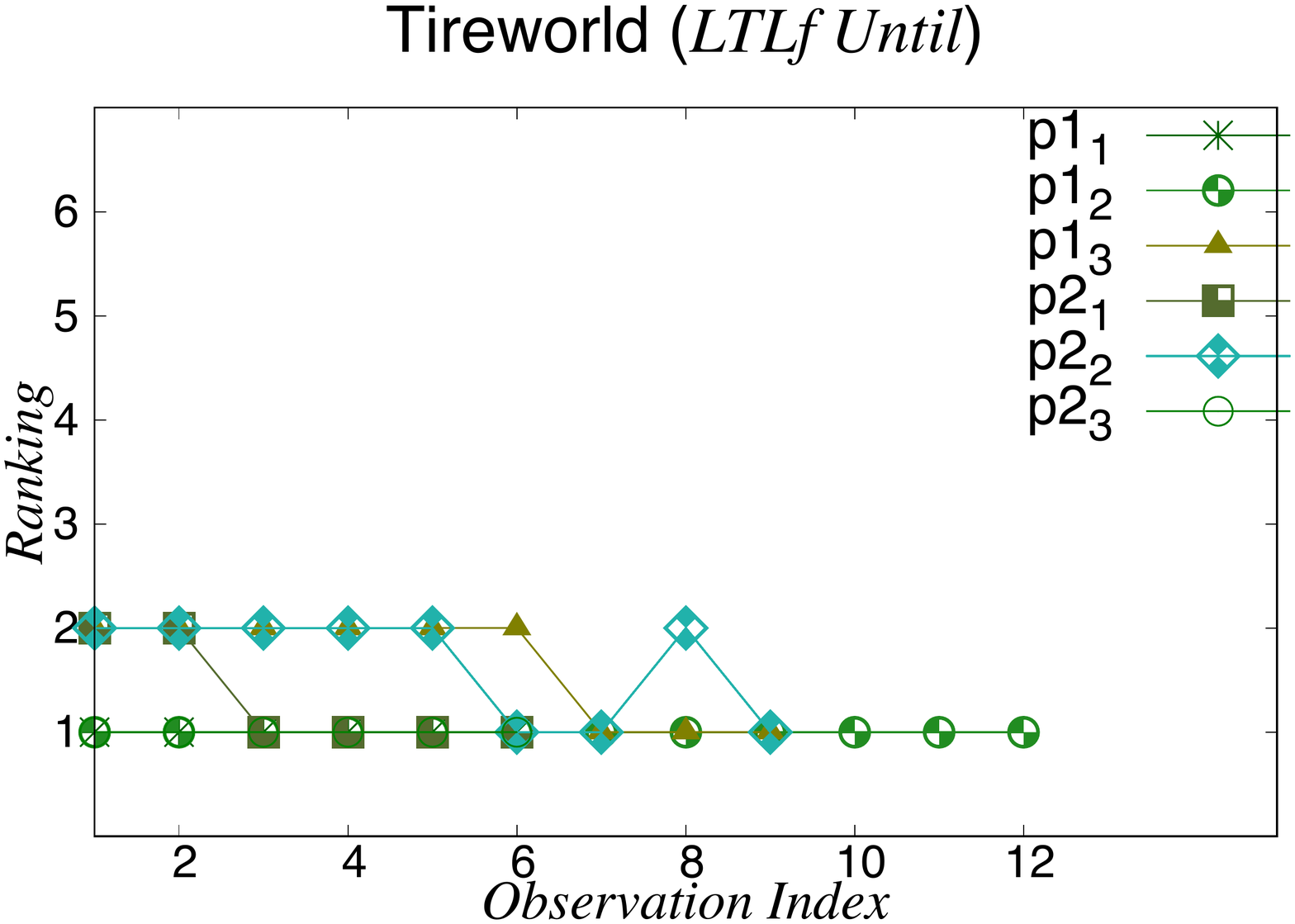}
		\caption{\LTLf Until.}
		\label{fig:tireworld-ltl2}
	\end{subfigure}
	~
	\begin{subfigure}[b]{0.185\textwidth}
 	    \includegraphics[width=\textwidth]{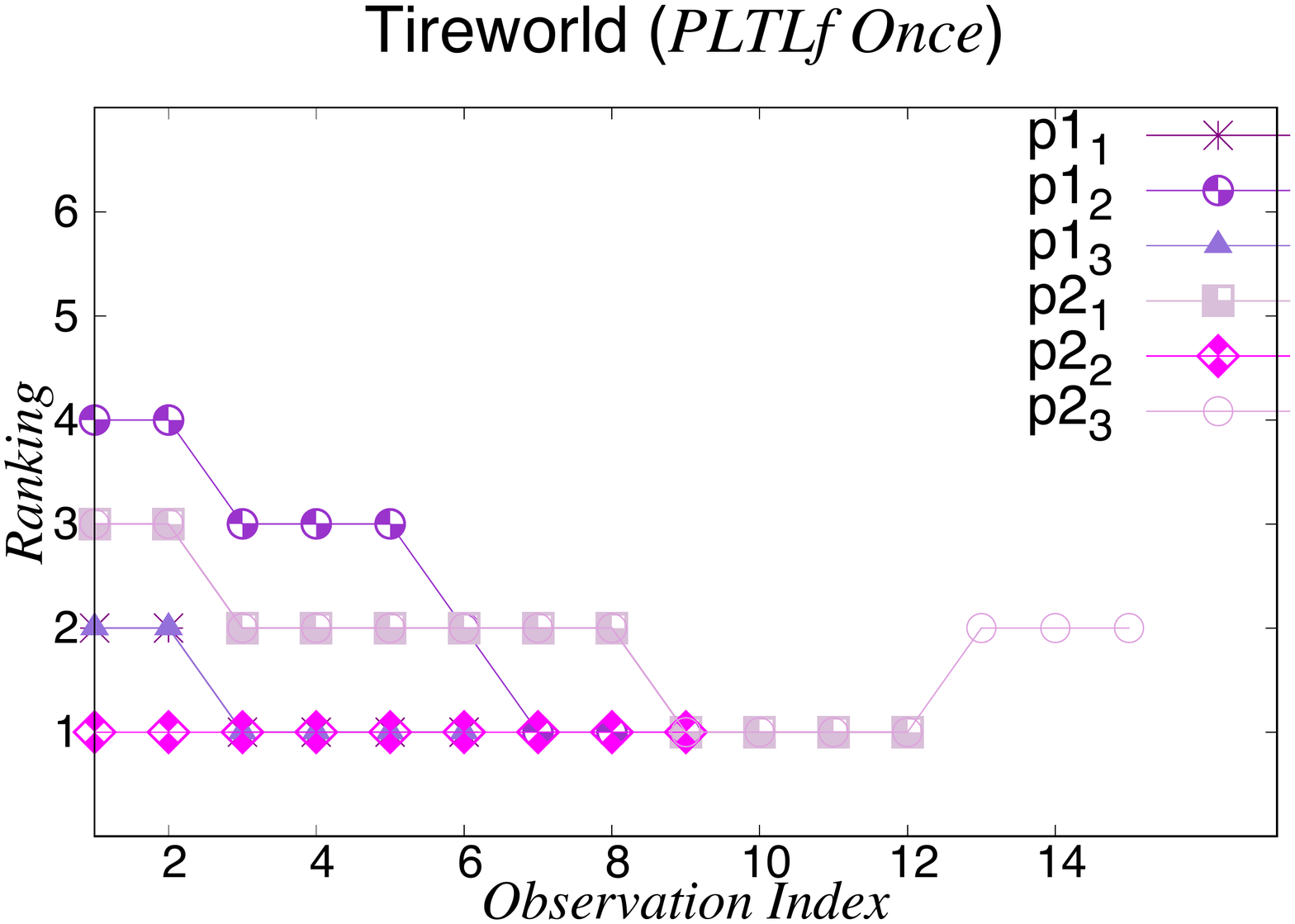}
		\caption{\PLTLf Once.}
		\label{fig:tireworld-pltl0}
	\end{subfigure}
	~
	\begin{subfigure}[b]{0.185\textwidth}
 	    \includegraphics[width=\textwidth]{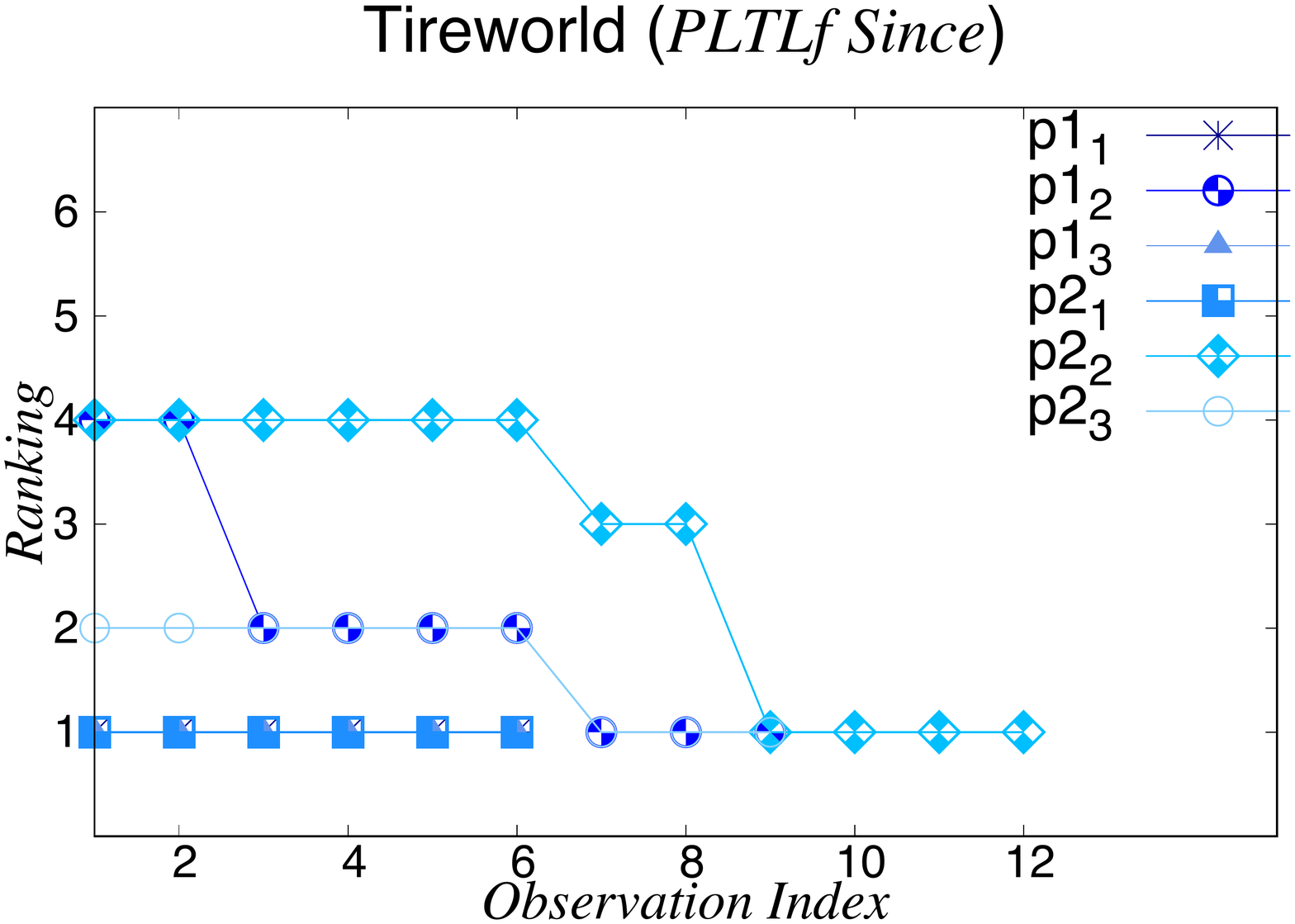}
		\caption{\PLTLf Since.}
		\label{fig:tireworld-pltl1}
	\end{subfigure}
	\vspace{-3mm}
	\caption{Online recognition ranking over the observations for \textsc{Tireworld}.}
	\label{fig:tireworld_ranking}
\end{figure*}

%--------------------------------------------------------------------

\begin{figure*}[!ht]
	\centering
	\begin{subfigure}[b]{0.185\textwidth}
 	    \includegraphics[width=\textwidth]{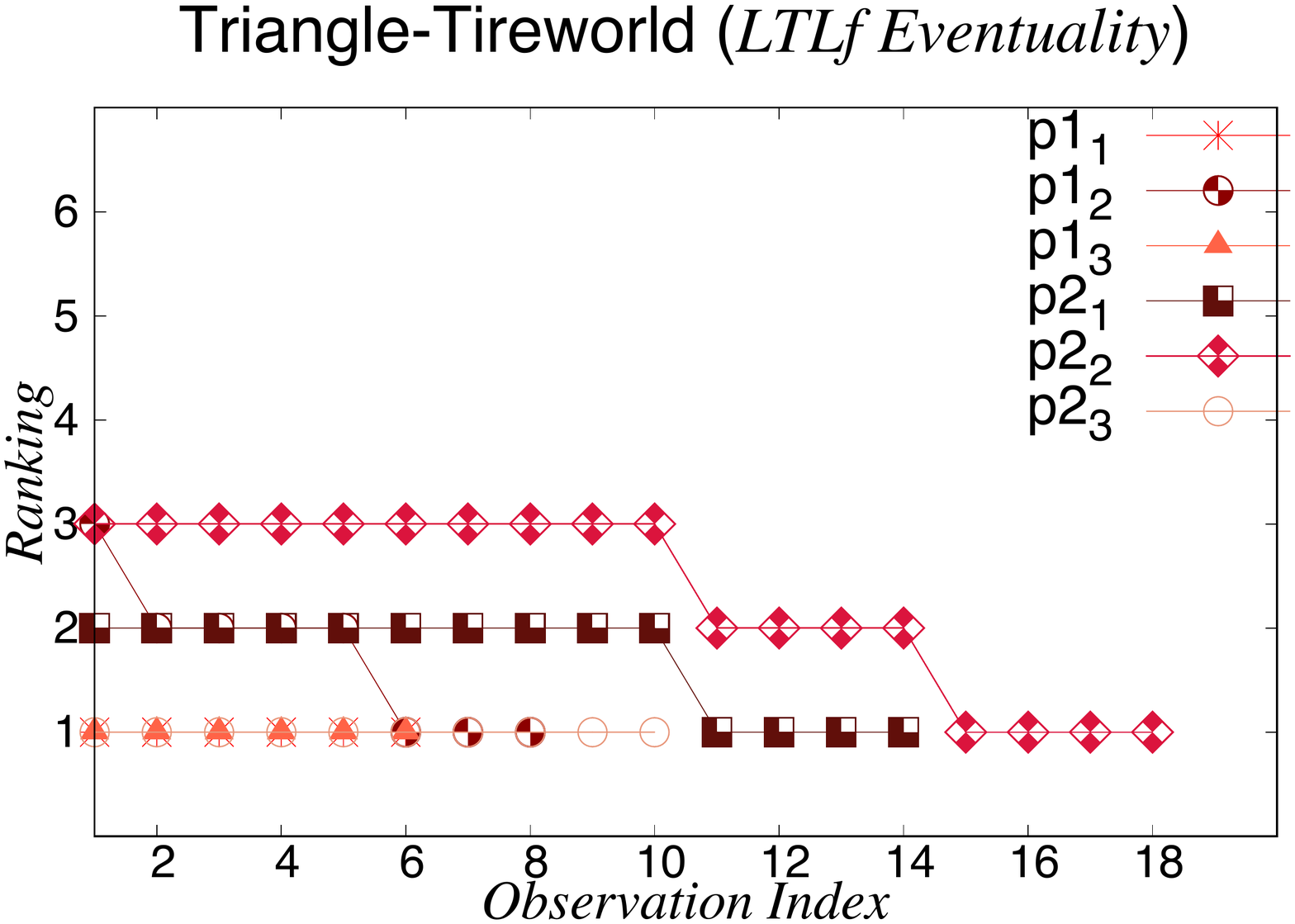}
		\caption{\LTLf Eventuality.}
		\label{fig:triangle_tireworld-ltl0}
	\end{subfigure}
	~
	\begin{subfigure}[b]{0.185\textwidth}
 	    \includegraphics[width=\textwidth]{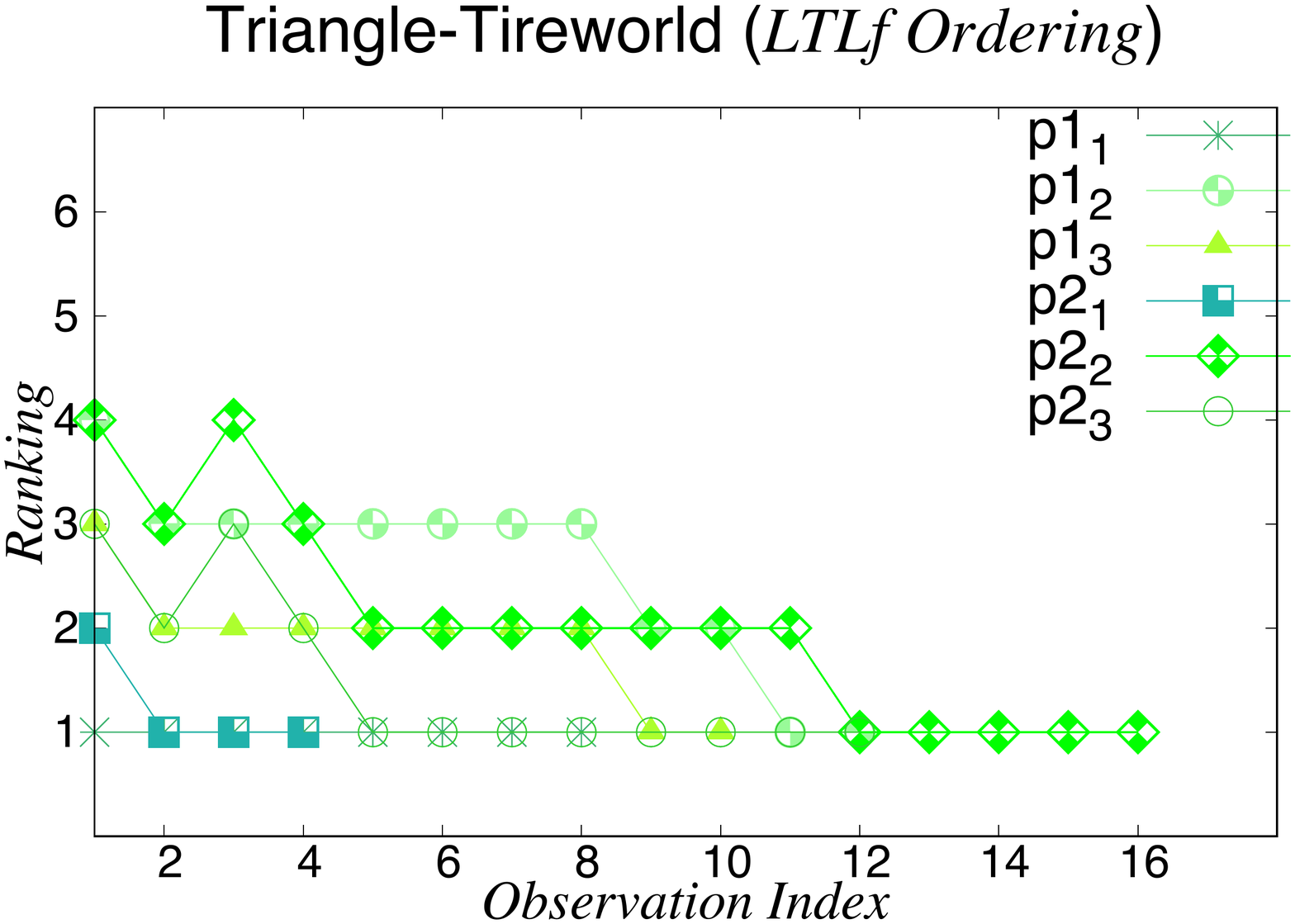}
		\caption{\LTLf Ordering.}
		\label{fig:triangle_tireworld-ltl1}
	\end{subfigure}
	~
	\begin{subfigure}[b]{0.185\textwidth}
 	    \includegraphics[width=\textwidth]{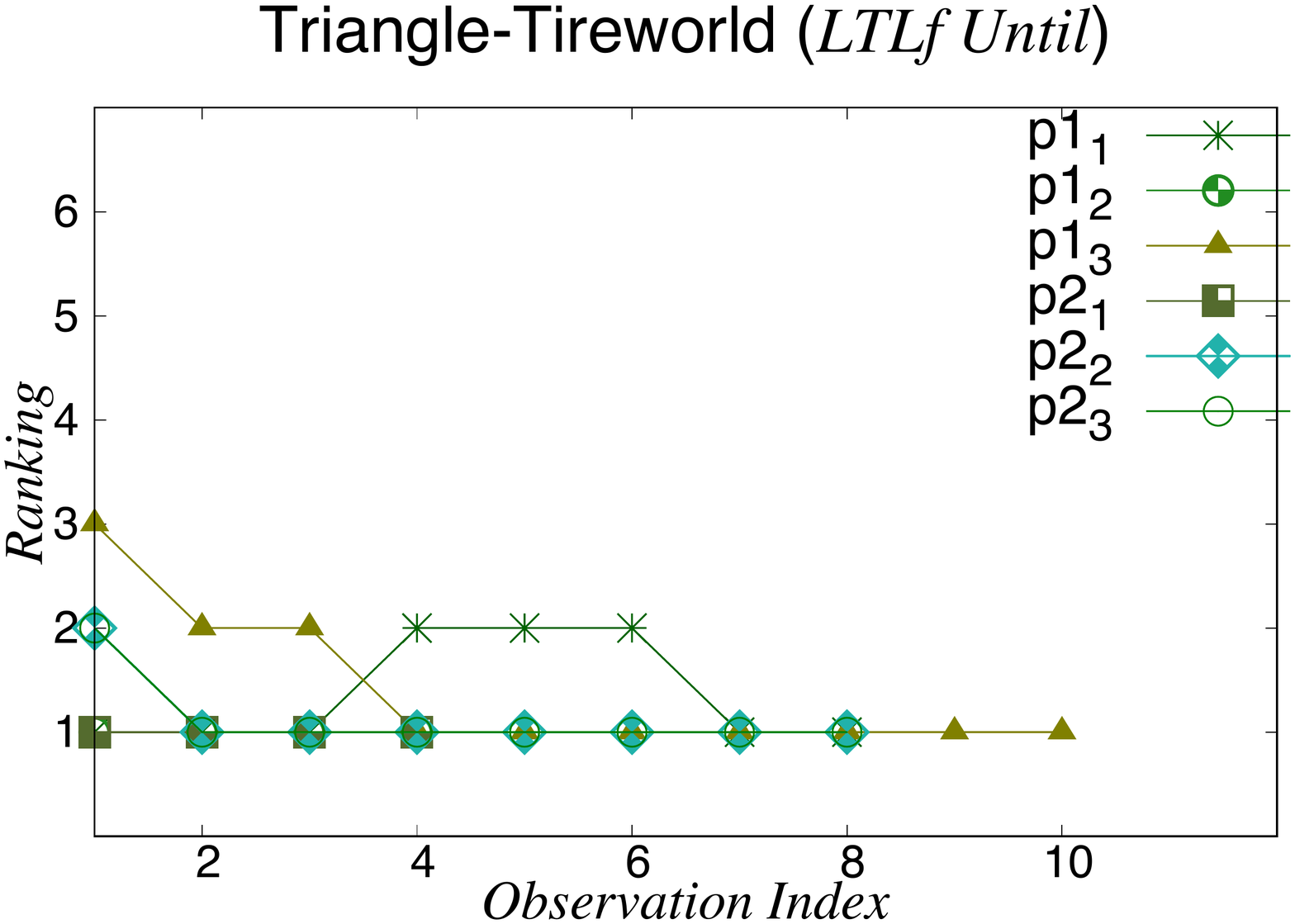}
		\caption{\LTLf Until.}
		\label{fig:triangle_tireworld-ltl2}
	\end{subfigure}
	~
	\begin{subfigure}[b]{0.185\textwidth}
 	    \includegraphics[width=\textwidth]{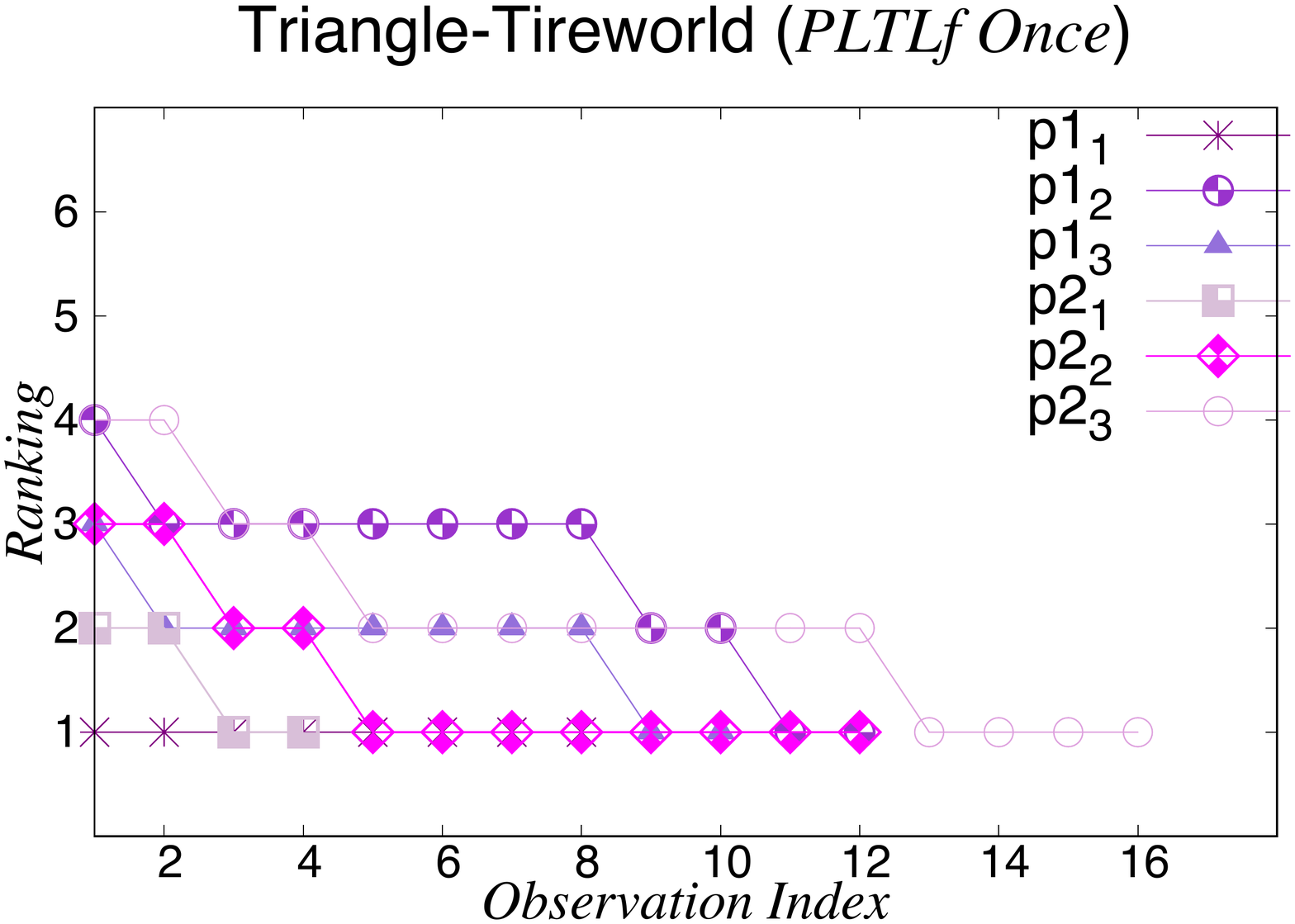}
		\caption{\PLTLf Once.}
		\label{fig:triangle_tireworld-pltl0}
	\end{subfigure}
	~
	\begin{subfigure}[b]{0.185\textwidth}
 	    \includegraphics[width=\textwidth]{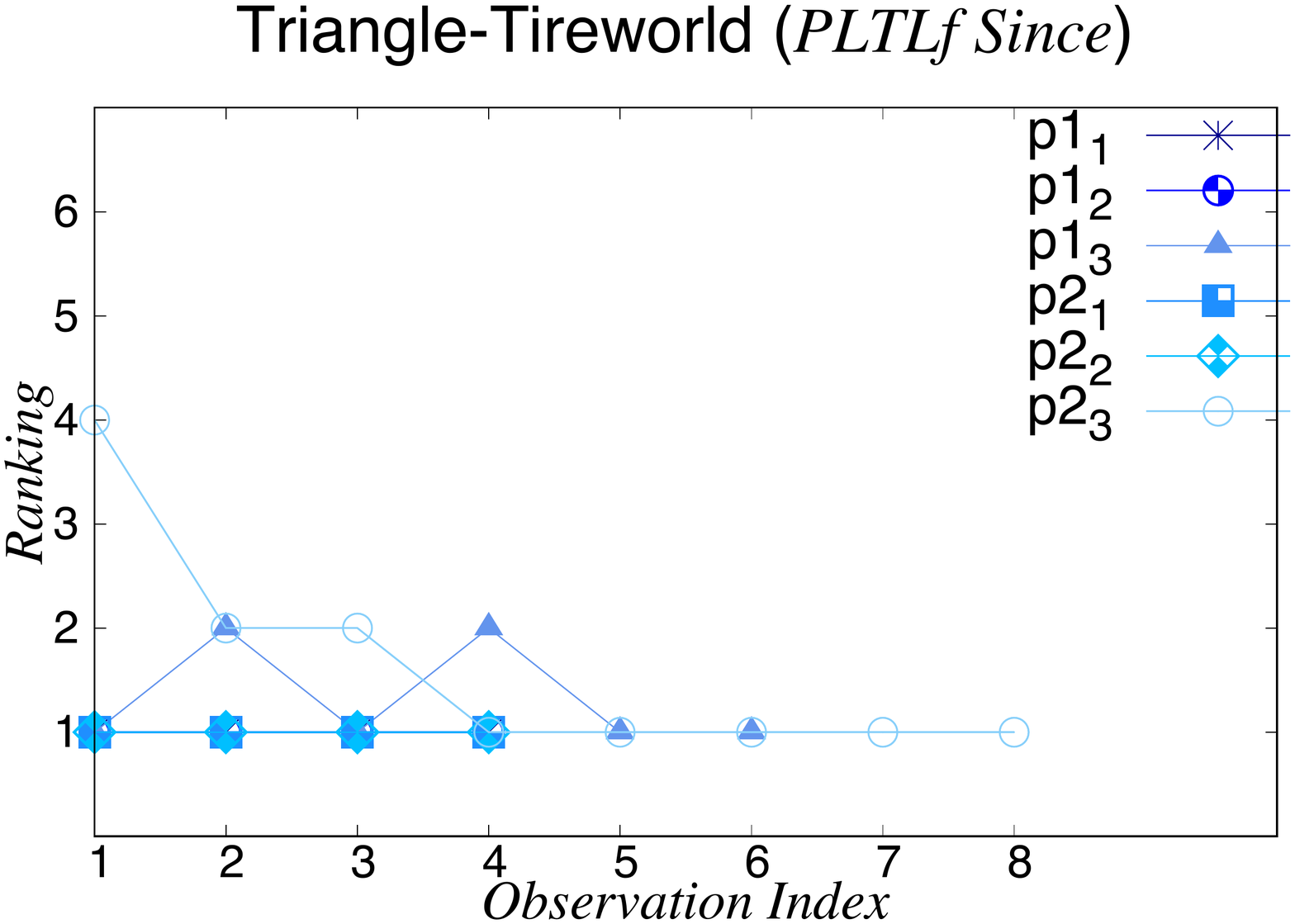}
		\caption{\PLTLf Since.}
		\label{fig:triangle_tireworld-pltl1}
	\end{subfigure}
	\vspace{-3mm}
	\caption{Online recognition ranking over the observations for \textsc{Triangle-Tireworld}.}
	\label{fig:fig:triangle_tireworld_ranking}
\end{figure*}

%--------------------------------------------------------------------

\begin{figure*}[!ht]
	\centering
	\begin{subfigure}[b]{0.185\textwidth}
 	    \includegraphics[width=\textwidth]{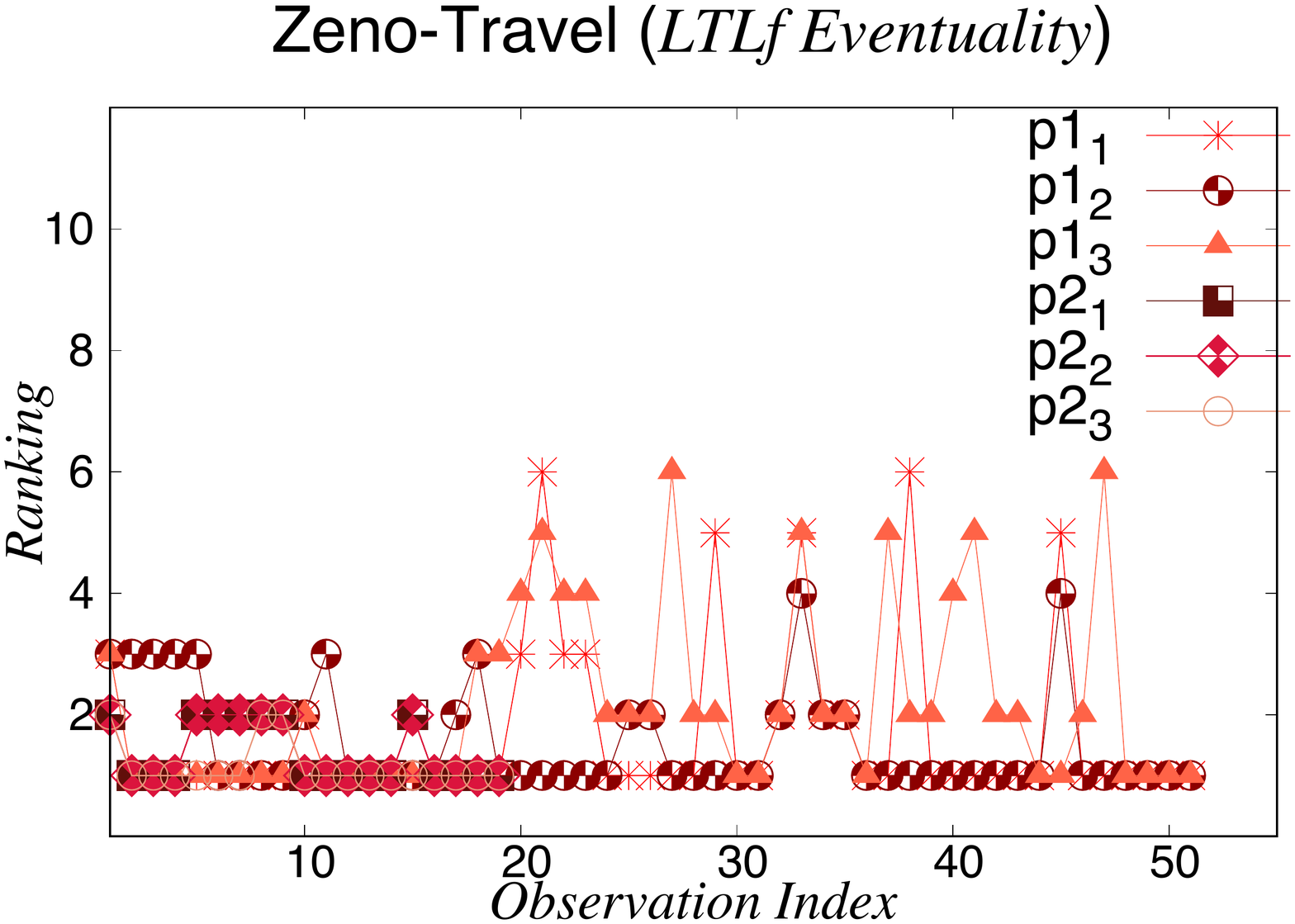}
		\caption{\LTLf Eventuality.}
		\label{fig:zenotravel-ltl0}
	\end{subfigure}
	~
	\begin{subfigure}[b]{0.185\textwidth}
 	    \includegraphics[width=\textwidth]{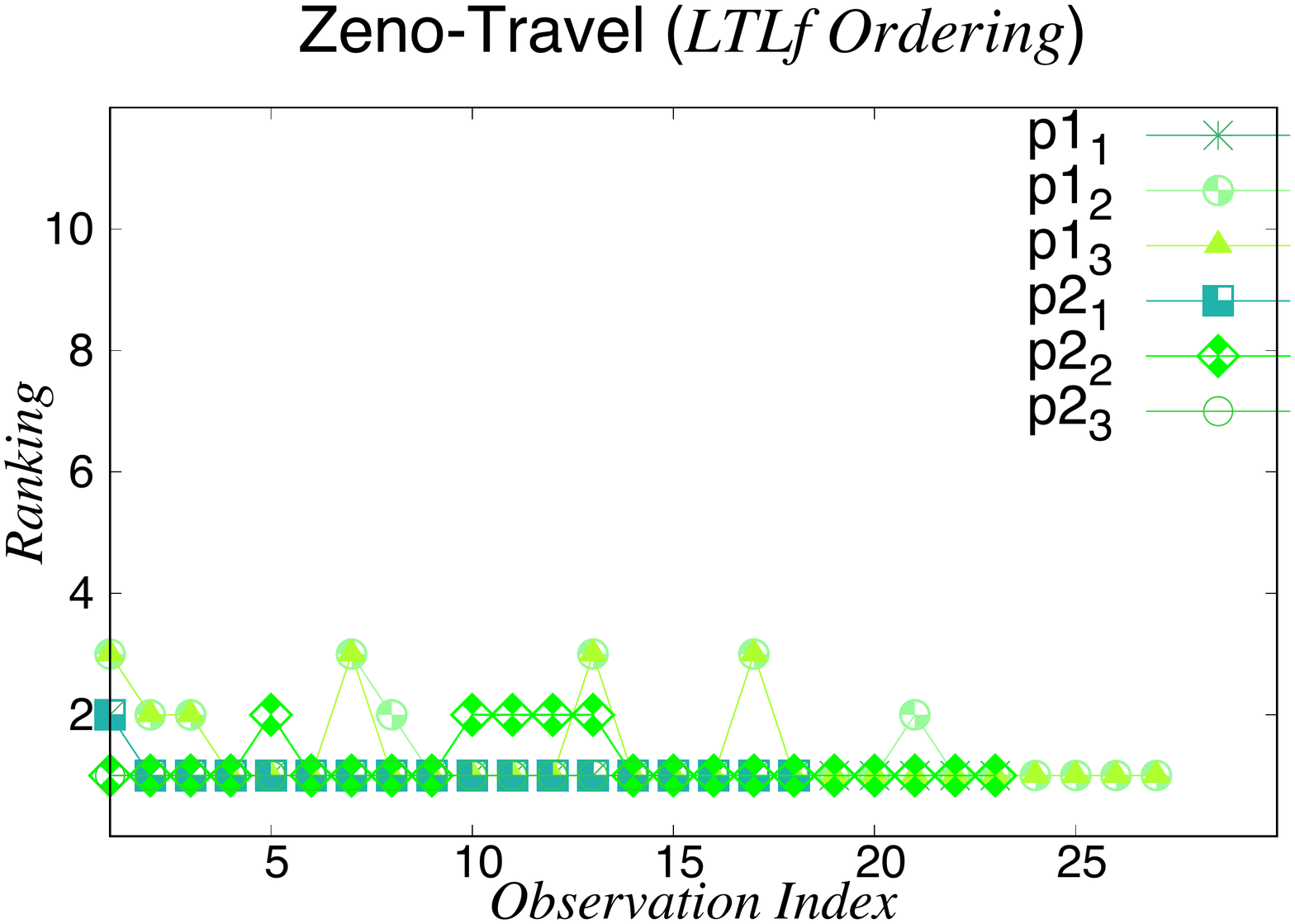}
		\caption{\LTLf Ordering.}
		\label{fig:zenotravel-ltl1}
	\end{subfigure}
	~
	\begin{subfigure}[b]{0.185\textwidth}
 	    \includegraphics[width=\textwidth]{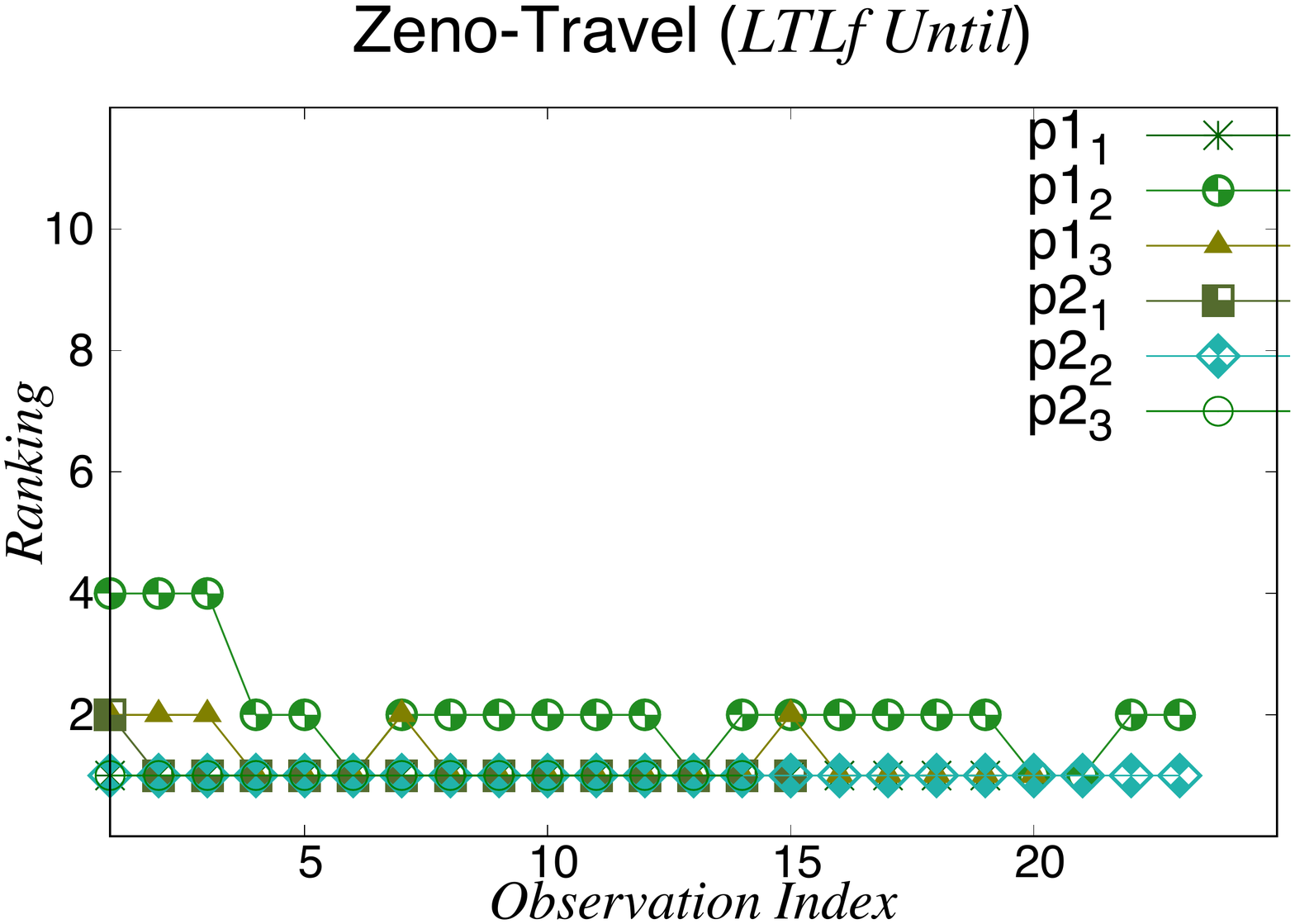}
		\caption{\LTLf Until.}
		\label{fig:zenotravel-ltl2}
	\end{subfigure}
	~
	\begin{subfigure}[b]{0.185\textwidth}
 	    \includegraphics[width=\textwidth]{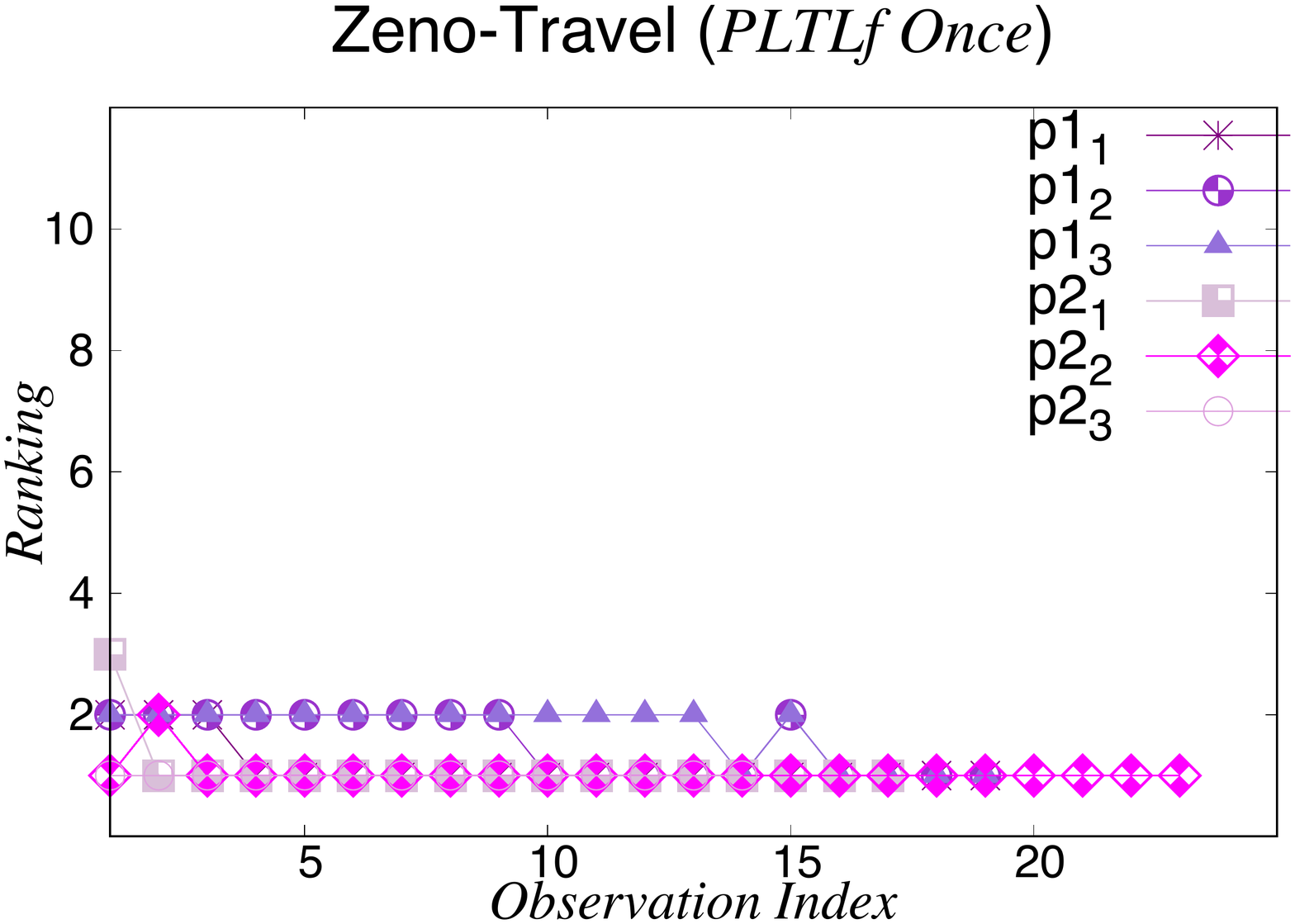}
		\caption{\PLTLf Once.}
		\label{fig:zenotravel-pltl0}
	\end{subfigure}
	~
	\begin{subfigure}[b]{0.185\textwidth}
 	    \includegraphics[width=\textwidth]{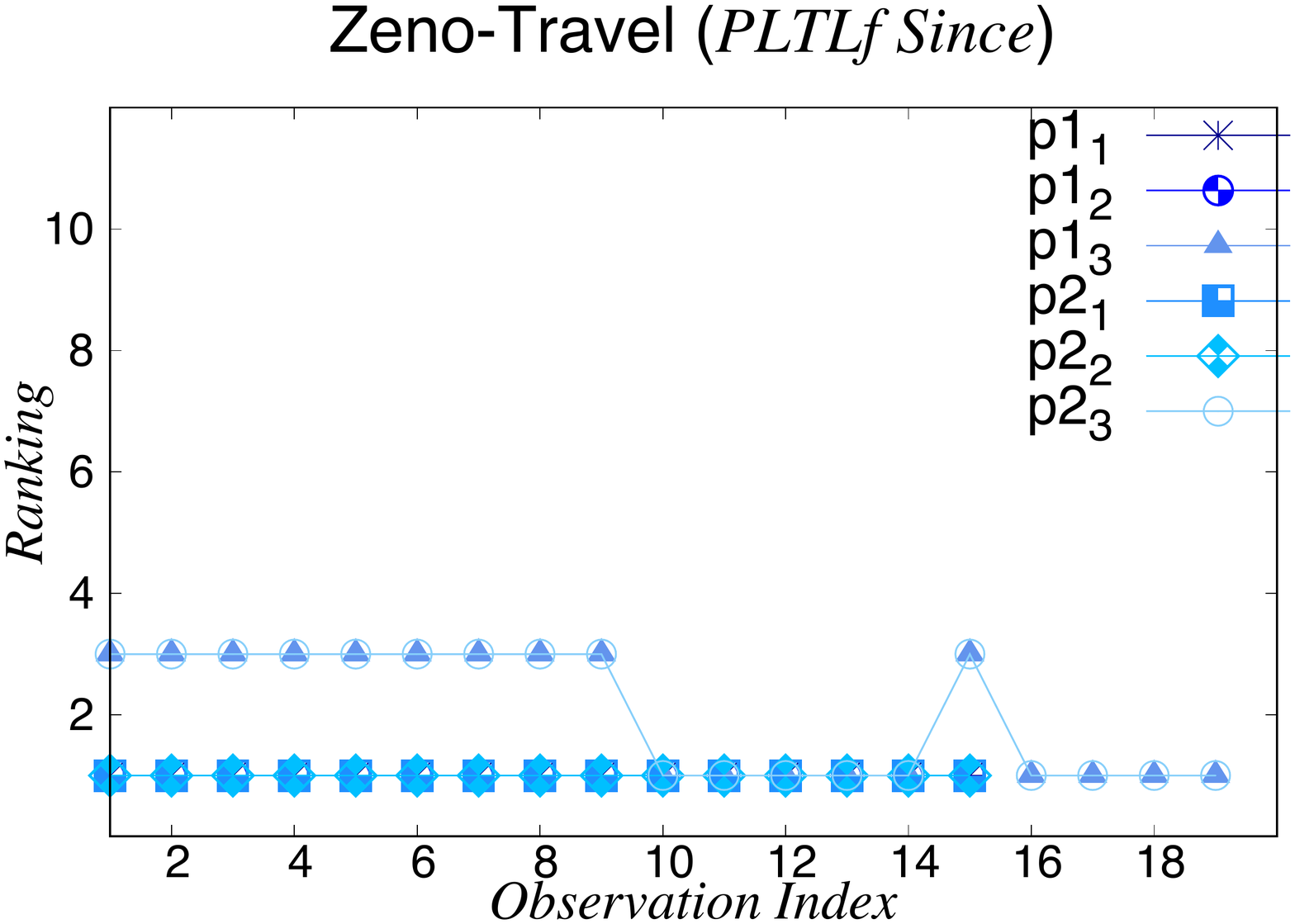}
		\caption{\PLTLf Since.}
		\label{fig:zenotravel-pltl1}
	\end{subfigure}
	\vspace{-3mm}
	\caption{Online Recognition ranking over the observations for \textsc{Zeno-Travel}.}
	\label{fig:zenotravel_ranking}
\end{figure*}

%--------------------------------------------------------------------

%----------------------------------------------------------------------------------------

\bibliographystyle{spbasic}       % basic style, author-year citations

\bibliography{bibliography}

\end{document}